\documentclass[nohyperref]{article}

\usepackage{microtype}
\usepackage{graphicx}
\usepackage{subfigure}
\usepackage{booktabs} 

\usepackage{hyperref}

\usepackage[accepted]{icml2022}

\usepackage{amsmath}
\usepackage{amssymb}
\usepackage{mathtools}
\usepackage{amsthm}
\usepackage{bm}
\usepackage{bbm}

\usepackage{enumitem}

\usepackage[capitalize,noabbrev]{cleveref}

\theoremstyle{plain}
\newtheorem{theorem}{Theorem}[section]

\newtheorem{lemma}[theorem]{Lemma}

\theoremstyle{definition}

\newtheorem{assumption}[theorem]{Assumption}
\newtheorem{remark}{Remark}

\usepackage[textsize=tiny]{todonotes}

\icmltitlerunning{Neural Tangent Kernel Beyond the Infinite-Width Limit}

\newcommand{\x}{\mathbf{x}}
\newcommand{\h}{\mathbf{h}}
\newcommand{\g}{\bm{\delta}}
\newcommand{\bias}{\mathbf{b}}
\newcommand{\W}{\mathbf{W}}
\newcommand{\w}{\mathbf{w}}

\DeclareMathOperator\erf{erf}

\begin{document}

\twocolumn[
\icmltitle{Neural Tangent Kernel Beyond the Infinite-Width Limit: \\ Effects of Depth and Initialization}



\icmlsetsymbol{equal}{*}

\begin{icmlauthorlist}
\icmlauthor{Mariia Seleznova}{lmu}
\icmlauthor{Gitta Kutyniok}{lmu}
\end{icmlauthorlist}

\icmlaffiliation{lmu}{Department of Mathematics, Ludwig-Maximilians-Universität München, Munich, Germany}

\icmlcorrespondingauthor{Mariia Seleznova}{selez@math.lmu.de}

\icmlkeywords{Machine Learning, ICML}

\vskip 0.3in
]



\printAffiliationsAndNotice{}  

\begin{abstract}
Neural Tangent Kernel (NTK) is widely used to analyze overparametrized neural networks due to the famous result by \citet{jacot2018neural}: in the infinite-width limit, the NTK is deterministic and constant during training. However, this result cannot explain the behavior of deep networks, since it generally does not hold if depth and width tend to infinity simultaneously.
In this paper, we study the NTK of fully-connected ReLU networks with depth comparable to width. We prove that the NTK properties depend significantly on the depth-to-width ratio and the distribution of parameters at initialization. In fact, our results indicate the importance of the three phases in the hyperparameter space identified in \citet{poole2016exponential}: \textit{ordered}, \textit{chaotic} and the \textit{edge of chaos} (EOC). We derive exact expressions for the NTK dispersion in the {infinite-depth-and-width} limit in all three phases and conclude that the NTK variability grows exponentially with depth at the EOC and in the chaotic phase but not in the ordered phase. We also show that the NTK of deep networks may stay constant during training only in the ordered phase and discuss how the structure of the NTK matrix changes during training.

\end{abstract}

\section{Introduction}
Despite the widespread use of Deep Neural Networks (DNNs), the theory behind their success is still poorly understood. For instance, no present theory can explain why highly overparametrized DNNs generalize very well in practice, contrary to classical statistical learning theory predictions. Likewise, it is surprising that optimizing a highly non-convex loss function of a DNN with a variant of Gradient Descent (GD) typically yields a good local minimum. 

Although training dynamics and generalization capabilities of DNNs stand among the biggest open problems of deep learning theory, it is possible to address these challenges in the special case of infinitely-wide DNNs using the so-called \textit{Neural Tangent Kernel} (NTK). This kernel captures the first-order approximation of DNN's evolution during GD training. Consider the gradient flow dynamics of the DNN's parameters: 
\begin{equation}\label{eq:gd_w}
    \dot{\w} = - \nabla_\w \mathcal{L}(\mathcal{D}) = - \sum_{(x_i,y_i)\in\mathcal{D}} \nabla_{\w}f(x_i) \dfrac{\partial \mathcal{L}(\mathcal{D})}{\partial f(x_i)} ,
\end{equation}
where $\w$ is the vector of all the trainable parameters, $f(\cdot)$ is the DNN's output function (defined in Section \ref{section:preliminaries}), $\mathcal{L}(\cdot)$ is the loss function and $\mathcal{D}$ is the dataset. Then the dynamics of the DNN's output function is given by:
\begin{equation}\label{eq:gd_f}
\begin{split}
    \dot{f}(x) = \nabla_{\w}f(x)\cdot\dot{\w} = - \sum_{(x_i,y_i)\in\mathcal{D}} \Theta(x,x_i) \dfrac{\partial \mathcal{L}(\mathcal{D})}{\partial f(x_i)},
\end{split}
\end{equation}
where the kernel $\Theta(x_i,x_j) := \bigl\langle \nabla_\w f(x_i), \nabla_\w f(x_j) \bigr\rangle$ is called the NTK. 

A famous result by \citet{jacot2018neural} states that in the infinite-width limit, the NTK is deterministic under proper random initialization and stays constant during training. Thereby, the dynamics in \eqref{eq:gd_f} is equivalent to kernel regression and has an analytical solution expressed in terms of the kernel. It is then possible to derive properties of trained infinitely-wide DNNs theoretically by means of their NTKs. Hence, many recent works used the NTK to explain empirically known properties of DNNs \cite{huang2020deep, adlam2020neural, wang2022and,tirer2020kernel, geiger2020scaling}. Numerous contributions also derived the infinite-width limit of the NTK for popular DNN architectures \citep{yang2020tensor,du2019graph,alemohammad2020recurrent}. Other papers established some non-asymptotic results on the concentration of the NTK at initialization \cite{arora2019exact,buchanan2020deep} and stability of the NTK during training \cite{huang2020dynamics,lee2019wide}.

However, the extent to which the results in the infinite-width limit extrapolate to realistic DNNs remains largely an open question. Indeed, multiple authors have argued that the NTK regime and, in general, the infinite-width limit cannot explain the success of DNNs \citep{chizat2019lazy,hanin2019finite,aitchison2020bigger,li2021future,seleznova2020analyzing,bai2020taylorized,huang2020dynamics}. The first argument in this direction is that no feature learning occurs if the NTK stays constant during training. Moreover, several works showed that the infinite-width limit of the NTK becomes completely data-independent as depth increases \citep{xiao2019disentangling,hayou2019mean}, which suggests poor generalization performance for deep networks in the NTK regime. Finally, numerous empirical results demonstrated that the performance of trained DNNs and the corresponding kernel methods often differs in practice \citep{fort2020deep,lee2020finite,arora2019harnessing}. That is why it is essential to understand the statistical properties of the NTK and how they depend on the myriad of settings of a given DNN to assess if the infinite-width limit provides a reasonable approximation for this network. We contribute to this line of research by exploring the combined effect of two factors on the NTK: the network's depth and initialization hyperparameters.

\textbf{Network's depth \;} Most results on the NTK are derived in the setting where the network's depth is kept constant while the width tends to infinity. This limit can only model very wide and shallow networks since the depth-to-width ratio tends to zero in it. Indeed, several recent papers demonstrated that infinite-width approximations often get worse as the depth increases \citep{li2021future,matthews2018gaussian,yang2017mean}. In particular, \citet{hanin2019finite} first showed that the NTK of fully-connected ReLU DNNs may be random and change during training if depth and width are comparable. \citet{hu2021random} also studied the effects of depth on the NTK distribution and derived an upper bound for the NTK moments. We expand on these results by precisely characterizing the variability of the NTK at initialization and generalizing to different initialization settings described below.

\textbf{Initialization hyperparameters \;} There are three phases in the initialization hyperparameter space where the properties of untrained infinitely-wide DNNs differ significantly: \textit{ordered}, \textit{chaotic} and the \textit{edge of chaos} (EOC) \cite{poole2016exponential}. In the ordered phase, the gradient norms decrease with depth, whereas in the chaotic phase the gradient norms increase, and the edge of chaos is the initialization at the border between these two phases \cite{schoenholz2016deep}. The results by \citet{hanin2019finite} concerned the statistical properties of the NTK of wide and deep ReLU networks at the EOC. At the same time, several contributions demonstrated that the properties of the infinite-width NTK depend significantly on the phase of initialization \cite{xiao2019disentangling,hayou2019mean}. However, these results do not apply to networks with depth comparable to width since they assume infinite width before considering the effects of growing depth. We fill this gap by deriving statistical properties of the NTK for wide and deep ReLU networks in all three phases of initialization. 

\subsection{Contributions}
We study the \textbf{variability of the NTK at initialization} for fully-connected ReLU DNNs with depth comparable to width and varying initialization hyperparameters in Section~\ref{section:var}. Our contributions are as follows:
\begin{itemize}[leftmargin=*,topsep=0pt,itemsep=0pt]
    \item We precisely characterize the dispersion of the diagonal elements $\Theta(x,x)$ of the NTK (for arbitrary input $x$) in the \textbf{infinite-depth-and-width limit} and conclude that the variability of the NTK grows exponentially with the depth-to-width ratio at the EOC and in the chaotic phase. Conversely, the variance of $\Theta(x,x)$ tends to zero in the same limit in the ordered phase. Our results allow to evaluate the variance of the NTK for a given DNN with any depth-to-width ratio and initialization.
    \item We provide non-asymptotic expressions for the first two moments of $\Theta(x,x)$ and illustrate \textbf{finite-width effects} that follow. We show that the variance of the finite-width NTK in the ordered phase gradually increases as the initialization approaches the EOC, which describes the transition between the two kinds of behavior in the limit. We also notice that the NTK dispersion depends on the architecture, i.e. on the varying widths of the fully-connected layers. Notably, the dispersion of $\Theta(x,x)$ decreases with depth in the ordered phase if the DNN increases the dimensionality in consecutive layers. This enables us to conclude that deeper networks are more robust to random initialization in this setting. 
    \item We lower-bound the ratio of the expected \textbf{non-diagonal elements} of the NTK, i.e. $\Theta(x,\tilde{x})$ with $x\neq\Tilde{x}$, and the diagonal elements $\Theta(x,x)$ in the infinite-depth-and-width limit. We also upper-bound the dispersion of the non-diagonal elements. In the ordered phase, our results allow to ensure that the whole NTK matrix is approximately deterministic and thus can be approximated by the infinite-width limit. 
    \item We provide extensive \textbf{numerical experiments} to ve\-ri\-fy our theoretical results. We use JAX \cite{jax2018github} and Flax (neural network library for JAX) \cite{flax2020github} to compute the NTK of fully-connected ReLU networks effortlessly. Source code to reproduce the presented results is available at: \href{https://github.com/mselezniova/ntk\_beyond\_limit}{https://github.com/mselezniova/ntk\_beyond\_limit}.
\end{itemize}

\vspace{1ex}
We study the \textbf{training dynamics of the NTK} for fully-connected ReLU DNNs with depth comparable to width and varying initialization hyperparameters in Section \ref{section:train}. Our contributions are as follows:
\begin{itemize}[leftmargin=*,topsep=0pt,itemsep=0pt]
    \item We show that the expected relative change of $\Theta(x,x)$ in \textbf{the first GD step} tends to infinity in the chaotic phase and to zero in the ordered phase in the infinite-depth-and-width limit. Combined with the result by \citet{hanin2019finite}, which states that the expected relative change of $\Theta(x,x)$ in the first GD step is exponential in the depth-to-width ratio at the EOC, we can conclude that the NTK of deep networks can stay approximately constant during GD training only in the ordered phase. 
    \item We discuss how the \textbf{structure of the NTK} matrix changes during training outside of the NTK regime. The NTK matrix at initialization has a diagonal structure with larger values on the main diagonal as compared to the non-diagonal ones. We speculate that the training process changes the NTK structure to block-diagonal with blocks of larger values corresponding to classes and provide experiments to support this sentiment.
\end{itemize}

\section{Preliminaries}\label{section:preliminaries}
We consider fully-connected ReLU DNNs of depth $L\in~\mathbb{N}$ with linear output layer and widths $(n_\ell)_{0\leq\ell\leq L}$, where $n_0\in\mathbb{N}$ is the input dimension and $n_L=1$ is the output dimension. Forward propagation in such a network is defined as follows:
\begin{equation}\label{eq:nn}
\begin{split}
    \x^\ell(x) &:= \phi(\h^\ell(x)), \quad \x^{0}(x):=x\in\mathbb{R}^{n_0},\\
    \h^\ell(x) &:= \W^\ell\x^{\ell-1}(x) + \mathbf{b}^\ell, \quad 1\leq\ell\leq L-1, \\
    f(x)&:=\W^L\x^{L-1}(x) + \mathbf{b}^L \in\mathbb{R},
\end{split}
\end{equation}
where $\phi(x):=x\mathbbm{1}_{\{x>0\}}$ denotes the ReLU function, $\W^\ell\in\mathbb{R}^{n_\ell\times n_{\ell-1}}$ and  $\bias^\ell\in\mathbb{R}^{n_\ell}$ are the weights and the biases and $f(x)$ is the output function of the DNN. The NTK of this network on a pair of inputs $(x,\Tilde{x})$ is given by:
\begin{equation}\label{eq:ntk}
\begin{split}
    \Theta(x,\Tilde{x}) &:= \Theta_W(x,\Tilde{x}) + \Theta_b(x,\Tilde{x}),\\
    \Theta_W(x,\Tilde{x}) &:= \sum_{\ell=1}^L\sum_{j=1}^{n_{\ell}}\sum_{i=1}^{n_{\ell-1}}\dfrac{\partial f(x)}{\partial \W_{ij}^\ell}\dfrac{\partial f(\Tilde{x})}{\partial \W_{ij}^\ell},\\
    \Theta_b(x,\Tilde{x}) &:= \sum_{\ell=1}^L\sum_{j=1}^{n_{\ell}} \dfrac{\partial f(x)}{\partial \bias_{j}^\ell}\dfrac{\partial f(\Tilde{x})}{\partial \bias_{j}^\ell},
\end{split}
\end{equation}
where $\Theta_W(x,\Tilde{x})$ comprises the gradients w.r.t. the weights and $\Theta_b(x,\Tilde{x})$ --- the gradients w.r.t. the biases.

 When we consider wide networks with unequal widths in the hidden layers, we define a width scale parameter $M$ and constants $\lambda$, $(\alpha_\ell)_{0\leq\ell\leq L-1}$ such that:
\begin{equation}\label{eq:scaling}
\begin{split}
    \dfrac{L}{M}=\lambda\in\mathbb{R},\quad \dfrac{n_\ell}{M}=\alpha_\ell\in\mathbb{R}, \quad 0\leq\ell\leq L-1.
\end{split}
\end{equation}
Then we can describe the asymptotic behavior of the NTK in terms of $M$ and the constants defined above.

\subsection{Initialization and parametrization}
We consider random i.i.d. initialization given by:
\begin{equation}\label{eq:init}
    \W_{ij}^\ell \sim \mathcal{N}\Bigl(0,\dfrac{\sigma_w^2}{n_{\ell-1}}\Bigr), \quad \bias^\ell_i \sim \mathcal{N}(0,\sigma_b^2),
\end{equation}
where $(\sigma_w,\sigma_b)$ are the initialization hyperparameters. This initialization corresponds to the so-called standard parametrization (SP), where the weights and the biases defined in \eqref{eq:nn} are the trainable parameters. We note that the NTK is often considered in the so-called NTK parametrization (NTP), where the weights in \eqref{eq:nn} are the scaled versions of trainable parameters: $\W_{ij}^\ell=\sigma_w/\sqrt{n_{\ell-1}} \w_{ij}^\ell$ for trainable $\w^\ell\in\mathbb{R}^{n_{\ell}\times n_{\ell-1}}$ initialized as $\w_{ij}^\ell \sim \mathcal{N}(0,1)$ i.i.d. This reparametrization, of course, does not change the distribution of the DNN's components. However, it scales the gradients by $O(1/M)$ and gives the NTK a well-defined infinite-width limit for fixed $L$. At the same time, NTP is equivalent to setting an individual learning rate in each layer inverse-proportionally to width, as explained, e.g., in \citet{yang2020feature}. In this paper, we focus on the NTK in SP since this parametrization is more common in practice (indeed, SP is the default setting in PyTorch). However, our results can be generalized to NTP straightforwardly.

\subsection{Information propagation in DNNs}
Results on information propagation in infinitely-wide DNNs established that the initialization hyperparameters $(\sigma_w,\sigma_b)$ determine the evolution of the variances $\mathbb{E}[(\x_i^\ell(x))^2]$ and the covariances $\mathbb{E}[\x_i^\ell(x)\x_i^\ell(\Tilde{x})]$ as they propagate through the DNN's layers. Based on this, \citet{poole2016exponential} identified three phases with distinct properties in the hyperparameter space: \textit{ordered}, \textit{chaotic} and the \textit{edge of chaos} (EOC). \citet{schoenholz2016deep} subsequently showed that the ordered phase corresponds to vanishing gradients and the chaotic phase corresponds to exploding gradients, i.e. the gradient norm decreases with depth in the ordered phase and increases in the chaotic phase. The edge of chaos is the initialization at the border between these two phases, which allows deeper signal propa\-gation through a DNN by avoiding vanishing or exploding gradients. Consider backpropagation equations given by:
\begin{figure*}
    \centering
    \includegraphics[width=0.42\textwidth]{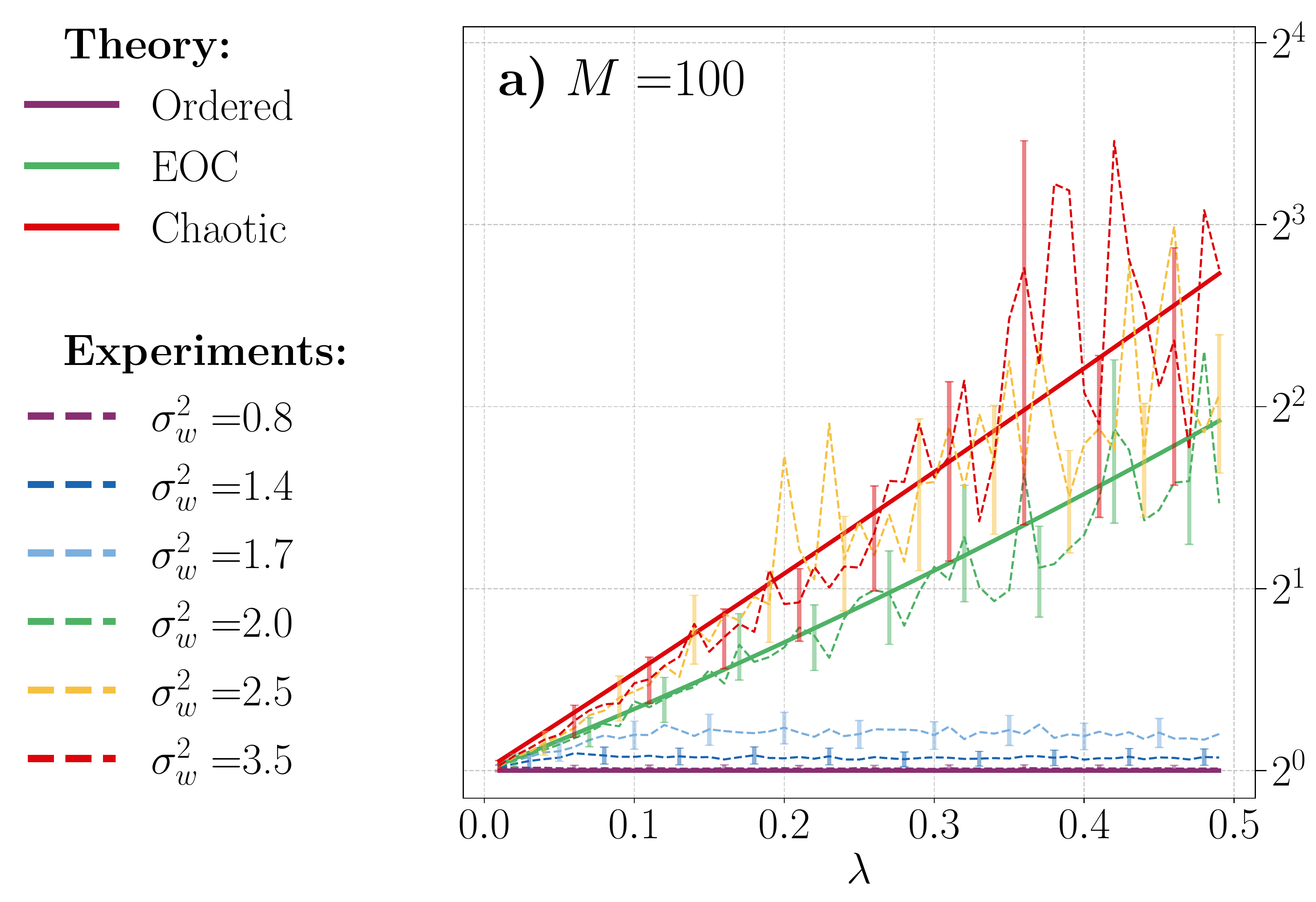}
    \includegraphics[width=0.275\textwidth]{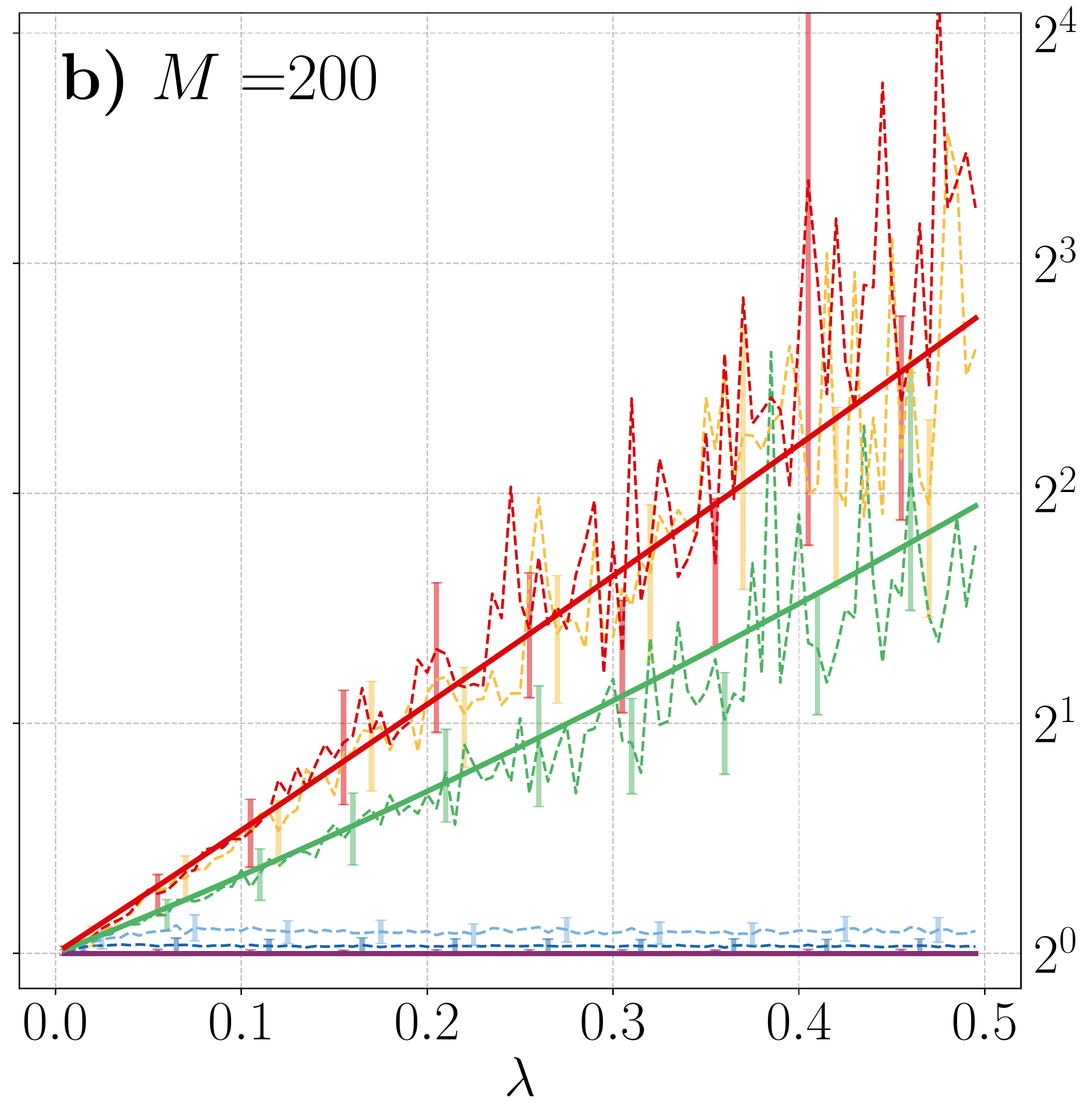}
    \includegraphics[width=0.275\textwidth]{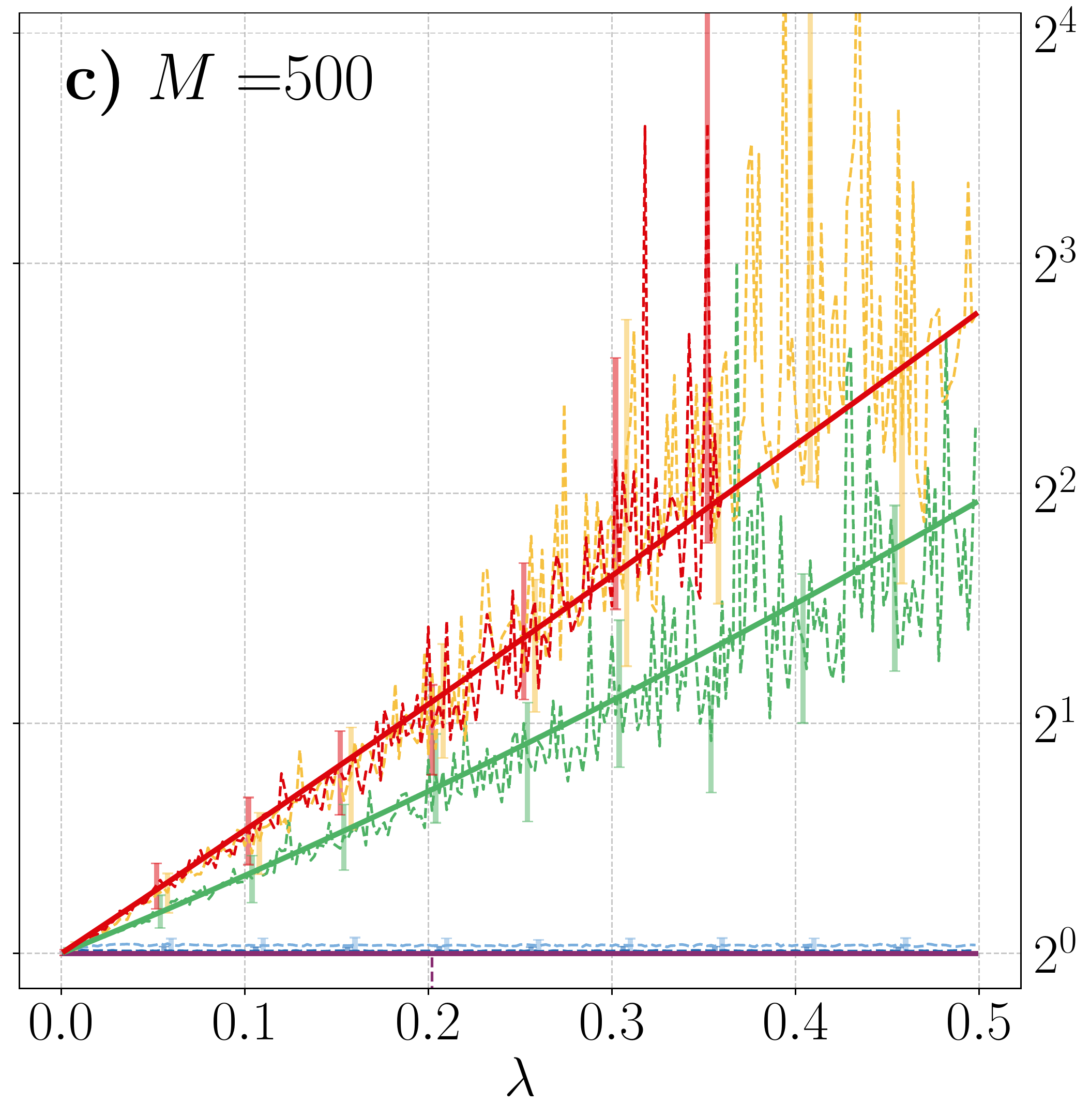}
    \caption{Ratio $\mathbb{E}[\Theta^2(x,x)]/\mathbb{E}^2[\Theta(x,x)]$ at initialization for fully-connected ReLU networks of constant width $M\in\{100,200,500\}$ with $\alpha_0=1$. The dashed lines represent the experimental results and the solid lines correspond to the theoretical predictions from Theorem \ref{th:NTK_dispersion_lim}. For each DNN configuration, we sampled 500 random initializations and computed an unbiased estimator for the ratio (see details in Appendix \ref{appendix:estimator}). The error bars (indicated by the vertical lines) show the bootstrap estimation of the standard error (only in a subset of points to keep the figure readable). We provide additional figures with continuous error bars in Appendix \ref{appendix:error_bars}.}
    \label{fig:NTK_var_lim}
\end{figure*}
\begin{equation}\label{eq:backprop}
\begin{split}
    \dfrac{\partial f(x)}{\partial \W_{ij}^\ell} &= \g_i^\ell\x^{\ell-1}_j,\quad \dfrac{\partial f(x)}{\partial \mathbf{b}_{i}^\ell} = \g_i^\ell, \\
    \g_i^\ell := \dfrac{\partial f(x)}{\partial \h_i^\ell} &= \phi^{'}(\h_i^\ell)\sum_j \g_j^{\ell+1}\W_{ji}^{\ell+1},
\end{split}
\end{equation}
\citet{schoenholz2016deep} studied the evolution of $\mathbb{E}[(\g_i^\ell)^2]$ along with $\mathbb{E}[(\x_i^\ell)^2]$ to find the distribution of DNNs' gradients. Some recent publications used these results to derive the properties of the infinite-width NTK in all the three phases of initialization \cite{karakida2018universal,xiao2019disentangling}. \citet{hayou2019mean} also showed that the infinite-\textit{depth} limit of the infinite-width NTK (when first the limit $M\to\infty$ is taken with fixed $L$ and then $L\to\infty$) yields a data-independent kernel and thus cannot explain properties of finite DNNs. Although our approach is different from the mentioned results since we do not assume infinite width before increasing depth, we show that the statistical properties of $\g^\ell$ and $\x^\ell$ can still be derived and lead to results on the NTK in our setting.

The initialization hyperparameters that comprise each phase differ depending on the chosen activation function. Since we are interested in ReLU networks, we note that the ordered phase corresponds to $\sigma_w^2<2$ and the chaotic phase --- to $\sigma_w^2>2$ for this activation function. The EOC is the initialization with $\sigma_w^2=2$. We refer, e.g., to \citet{schoenholz2016deep} for a method to compute the border between phases for a given activation function.

\section{Variability of the NTK}\label{section:var}

In the infinite-width limit, the NTK is deterministic under random initialization, which is one of the main results of the NTK theory. We investigate when this result holds outside of the NTK limit and, consequently, when the infinite-width behavior of the NTK gives a good approximation for realistic DNNs. 

\subsection{Infinite-depth-and-width limit}
Most results on the NTK assume that the network's depth is fixed as the width tends to infinity, i.e., $L/M \to 0$ in the limit. This setting, of course, does not describe deep finite-width networks since their depth-to-width ratio is bounded away from zero. Indeed, some recent works demonstrated that infinite-width approximations often get worse as the network's depth increases \citep{li2021future,matthews2018gaussian,yang2017mean}. In particular, \citet{hanin2019finite} considered this effect for the NTK and derived bounds for the ratio ${\mathbb{E}[\Theta^2(x,x)]}/{\mathbb{E}^2[\Theta(x,x)]}$ in case of ReLU DNNs initialized at the EOC ($\sigma_w=2$). This ratio characterizes the dispersion of the NTK: it is close to one if the NTK is approximately deterministic and is larger than two if the NTK's distribution is of high variance. Our first main result characterizes this ratio in the infinite-depth-and-width limit under different initializations:

\begin{figure*}
    \centering
    \includegraphics[width=0.42\textwidth]{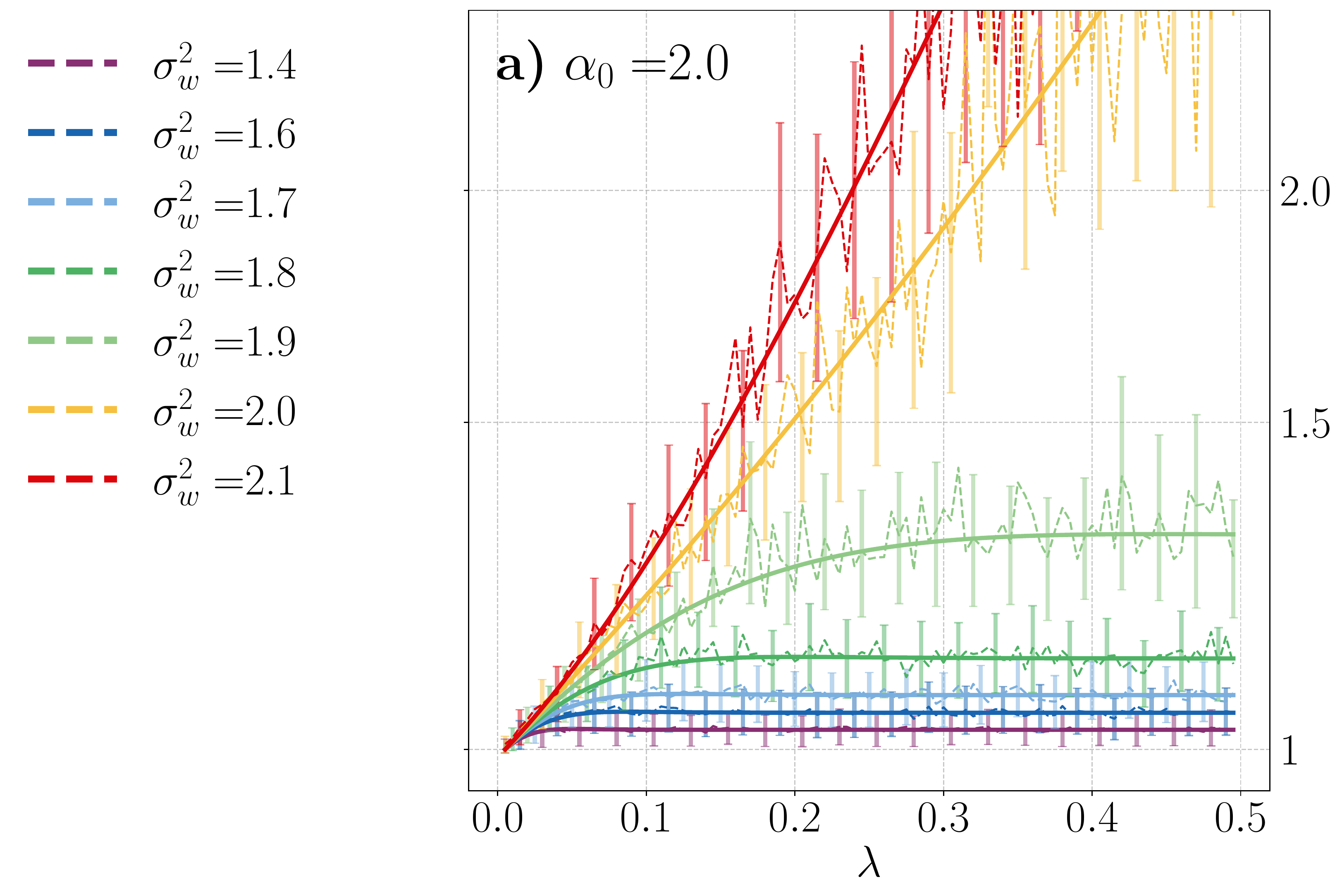}
    \includegraphics[width=0.275\textwidth]{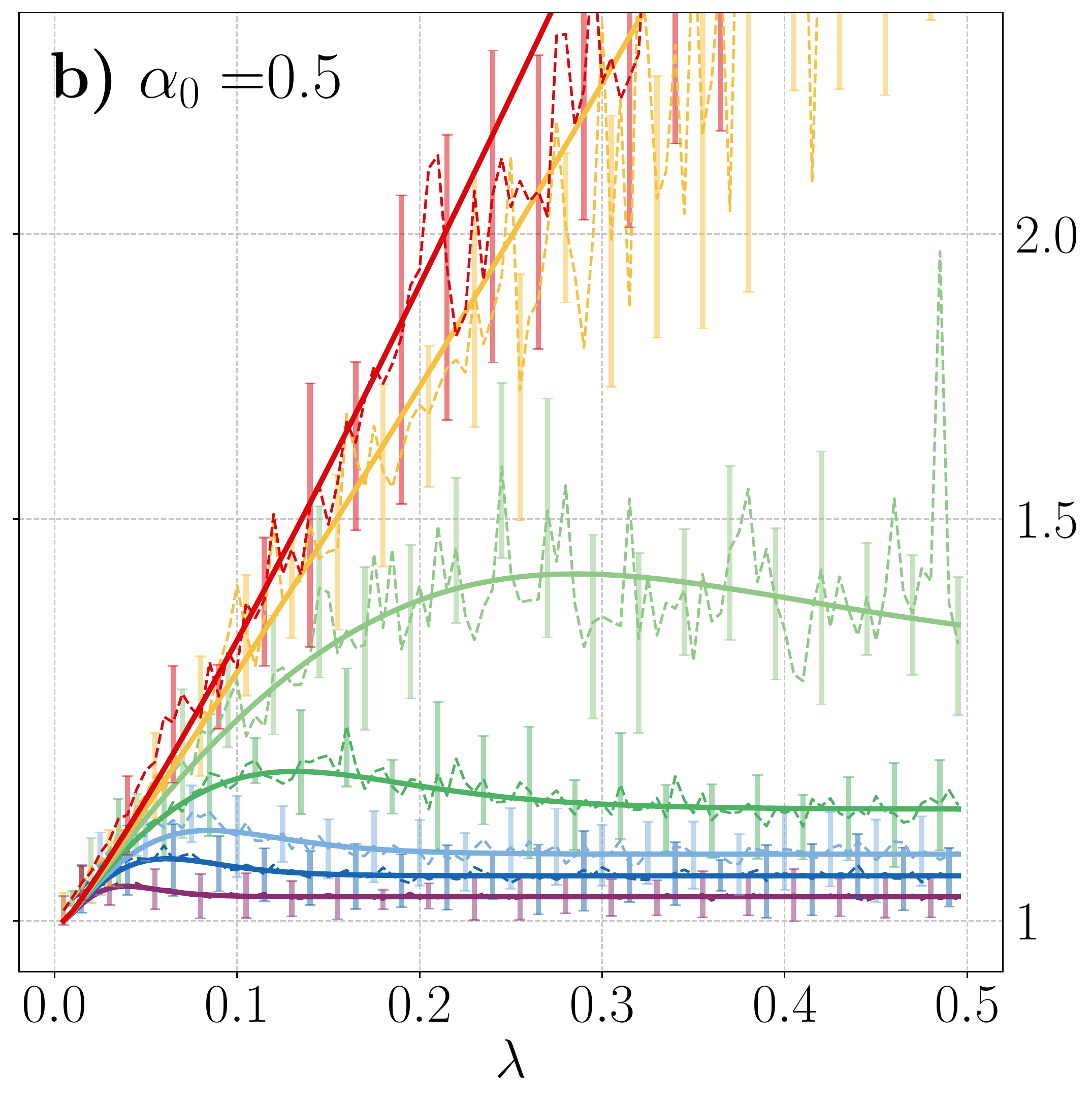}
    \includegraphics[width=0.275\textwidth]{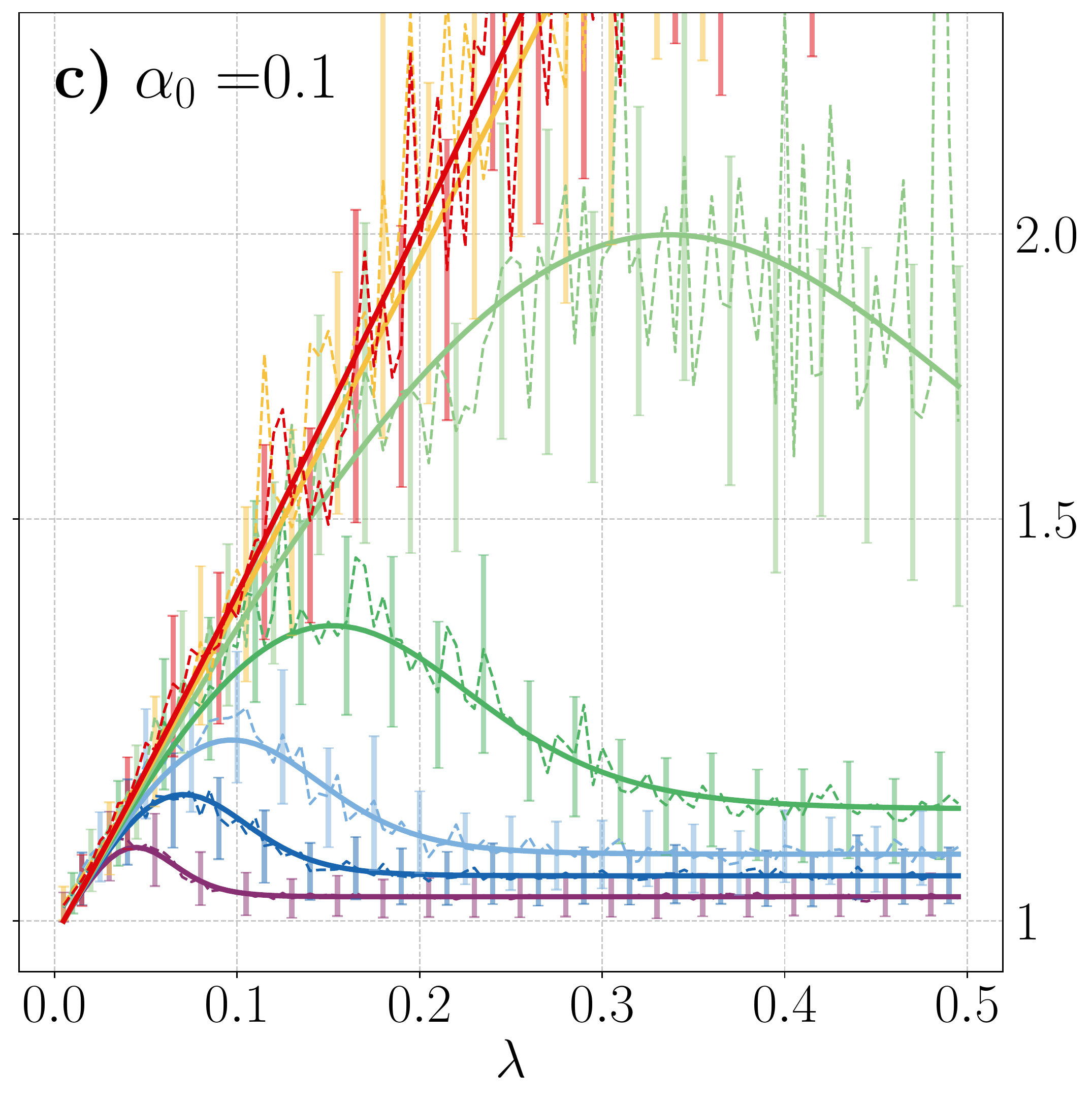}
    \caption{Ratio $\mathbb{E}[\Theta^2(x,x)]/\mathbb{E}^2[\Theta(x,x)]$ at initialization for fully-connected ReLU networks of constant width $M=200$ with the ratio $\alpha_0:=n_0/M \in\{2.0,0.5,0.1\}$. The initialization hyperparameter $\sigma_w^2$ is close to the EOC for all the lines. The dashed lines represent the experimental results (computed as described in Figure \ref{fig:NTK_var_lim}) and the solid lines show the theoretical predictions given by Theorem \ref{th:NTK_moments}. The error bars are shown only for a subset of points to keep the figure readable. We provide additional figures with continuous error bars in Appendix \ref{appendix:error_bars}.}
    \label{fig:NTK_var_eoc}
\end{figure*}

\begin{theorem}[Dispersion of the NTK at initialization in the limit]\label{th:NTK_dispersion_lim}
Consider a ReLU DNN as defined in \eqref{eq:nn} with constant width of hidden layers $M\in\mathbb{N}$, input dimension $n_0=\alpha_0 M$, $\alpha_0\in\mathbb{R}$ and output dimension $n_L=1$. The initialization is given by \eqref{eq:init} and the biases are initialized to zero, i.e. $\sigma_b=0$. Then, in the infinite-depth-and-width limit $M\to\infty,\ L\to\infty,\ L/M \to \lambda\in\mathbb{R}$, the following holds for the dispersion of the NTK: 

    1. In the \textbf{chaotic phase} ($a:=~{\sigma_w^2}/{2} > 1$), the NTK dispersion grows exponentially with depth-to-width ratio $\lambda:={L}/{M}$ as follows:
    \begin{equation}
    \begin{split}
        \dfrac{\mathbb{E}[\Theta^2(x,x)]}{\mathbb{E}^2[\Theta(x,x)]}\xrightarrow[]{} \dfrac{1}{2\lambda} e^{5\lambda} \Biggl( 1 - \dfrac{1}{4\lambda}(1 - e^{-4\lambda}) \Biggr).
    \end{split}
    \end{equation}
    
    2. At the \textbf{EOC} ($a=1$), the NTK dispersion grows exponentially with depth-to-width ratio $\lambda$ as well, but with a slower rate given by:
    \begin{equation}
    \begin{split}
        \dfrac{\mathbb{E}[\Theta^2(x,x)]}{\mathbb{E}^2[\Theta(x,x)]} \to  \dfrac{1}{(1+\alpha_0)^2} \Biggl[ e^{5\lambda}\Bigl(\dfrac{1}{2\lambda} + \dfrac{2\alpha_0^2-8\alpha_0}{25\lambda^2}\Bigr)\\
        + (e^\lambda - e^{5\lambda}) \dfrac{1-4\alpha_0}{8\lambda^2}+
        \dfrac{2\alpha_0}{5\lambda}\Bigl( \dfrac{4-\alpha_0}{5\lambda} - 1-\alpha_0\Bigr)\Biggr].
    \end{split}
    \end{equation}
    
    3. In the \textbf{ordered phase} ($a < 1$), the NTK dispersion tends to one:
    \begin{equation}
        \dfrac{\mathbb{E}[\Theta^2(x,x)]}{\mathbb{E}^2[\Theta(x,x)]} \to  1.
    \end{equation}

\end{theorem}

Our numerical experiments in Figure \ref{fig:NTK_var_lim} demonstrate that Theorem \ref{th:NTK_dispersion_lim} provides accurate approximations for the behavior of sufficiently deep and wide DNNs. Indeed, the proofs listed in Appendix \ref{proofs:NTK_var} show that the expressions in the above theorem are true up to the approximation given by $(1+c/M)^L\approx e^{c\lambda}$ and $O(1/\sqrt{M})$ in the coefficients of the exponents in case of finite width and depth.

\begin{remark}
The EOC expression in Theorem~\ref{th:NTK_dispersion_lim} tends to the chaotic phase expression if $\alpha_0:= n_0/M$ tends to zero (i.e. when the input dimension is fixed). We discuss this effect in Appendix \ref{appendix:eoc_a0}.
\end{remark}
\begin{remark}
Model scaling introduced in papers on the so-called "lazy training" phenomenon \cite{chizat2019lazy} does not change the results of Theorem \ref{th:NTK_dispersion_lim}. We discuss lazy training and its effects on our analysis in Appendix \ref{appendix:lazy_training}.
\end{remark}

\subsection{Finite depth and width effects}\label{section:finite-width}

We notice that some features of the NTK dispersion are still not visible in the infinite-depth-and-width limit. One can see in Figure \ref{fig:NTK_var_lim} that the NTK variance in the ordered phase is not exactly zero for finite-width DNNs, contrary to the prediction in the limit. This is especially noticeable for initialization close to the EOC, where the transition between the two kinds of limiting behavior occurs. 
Moreover, Theorem \ref{th:NTK_dispersion_lim} cannot reveal the effects of the architecture since it considers only DNNs of constant width. Therefore, we provide non-asymptotic expressions for the first two moments of the NTK at initialization in the following theorem and show that these expressions accurately describe the behavior of finite-width DNNs.
\begin{theorem}[Moments of the NTK at initialization]\label{th:NTK_moments}
Consider a ReLU DNN defined in \eqref{eq:nn} with widths scaling defined in \eqref{eq:scaling} and the output dimension $n_L=1$. The initialization is given by \eqref{eq:init} and $\sigma_b=0$. Then the expectation of the NTK is determined by the following terms:
\begin{equation}
    \mathbb{E}[\Theta_W(x,x)] = \|\x^0\|^2 a^{L-1} \sum_{\ell=1}^L \dfrac{n_{\ell-1}}{n_0},
\end{equation}
\vspace{-1.5ex}
\begin{equation}
    \mathbb{E}[\Theta_b(x,x)] = \sum_{\ell=1}^L a^{L-\ell},
\end{equation}
where the NTK components $\Theta_W$ and $\Theta_b$ are defined in \eqref{eq:ntk}. Moreover, the second moment of the NTK is determined by:
\begin{equation}
\begin{split}
    \qquad\qquad\qquad\mathllap{\dfrac{\mathbb{E}[\Theta_W^2(x,x)]}{\|\x^0\|^4 a^{2(L-1)}} }  & =\mathcal{X}_{(1,L)} \Biggl[ \sum_{\ell=1}^L \dfrac{n_{\ell-1}^2}{n_{0}^2} \\
    &+ \sum_{\ell_1 < \ell_2}\dfrac{n_{\ell_2-1}n_{\ell_1-1}}{n_{0}^2} \dfrac{\mathcal{C}_{(\ell_1,\ell_2)}}{\mathcal{X}_{(\ell_1,\ell_2)}} \Biggr],
\end{split}
\end{equation}
\begin{equation}
\begin{split}
    \qquad\qquad\qquad\mathllap{\dfrac{\mathbb{E}[\Theta_b(x,x)^2]}{a^{2L}}}  = \sum_{\ell=1}^L & \dfrac{\mathcal{X}_{(\ell,L)}}{a^{2\ell}} + 2 \sum_{\ell_1<\ell_2} \dfrac{\mathcal{X}_{(\ell_2,L)}}{a^{\ell_1+\ell_2}},
\end{split}
\end{equation}
\vspace{-3ex}
\begin{equation}
\begin{split}
    &\dfrac{\mathbb{E}[\Theta_W(x,x)\Theta_b(x,x)]}{\|\x^0\|^2a^{2L-1} }
      = \sum_{\ell=1}^L \dfrac{n_{\ell-1}}{n_0} \dfrac{\mathcal{X}_{(\ell,L)}}{a^{\ell}} \\
    & +  \sum_{\ell_1 < \ell_2} \dfrac{\mathcal{X}_{(\ell_2,L)}}{a^{\ell_1}}\dfrac{n_{\ell_1-1}}{n_0} \Biggl(\dfrac{n_{\ell_2-1}}{n_{\ell_1-1}} \mathcal{C}_{(\ell_1,\ell_2)} + \dfrac{a^{\ell_1}}{a^{\ell_2}} \Biggr),
\end{split}
\end{equation}
where we denoted $\mathcal{X}_{(i,j)}: =\prod_{k=i}^{j-1} \Bigl(1 + \frac{5}{n_k} + O\bigl({M^{-3/2}}\bigr) \Bigr)$, $\mathcal{C}_{(i,j)}: =\prod_{k=i}^{j-1} \Bigl(1 + \frac{1}{n_k} + O\bigl({M^{-3/2}}\bigr) \Bigr)$ and $a:=\sigma_w^2/2$.
\end{theorem}
These expressions are derived in Appendix \ref{proofs:NTK_var} as a part of the proof of Theorem \ref{th:NTK_dispersion_lim} and they simplify to the results in the limit by noticing that $\mathcal{X}_{(1,L)}\to e^{5\lambda}$ and $\mathcal{C}_{(1,L)}\to e^{\lambda}$.

Figure \ref{fig:NTK_var_eoc} examines how well the above expressions approximate the NTK of DNNs with varying ratios $\alpha_0:=n_0/M$ between the input dimension and the width of hidden layers. One can see that the NTK variance in the ordered phase indeed grows as the initialization approaches the EOC. This effect is due to the terms proportional to $\bigl((a-1)M\bigr)^{-1}$ in the moments of $\Theta_b(x,x)$ and $\Theta_W(x,x)\Theta_b(x,x)$. When the initialization is close enough to the EOC, $(a-1)^{-1}$ becomes comparable with finite $M$, and therefore the behavior diverges from the limit. 

\begin{figure*}
    \centering
    \includegraphics[width=\textwidth]{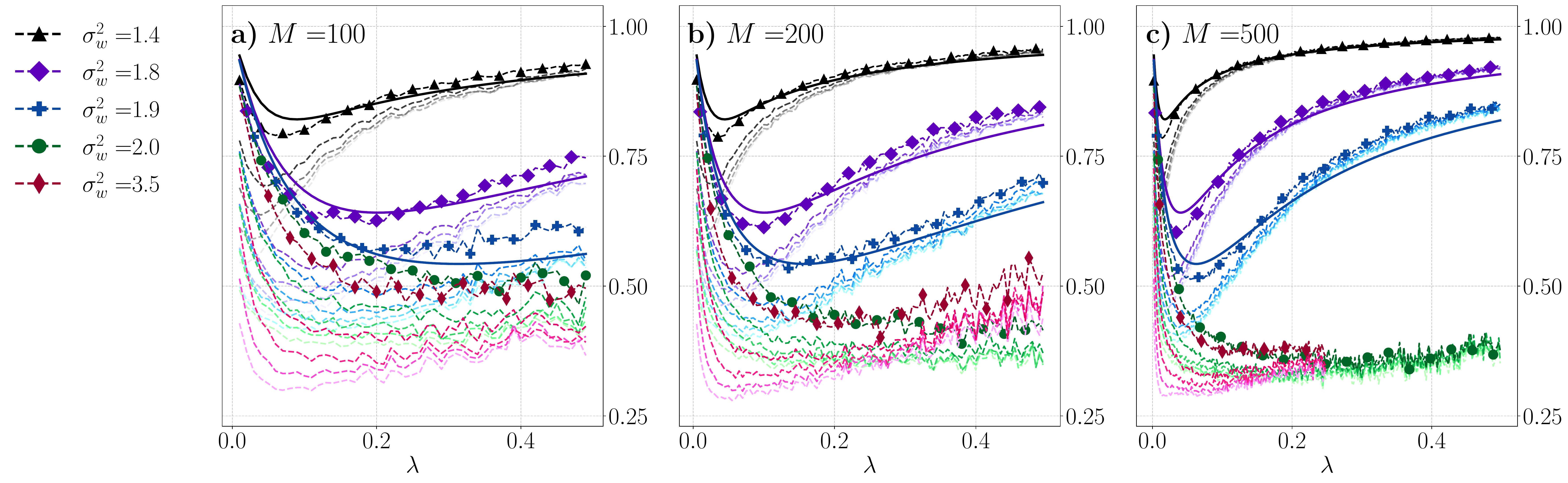}
    \caption{Ratio $\mathbb{E}[\Theta(x,\Tilde{x})]/\mathbb{E}[\Theta(x,x)]$ at initialization for fully-connected ReLU networks of constant width $M\in\{100,200,500\}$ with $\alpha_0=1$. Colors and markers indicate different values of $\sigma_w^2$. There are 5 dashed lines for each $\sigma_w^2$ value, which correspond to 5 values of the initial angle between input samples $\langle \x^0,\Tilde{\x}^0\rangle \in \{0.1,0.3,0.5,0.7,0.9\}$. Darker lines (which also display larger values of the ratio of interest) correspond to larger product $\langle \x^0,\Tilde{\x}^0\rangle$. Expectations are computed by sampling 500 random initializations of each DNN configuration. The solid lines show the estimate for the ratio of interest given by \eqref{eq:non_diag_ord} in the ordered phase.}
    \label{fig:non_diag}
\end{figure*}

Another remarkable observation is that the NTK dispersion may decrease with depth in the ordered phase for DNNs that increase the dimensionality (i.e. $n_0\leq n_1 \leq \dots n_{L-1}$), which means that deeper networks can be more robust. Indeed, in Subfigures b) and c) of Figure \ref{fig:NTK_var_eoc}, the dispersion reaches its peak at a certain depth and then decreases. We provide additional results characterizing this effect in DNNs with non-constant width in hidden layers in Appendix \ref{appendix:architecture}.

\subsection{Non-diagonal elements of the NTK}\label{section:non-diag}
The results stated so far only concern the diagonal elements of the NTK. To generalize to the whole kernel, we provide the following theorem proven in Appendix \ref{appendix:non-diag}:

\begin{theorem}[Non-diagonal elements of the NTK]\label{th:non-diag}
Consider a ReLU DNN from Theorem \ref{th:NTK_moments}. The following bounds hold for the ratio of non-diagonal and diagonal elements of the NTK:
\begin{equation}
    1\geq \lim_{\substack{L\to\infty,M\to\infty\\L/M\to\lambda\in\mathbb{R}}} \dfrac{\mathbb{E}[\Theta(x,\Tilde{x})]}{\mathbb{E}[\Theta(x,x)]} \geq \dfrac{1}{4}.
\end{equation}
Moreover, the dispersion of the non-diagonal elements is bounded by the dispersion of diagonal ones:
\begin{equation}\label{eq:non_diag_disp}
     \lim_{\substack{L\to\infty\\M\to\infty\\L/M\to\lambda}}\dfrac{\mathbb{E}[\Theta^2(x,\Tilde{x})]}{\mathbb{E}^2[\Theta(x,\Tilde{x})]} \leq 16\lim_{\substack{L\to\infty\\M\to\infty\\L/M\to\lambda}} \dfrac{\mathbb{E}[\Theta^2(x,x)]}{\mathbb{E}^2[\Theta(x,x)]}.
\end{equation}
\end{theorem}
Of course, the bound in \eqref{eq:non_diag_disp} is too loose for practical applications if the goal is to prove that the NTK is approximately deterministic. However, we note that the ratio of non-diagonal and diagonal elements can be close to the lower bound only in the chaotic phase. In the ordered phase, our proof suggests the following bound for sufficiently wide and deep networks: 
\begin{equation}\label{eq:non_diag_ord}
    \dfrac{\mathbb{E}[\Theta(x,\Tilde{x})]}{\mathbb{E}[\Theta(x,x)]} \gtrsim \dfrac{\sum_{\ell=1}^La^{L-\ell}\prod_{k=\ell}^{L-1}g(\rho_{k-1})}{\sum_{\ell=1}^La^{L-\ell}},
\end{equation}
where $g(t):= \frac{1}{\pi}(\pi/2+\arcsin t)$ and $\rho_{k}$ is the infinite-width approximation of the cosine distance between $\x^k$ and $\Tilde{\x}^k$, which only increases with depth and is given by applying the function $r(t)~:=~\frac{1}{\pi}\bigl(\sqrt{1-t^2}~+~t\pi/2~+~t\arcsin t\bigr)$ to $\langle \x^0,\Tilde{\x}^0\rangle$ consecutively $k$ times. The function $r(\cdot)$ arises from the expectation of a product of two correlated Gaussian variables under ReLU function. 

We provide empirical results on the ratio of non-diagonal and diagonal elements of the NTK in Figure \ref{fig:non_diag}. We also plot the estimate for the ratio given by \eqref{eq:non_diag_ord} in the same figure. One can see that the ratio quickly increases with depth in the ordered phase. Moreover, the lower bound in \eqref{eq:non_diag_ord} gives a good approximation for the experimental results. Then for a given network in the ordered phase one can replace the coefficient $16$ in the bound \eqref{eq:non_diag_disp} with $1/c^2$, where $c$ is a better estimate for the lower bound of $\mathbb{E}[\Theta(x,\Tilde{x})]/\mathbb{E}[\Theta(x,x)]$ and can be close to one in the ordered phase.

We also provide experiments on the dispersion of the non-diagonal elements in Appendix \ref{appendix:non_diag_exp}. Our results indicate that, in practice, the dispersion here is only slightly higher than the prediction for the diagonal elements. The general picture stays the same as in Figure \ref{fig:NTK_var_lim}: the dispersion is low and does not grow with depth in the ordered phase but increases exponentially with the depth-to-width ratio at the EOC and in the chaotic phase. The finite-width effects represented in Figure \ref{fig:NTK_var_eoc} also remain the same for the non-diagonal elements.

\subsection{Proof ideas}\label{section:proof-ideas}
All our proofs are based on the following decomposition of the NTK:
\begin{equation}
    \Theta(x,x) = \sum_{\ell=1}^L \|\g^\ell(x)\|^2 \bigl(\|\x^{\ell-1}(x)\|^2 +1\bigr),
\end{equation}
which directly follows from \eqref{eq:ntk} and the representation of the gradients in backpropagation \eqref{eq:backprop}. Using forward-propagation equations \eqref{eq:nn} and backpropagation equations \eqref{eq:backprop}, we derive the first two moments for the ratios $\mathcal{N}_x^\ell:=~\|\x^{\ell}\|^2/\|\x^{\ell-1}\|^2$ and $\mathcal{N}_\delta^\ell:=\|\g^{\ell}\|^2/\|\g^{\ell+1}\|^2$ in Lemmas \ref{lemma:x_ratio} and \ref{lemma:d_ratio}. We then notice that $\mathcal{N}_x^\ell$ are uncorrelated in different layers of the networks, as well as $\mathcal{N}_\delta^\ell$, while $\mathcal{N}_x^\ell$ and $\mathcal{N}_\delta^\ell$ in the same layer can be weakly correlated and we quantify the effects of this dependence in Lemma~\ref{lemma:gia}. Given the moments of $\mathcal{N}_x^\ell$ and $\mathcal{N}_\delta^\ell$ and the results on their correlations, we can represent summands of the NTK as the following telescopic products:
\begin{equation}
    \|\g^\ell\|^2 \|\x^{\ell-1}\|^2 = \|\x^0\|^2\|\g^L\|^2 \prod_{k=1}^{\ell-1} \mathcal{N}_x^k \prod_{p=\ell}^{L-1} \mathcal{N}_\delta^p
\end{equation}
and use this decomposition to compute the expectation and the second moment of the NTK. We derive the first two moments for $\Theta_W(x,x)$ and $\Theta_b(x,x)$ separately in Lemmas \ref{lemma:theta_w} and \ref{lemma:theta_b}. These two components have very different properties in the infinite-depth-and-width limit and, as we show in the proof of Theorem~\ref{th:NTK_dispersion_lim}, the behavior of the NTK is determined by $\Theta_W(x,x)$ in the chaotic phase and by $\Theta_b(x,x)$ in the ordered phase. We also derive the expectation of $\Theta_W(x,x)\Theta_b(x,x)$ in Lemma~\ref{lemma:theta_wb} to complete the calculations of the second moment of the NTK.

 We note that many papers on the NTK use the so-called gradient independence assumption (GIA), which leads to the independence of $\mathcal{N}_x^\ell$ and $\mathcal{N}_\delta^\ell$. This assumption often leads to correct results in the infinite-width limit, as discussed in \citet{yang2019scaling}. However, in our case of infinite depth and width, it may have a non-negligible effect even for simple fully-connected networks with all the weights initialized independently. Thus, we have to calculate this effect explicitly in our proofs. We also note that \citet{li2021future} used a similar technique involving telescoping products of weakly-correlated variables to derive the distribution of the activation norms of ResNets.

\section{Training dynamics of the NTK}\label{section:train}
\begin{figure*}
    \centering
    \includegraphics[width=0.24\textwidth]{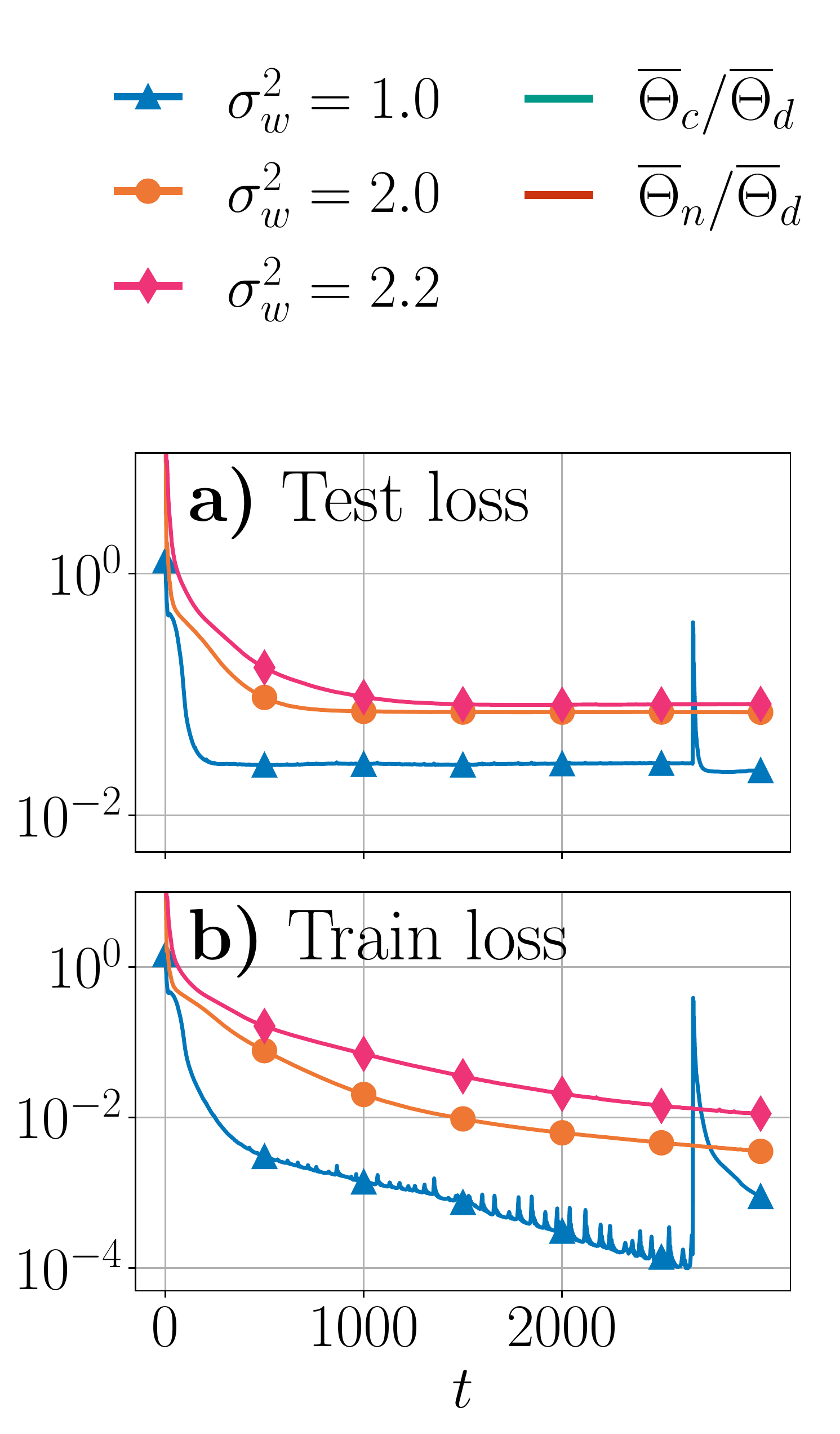}
    \includegraphics[width=0.7\textwidth]{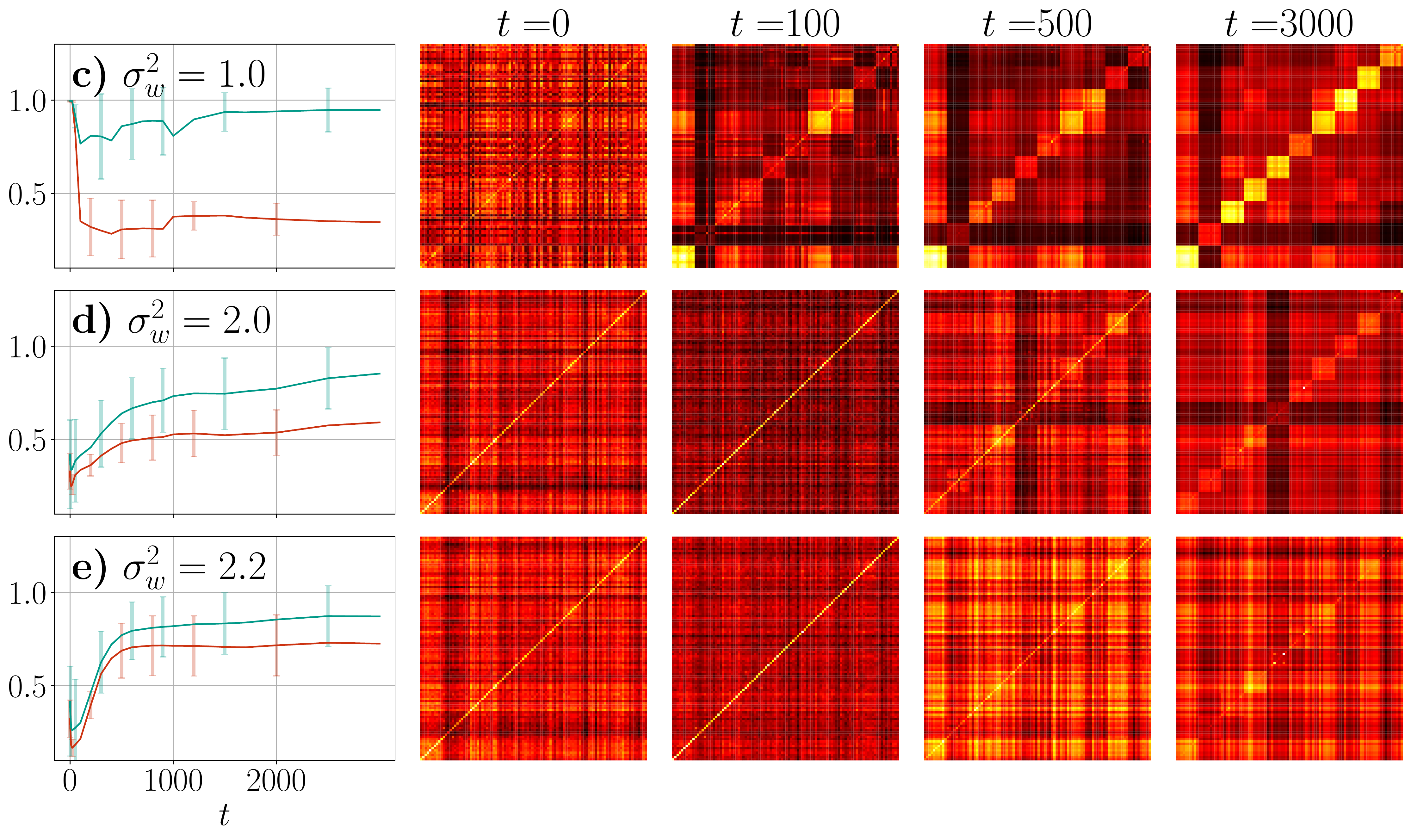}
    \caption{Structure of the NTK matrix in different stages of training for fully-connected ReLU DNNs with $L=20$ and $M=300$. The DNNs are initialized with $\sigma_w^2\in\{1.0,2.0,2.2\}$ and trained on MNIST using Adam algorithm with learning rate $10^{-5}$. Subplots \textbf{a)} and \textbf{b)} show the test and the train loss achieved by each DNN. Subplots \textbf{c)}, \textbf{d)} and \textbf{e)} characterize label-awareness of the NTK. Variables $\overline{\Theta}_d$, $\overline{\Theta}_c$ and $\overline{\Theta}_n$ are defined in \eqref{eq:structure}. The heatmaps show the NTK matrix on MNIST subsample of size $100$ at epoch $t\in\{0,100,500,3000\}$. The subsample contains $10$ elements of each class and is arranged so that consecutive diagonal blocks of size $10$ contain pairwise NTK values on each class. The color range in the heatmaps is adjusted to include the interval between the maximal and the minimal values of the NTK in a given epoch, i.e. the colors correspond to different values for different epochs. Brighter colors indicate larger values.}
    \label{fig:NTK_structure}
\end{figure*}
In the infinite-width limit, the NTK stays constant during training, which allows to study the gradient flow dynamics of infinitely-wide DNNs analytically. In this section, we discuss when this result holds outside of the infinite-width limit and how the empirical NTK changes during training.

\subsection{The first GD step}\label{section:gd}
\citet{hanin2019finite} proved that the NTK of overparametrized fully-connected ReLU networks initialized at the EOC can evolve non-trivially during GD training if depth and width of the network are comparable. In particular, their result bounds the relative change of the diagonal elements of the NTK $\Theta(x,x)$ in the first GD step carried out on a single sample $x$ above and below by an exponential function of the depth-to-width ratio $\lambda$. We generalize this result to different initializations with the following theorem proven in Appendix \ref{appendix:train}:

\begin{theorem}[GD step of the NTK]\label{theorem:NTK_gd}
Consider a ReLU DNN from Theorem \ref{th:NTK_moments}. A single GD update on a sample $(x,y)\in\mathcal{D}$ results in the following changes of the NTK:
\begin{enumerate}
    \item In the \textbf{chaotic phase} ($a:={\sigma_w^2}/{2} > 1$), the changes to the NTK value are infinite in the limit for a constant learning rate $\eta\in\mathbb{R}$:
     \begin{equation}
            \dfrac{\mathbb{E}[\Delta{\Theta}(x,x)] }{\mathbb{E}[{\Theta}(x,x)] }  \xrightarrow[\substack{M\to\infty, L\to\infty, \\L/M \to \lambda\in\mathbb{R}}]{} \infty.
        \end{equation}
    \item In the \textbf{ordered phase} ($a < 1$), the NTK stays constant in the limit:
       \begin{equation}
            \dfrac{\mathbb{E}[\Delta{\Theta}(x,x)] }{\mathbb{E}[{\Theta}(x,x)] }  \xrightarrow[\substack{M\to\infty, L\to\infty, \\L/M \to \lambda\in\mathbb{R}}]{} 0.
        \end{equation}
\end{enumerate}
\end{theorem}
This result shows that deep networks can potentially behave according to the NTK theory during GD training only in case of initialization in the ordered phase. We refer to experiments in \citet{seleznova2020analyzing}, which confirm that the relative change of the NTK during training on MNIST is significant and grows with depth in the chaotic phase and at the EOC but not in the ordered phase. However, it is unclear how to generalize this result to realistic scenarios of DNN training, which include randomly selected batches of arbitrary size and optimization algorithms beyond vanilla GD. Our experiments in the next subsection show that the NTK evolution is in general non-trivial even in the ordered phase.

\begin{remark}
Deep networks rescaled as in \citet{chizat2019lazy} can exhibit lazy training (with random NTK at initialization) in the chaotic phase only if the scaling parameter grows exponentially with depth $L$. We discuss the lazy training phenomenon and its effects on our results in Appendix \ref{appendix:lazy_training}.
\end{remark}

\subsection{Changes of the NTK structure}\label{section:structure}
 The NTK at initialization is label-agnostic, i.e. its value on a pair $(x,\Tilde{x})$ is independent of whether the labels of $x$ and $\Tilde{x}$ are the same or not. Clearly, label-agnostic features cannot provide an optimal representation system for an arbitrary task and many authors studied the benefits of adding label information to kernels \cite{shawe2002kernel,gonen2011multiple,tishby2015deep}. In particular, \citet{chen2020label} argued that label-agnosticism can explain the performance gap between trained DNNs and the NTK and demonstrated that adding label-awareness improves the performance of the infinite-width NTK. Thus, it is important to characterize label-awareness of the empirical NTK and how the training process leads to it to understand the properties of DNNs.

We saw in Section \ref{section:non-diag} that the NTK at initialization has an approximately diagonal structure with the diagonal values larger than the non-diagonal ones. On the contrary, the "optimal kernel" for a classification task would be block-diagonal with blocks of larger values corresponding to samples of the same class. Thus, we expect the NTK to naturally change towards the block-diagonal structure during the training process. Our experiments in Figure~\ref{fig:NTK_structure} confirm this intuition in a simple setting of fully-connected ReLU networks trained on MNIST. Let us define the following variables that characterize label-awareness of the NTK matrix:
\begin{equation}\label{eq:structure}
\begin{split}
    \overline{\Theta}_d &:= \dfrac{1}{|\mathcal{X}|}\sum_{x\in\mathcal{X}}\Theta(x,x),\\
    \overline{\Theta}_c &:= \dfrac{1}{K}\sum_{k=1}^K \dfrac{1}{|\mathcal{X}_k|(|\mathcal{X}_k|-1|)}\sum_{\substack{x_i\neq x_j,\\x_i,x_j\in \mathcal{X}_k}}\Theta(x_i,x_j), \\
    \overline{\Theta}_n &:= \dfrac{1}{K}\sum_{k=1}^K \dfrac{1}{|\mathcal{X}_k|(|\mathcal{X}|-|\mathcal{X}_k|)}\sum_{\substack{x_i\in\mathcal{X}_k,\\x_j\not\in \mathcal{X}_k}}\Theta(x_i,x_j),
\end{split}
\end{equation}
where $\mathcal{X}=\cup_{k=1}^K \mathcal{X}_k$ is the decomposition of the dataset $\mathcal{X}$ into $K$ classes. Then $\overline{\Theta}_d$ is the mean diagonal value, $\overline{\Theta}_c$ is the mean value of the NTK on samples from the same class and $\overline{\Theta}_n$ is the mean value on samples from different classes. Figure \ref{fig:NTK_structure} suggests that a larger gap between $\overline{\Theta}_n/\overline{\Theta}_d$ and $\overline{\Theta}_c/\overline{\Theta}_d$ may be related to better performance of DNNs. Moreover, the gap between $1$ and the ratio $\overline{\Theta}_c/\overline{\Theta}_d$ may characterize overfitting. Therefore, we believe that the structure of the NTK can be a proxy for generalization of DNNs even outside of the NTK regime. One can also see that the structure of the NTK changes more rapidly in the early stages of training, which is coherent with the conclusion in \citet{fort2020deep} that useful features are mostly learned in the first epochs of training. Thus, dynamics of the NTK may provide information about the state of the training process.

\section{Conclusions and future work}
This paper adds to the line of research on the statistical properties of the NTK and the correspondence between finite-width DNNs and their infinite-width approximations. Our results in Section \ref{section:var} precisely quantify variability of the NTK at initialization for a given fully-connected ReLU DNN and assess how well the kernel is approximated by its infinite-width limit. Combining our findings from Section \ref{section:var} with the results on the GD update of the NTK in Section \ref{section:gd}, we conclude that the NTK regime can approximate trained networks with non-trivial depth-to-width ratio only in the ordered phase. At the same time, the behavior of overparametrized DNNs outside of the NTK regime is very poorly understood so far. It is unclear how to characterize DNNs' training dynamics in the general case and what role the properties of the (random and dynamic) NTK play here. We make a step into this direction in Section \ref{section:structure} by demonstrating how the NTK acquires a block-diagonal structure during training. We believe that precisely characterizing the effects of this NTK structure on the generalization of DNNs is a promising direction for future work. In general, we hope to establish new connections between the NTK and other aspects of DNN training outside of the NTK regime.

\section*{Acknowledgements}
G.K. acknowledges partial support by the NSF–Simons Research Collaboration on the Mathematical and Scientific Foundations of Deep Learning (MoDL) (NSF DMS 2031985) and DFG SPP 1798, KU 1446/27-2 and KU 1446/21-2.

\bibliography{references}
\bibliographystyle{icml2022}

\newpage
\appendix
\onecolumn
\section{Proofs}
\subsection{Variability of the NTK at initialization}\label{proofs:NTK_var}

\begin{lemma}[Forward-propagation of variance]\label{lemma:x_ratio}{Consider a fully-connected DNN defined in \eqref{eq:nn} initialized as in \eqref{eq:init}. The activation function in the hidden layers is ReLU, i.e. $\phi(x)= x \mathbbm{1}\{x>0\}$. Assume further that the biases are initialized to zero, i.e. $\sigma_b=0$. Then the following holds for the ratios of the activation norms in consecutive layers of the network, denoted $\mathcal{N}^{\ell}_x := {\|\x^\ell\|^2}/{\|\x^{\ell-1}\|^2}$, $\ell=1,\dots,L-1$:

\begin{equation}
    \mathbb{E}[\mathcal{N}^{\ell}_x] =  \dfrac{\sigma_w^2}{2}\dfrac{n_\ell}{n_{\ell-1}}, \quad \mathbb{E}[(\mathcal{N}^{\ell}_x)^2] =  \Bigl(\dfrac{\sigma_w^2}{2}\Bigr)^2\Bigl(\dfrac{n_\ell}{n_{\ell-1}}\Bigr)^2\Bigl(1 + \dfrac{5}{n_\ell} \Bigr).
\end{equation}

\begin{equation}
    \dfrac{\mathcal{N}^{\ell}_x - \mathbb{E}[\mathcal{N}^{\ell}_x]}{\sqrt{\mathbb{V}[\mathcal{N}^{\ell}_x]}} \xrightarrow[n_\ell\to\infty]{d} \mathcal{N}(0,1),
\end{equation}}
where $\mathcal{N}(0,1)$ is the standard normal distribution. Moreover, random variables $\{\mathcal{N}_x^{\ell}\}_{\ell=0,\dots,L-1}$ are mutually independent.
\end{lemma}
\begin{proof}
The squared norm of the activation vector in layer $\ell$ is given by 
\begin{equation*}
\begin{split}
    \|\x^{\ell}\|^2= \sum_{i=1}^{n_{\ell}}\phi^2(\W_{i\cdot}^{\ell}\x^{\ell-1}+b_i^{\ell})
\end{split}
\end{equation*}
Here $\x^{\ell-1}$ depends only on $\{(\W^j,\bias^j)\}_{j=1,\dots \ell-1}$, therefore $\x^{\ell-1}$ is independent of $(\W^{\ell},\bias^{\ell})$. Since elements of $\W^\ell$ are i.i.d Gaussian, the distribution of $\W^\ell\x^{\ell-1}$ depends only on the norm of $\x^{\ell-1}$ and not on the direction. Then we can write the following equalities in distribution: 
\begin{equation*}
\begin{split}
    \W^{\ell}_{i\cdot}\x^{\ell-1} &{=} \sqrt{\dfrac{\sigma_w^2}{n_{\ell-1}}}\|\x^{\ell-1}\| \mathcal{U}^\ell_i,\\
    \phi^2\Bigl(\W^{\ell}_{i\cdot}\x^{\ell-1} + b_i^{\ell}\Bigr) &{=} \Bigl(\sqrt{\dfrac{\sigma_w^2}{n_{\ell-1}}}\|\x^{\ell-1}\| +\sigma_b\Bigr)^2\phi^2(\mathcal{U}^\ell_i),
\end{split}
\end{equation*}
where we introduced i.i.d random variables $\mathcal{U}^\ell_i {=} \Bigl\langle \sqrt{\frac{n_{\ell-1}}{\sigma_w^2}}(\W^{\ell}_{i\cdot})^T ,\frac{\x^{\ell-1}}{\|\x^{\ell-1}\|} \Bigr\rangle \sim\mathcal{N}(0,1)$, $i=1,\dots,n_{\ell}$, which are independent of $\x^{\ell-1}$, and used the fact that $\phi(\alpha x)= \alpha \phi(x)$ for $\alpha \in \mathbb{R}^{+}$. Therefore for the norm of the activation vector we have the following:
\begin{equation*}
\|\x^{\ell}\|^2= \sum_{i=1}^{n_{\ell}}\phi^2(\W_{i\cdot}^{\ell}\x^{\ell-1}+b_i^{\ell}) {=} \Bigl(\sqrt{\dfrac{\sigma_w^2}{n_{\ell-1}}}\|\x^{\ell-1}\| +\sigma_b\Bigr)^2 \sum_{i=1}^{n_{\ell}} \phi^2(\mathcal{U}^\ell_i),
\end{equation*}
where only the first bracket depends on $\x^{\ell-1}$. 
Then in case of zero biases, i.e. $\sigma_b=0$, for the ratio between the norms of consecutive activation vectors we  have 
\begin{equation*}
    \mathcal{N}_x^\ell {=} \dfrac{\sigma_w^2}{n_{\ell-1}} \sum_{i=1}^{n_{\ell}} \phi^2(\mathcal{U}^\ell_i),
\end{equation*}
where the variables $\mathcal{U}_i^\ell, i=1,\dots,n_\ell$ depend only on the weights in the given layer $\W^\ell$. Then the ratios $\mathcal{N}_x^\ell$ in different layers are independent and we can obtain the desired moments of $\mathcal{N}_x^\ell$ as follows: 
\begin{equation*}
\begin{split}
    \mathbb{E}[\mathcal{N}^{\ell}_x] &= \dfrac{\sigma_w^2}{n_{\ell-1}} \sum_{i=1}^{n_{\ell}} \mathbb{E}[\phi^2(\mathcal{U}^\ell_i)] = \dfrac{\sigma_w^2}{2}\dfrac{n_\ell}{n_{\ell-1}},\\
    \mathbb{E}[(\mathcal{N}^{\ell}_x)^2] &= \Bigl(\dfrac{\sigma_w^2}{n_{\ell-1}}\Bigr)^2 \sum_{i=1}^{n_{\ell}} \mathbb{V}[\phi^2(\mathcal{U}^\ell_i)] + \mathbb{E}^2[(\mathcal{N}_x^\ell)] = \Bigl(\dfrac{\sigma_w^2}{2}\Bigr)^2\Bigl(\dfrac{n_\ell}{n_{\ell-1}}\Bigr)^2\Bigl(1 + \dfrac{5}{n_\ell} \Bigr),
\end{split}
\end{equation*}
where we used the moments of variables $\phi(\mathcal{U}_i)$, which can be calculated by integration:
\begin{equation*}
    \mathbb{E}[\phi^2(\mathcal{U}^\ell_i)]=\dfrac{1}{2}, \quad \mathbb{V}[\phi^2(\mathcal{U}^\ell_i)]=\dfrac{5}{4}, \quad i=1,\dots, n_\ell.
\end{equation*}
Moreover, by the central limit theorem we have
\begin{equation*}
    \dfrac{\mathcal{N}^{\ell}_x - \mathbb{E}[\mathcal{N}^{\ell}_x]}{\sqrt{\mathbb{V}[\mathcal{N}^{\ell}_x]}} = \dfrac{2\Bigl(\dfrac{1}{n_\ell}\sum_{i=1}^{n_{\ell}} \phi^2(\mathcal{U}^\ell_i) - \dfrac{1}{2} \Bigr)}{\sqrt{{5}/{n_{\ell}}}} \xrightarrow[n_{\ell}\to\infty]{d} \mathcal{N}(0,1).
\end{equation*}
\end{proof}

\begin{lemma}[Backpropagation of variance]\label{lemma:d_ratio}{Consider the same setting as in Lemma \ref{lemma:x_ratio}. Then the following holds for the ratios of norms of backpropagated errors (defined in \eqref{eq:backprop}) in consecutive layers, denoted $\mathcal{N}^{\ell}_\delta := {\|\g^\ell\|^2}/{\|\g^{\ell+1}\|^2}$, $\ell=1,\dots,L-1$: 

\begin{equation}
    \mathbb{E}[\mathcal{N}^{\ell}_\delta] =  \dfrac{\sigma_w^2}{2}, \quad \mathbb{E}[(\mathcal{N}^{\ell}_x)^2] =  \Bigl(\dfrac{\sigma_w^2}{2}\Bigr)^2\Bigl(1 + \dfrac{5}{n_\ell} \Bigr).
\end{equation}

\begin{equation}
    \dfrac{\mathcal{N}^{\ell}_\delta - \mathbb{E}[\mathcal{N}^{\ell}_\delta]}{\sqrt{\mathbb{V}[\mathcal{N}^{\ell}_\delta]}} \xrightarrow[n_\ell\to\infty]{d} \mathcal{N}(0,1),
\end{equation}
where $\mathcal{N}(0,1)$ is the standard normal distribution.}
\end{lemma}
\begin{proof}
The recursive formula for the backpropagated errors is given by
\begin{equation*}
    \g_i^\ell = \phi'(\h_i^\ell)\sum_{j=1}^{n_{\ell+1}}\W^{\ell+1}_{ji}\g_j^{\ell+1} = \phi'(\h_i^\ell) (\W^{\ell+1})^T_{i\cdot}\g^{\ell+1}.
\end{equation*}
Then, in the same way as in Lemma \ref{lemma:x_ratio}, we have the following for the squared norm of $\g^\ell$:
\begin{equation*}
    \|\g^\ell\|^2 =  \sum_{i=1}^{n_\ell} (\phi'(\h_i^\ell))^2 \Bigl( (\W^{\ell+1})^T_{i\cdot}\g^{\ell+1} \Bigr)^2 {=} \dfrac{\sigma_w^2}{n_\ell}\|\g^{\ell+1}\|^2 \sum_{i=1}^{n_\ell} (\phi'(\h_i^\ell))^2 (\mathcal{V}^{\ell+1}_i)^2,
\end{equation*}
where we introduced i.i.d. random variables $\mathcal{V}_i^{\ell+1} {=} \Bigl\langle \sqrt{\frac{n_{\ell}}{\sigma_w^2}}\W^{\ell+1}_{\cdot i} ,\frac{\g^{\ell+1}}{\|\g^{\ell+1}\|}\Bigr\rangle \sim \mathcal{N}(0,1)$, $i=1,\dots,n_\ell$, which are independent of $\g^{\ell+1}$. One can also see that $\phi'(\h^\ell)$ can only depend on $\{(\W^j,\bias^j)\}_{j=1,\dots \ell}$, therefore it is independent of $\|\g\|^{\ell+1}$ and of $\mathcal{V}^{\ell+1}_i, i = 1,\dots,n_\ell$. Moreover, $\phi'(\h_i^\ell) = \phi'(\W^\ell_{i\cdot}\x^{\ell-1}) = \phi'(\mathcal{U}_i^\ell)$  for all $i = 1,\dots,n_\ell$, therefore $\phi'(\h^\ell)$ depends only on $\W^\ell$. Then we can write the following for the ratio of interest and its moments:
\begin{equation*}
\begin{split}
    \mathcal{N}^\ell_\delta &{=} \dfrac{\sigma_w^2}{n_\ell} \sum_{i=1}^{n_\ell} \phi'(\mathcal{U}_i^\ell) (\mathcal{V}^{\ell+1}_i)^2,\\
    \mathbb{E}[\mathcal{N}^\ell_\delta] = \dfrac{\sigma_w^2}{2},&\qquad \mathbb{E}[(\mathcal{N}^\ell_\delta)^2] = \Bigl(\dfrac{\sigma_w^2}{2}\Bigr)^2\Bigl(1 + \dfrac{5}{n_\ell} \Bigr),
\end{split}   
\end{equation*}
where we calculated the moments of the summands as
\begin{equation*}
    \mathbb{E}[(\phi'(\h_i^\ell))^2 \mathcal{V}_i^2] = \mathbb{E}[(\phi'(\h_i^\ell))^2]\mathbb{E}[\mathcal{V}_i^2] = \dfrac{1}{2},
    \quad \mathbb{V}[(\phi'(\h_i^\ell))^2 \mathcal{V}_i^2]= \mathbb{E}[(\phi'(\h_i^\ell))^4]\mathbb{E}[\mathcal{V}_i^4] - \dfrac{1}{4}= \dfrac{5}{4}.
\end{equation*}
Here we used that, in case of  ReLU activation, $\phi'(\h^\ell_i), i=1,\dots,n_\ell$ are Bernoulli variables with probability of 1 and 0 equal to $1/2$, since $\h^\ell$ is symmetric around zero. Therefore, $\mathbb{E}[(\phi'(\h_i^\ell))^2]=\mathbb{E}[(\phi'(\h_i^\ell))^4] = 1/2$.

Same as in Lemma \ref{lemma:x_ratio}, the limiting distribution of $\mathcal{N}^\ell_\delta$ is given by the central limit theorem.

\end{proof}

As we note in Section \ref{section:proof-ideas}, many papers that study the NTK adopt the following assumption:
\begin{assumption}[Gradient independence assumption (GIA)]\label{assumption:gia}
{Matrix $(\W^\ell)^T$ in backpropagation equations \eqref{eq:backprop} and matrix $\W^\ell$ in forward-propagation equations \eqref{eq:nn} are independent for all $\ell\in\{1,\dots,L\}$.}
\end{assumption}
This assumption is of course not true; however, the products $(\W^\ell)_{i\cdot}^T\x=\sum_{k=1}^{n_\ell} \W^\ell_{ki}\x_k$ and $\W^\ell_{j\cdot}\x=\sum_{k=1}^{n_{\ell-1}} \W^\ell_{jk}\x_k$ are only dependent through the single summand containing $\W^\ell_{ij}$. Thus, the correlations caused by this dependence are of order $O(1/M)$ and can be disregarded in the infinite-width limit. However, in our case of the infinite-depth-and-width limit terms of order $O(1/M)$ can have a non-trivial impact on the computations. Therefore, we calculate the effects of the dependence between the forward-propagated chain and the backpropagated chain in the following lemma.

\begin{lemma}[Gradient independence assumption (GIA)]\label{lemma:gia}{Consider the same setting as in Lemma \ref{lemma:x_ratio}. Then the following statements hold:
\begin{enumerate}
    \item GIA does not change the expectation of $\|\g^\ell\|^2/\|\g^{\ell+k+1}\|^2$:
    \begin{equation}
    \mathbb{E}\Bigl[\prod_{p=0}^k \mathcal{N}^{\ell+p}_\delta\Bigr] = \prod_{p=0}^k \mathbb{E}[\mathcal{N}^{\ell+p}_\delta]\\
\end{equation}
\item GIA changes the expectation of  $\|\g^\ell\|^2/\|\g^{\ell+k+1}\|^2 \cdot \|\x^{\ell+k}\|^2/\|\x^{\ell-1}\|^2$ by a term that has a non-trivial depth-and-width limit where $M\to\infty, L\to\infty, L/M\to\lambda\in\mathbb{R}$. In particular, we have:
\begin{equation}
    \mathbb{E}\Bigl[\prod_{p=0}^k \mathcal{N}^{\ell+p}_\delta\mathcal{N}^{\ell+p}_x\Bigr] = \prod_{p=0}^k \mathbb{E}[\mathcal{N}^{\ell+p}_\delta]\mathbb{E}[\mathcal{N}^{\ell+p}_x] \Bigl(1 + \dfrac{1}{n_{\ell+p}} + O\Bigl(\dfrac{1}{M^{3/2}}\Bigr) \Bigr),\\
\end{equation}
where $n_\ell = \alpha_\ell M, \alpha_\ell \in \mathbb{R}, \ell = 1,\dots, L-1$
\item 
GIA does not change the expectation of $(\|\g^\ell\|^2/\|\g^{\ell+k+1}\|^2)^2$ in the infinite-depth-and-width limit where $M\to\infty, L\to\infty, L/M\to\lambda\in\mathbb{R}$. In particular, we have:
\begin{equation}
    \mathbb{E}\Bigl[\prod_{p=0}^k (\mathcal{N}^{\ell+p}_\delta)^2\Bigr] = \prod_{p=0}^k \mathbb{E}[(\mathcal{N}^{\ell+p}_\delta)^2]\Bigl( 1 + O\Bigl(\dfrac{1}{M^{3/2}}\Bigr)\Bigr),\\
\end{equation}
where $n_\ell = \alpha_\ell M, \alpha_\ell \in \mathbb{R}, \ell = 1,\dots, L-1$:
\end{enumerate}
}
\end{lemma}

\begin{proof}
In Lemmas \ref{lemma:x_ratio} and \ref{lemma:d_ratio} we derived the following equations for $\mathcal{N}^\ell_\delta$ and $\mathcal{N}^\ell_x$: 
\begin{equation*}
\begin{split}
    \mathcal{N}^\ell_\delta &{=} \dfrac{\sigma_w^2}{n_\ell} \sum_{i=1}^{n_\ell} (\phi'(\h_i^\ell))^2 (\mathcal{V}_i^{\ell+1})^2 {=} \dfrac{\sigma_w^2}{n_\ell} \sum_{i=1}^{n_\ell} (\phi'(\mathcal{U}_i^\ell))^2 (\mathcal{V}^{\ell+1}_i)^2,\\
    \mathcal{N}_x^\ell &{=} \dfrac{\sigma_w^2}{n_{\ell-1}} \sum_{i=1}^{n_{\ell}} \phi^2(\mathcal{U}^\ell_i),
\end{split}
\end{equation*}
where $\mathcal{U}_i^\ell$ depends only on the $i$-th row of the weights matrix $\W_{i\cdot}^\ell$ and $\mathcal{V}_j^\ell$ depends only on $j$-th column of the same matrix $\W_{\cdot j}^\ell$ for $i=1,\dots,n_\ell, j=1,\dots,n_{\ell-1}, \ell = 1,\dots,L-1$. Therefore, variables $\mathcal{U}_i^\ell$ and $\mathcal{V}_j^\ell$ are only dependent through the single weight $\W_{ij}^\ell$, which nevertheless makes $\mathcal{N}_\delta^\ell$ and $\mathcal{N}_\delta^{\ell+1}$ dependent for any $\ell=1,\dots,L-2$. One can also see that $\mathcal{N}_\delta^\ell$ and $\mathcal{N}_x^\ell$ are dependent through $\{\mathcal{U}_i^\ell\}_{i=1,\dots,n_\ell}$. The objective of this lemma is to determine the effects of these weak dependencies on the expectation of products that appear in the NTK. 

\textbf{Part 1.}
We first consider the product of ratios of the backpropagated errors:
    \begin{equation*}
    \begin{split}
        \mathbb{E}\Bigl[\prod_{p=0}^k \mathcal{N}^{\ell+p}_\delta\Bigr] &= \prod_{p=0}^k \dfrac{\sigma_w^2}{n_{\ell+p}} \sum_{i_0=1}^{n_\ell}\dots\sum_{i_k=1}^{n_{\ell+k}}\mathbb{E}[ \phi'(\mathcal{U}_{i_0}^\ell) (\mathcal{V}^{\ell+1}_{i_0})^2\phi'(\mathcal{U}_{i_1}^{\ell+1}) (\mathcal{V}^{\ell+2}_{i_1})^2\cdot\dots\cdot\phi'(\mathcal{U}_{i_k}^{\ell+k}) (\mathcal{V}^{\ell+k+1}_{i_k})^2]\\
        &= \prod_{p=0}^k \dfrac{\sigma_w^2}{n_{\ell+p}}\sum_{i_0=1}^{n_\ell}\dots\sum_{i_k=1}^{n_{\ell+k}}\mathbb{E}[\phi'(\mathcal{U}_{i_0}^\ell)] \mathbb{E}[(\mathcal{V}^{\ell+k+1}_{i_k})^2] \prod_{p=1}^k \mathbb{E}[(\mathcal{V}^{\ell+p}_{i_{p-1}})^2\phi'(\mathcal{U}_{i_p}^{\ell+p})] \\
        &= \dfrac{1}{2}\prod_{p=0}^k \dfrac{\sigma_w^2}{n_{\ell+p}}\sum_{i_0=1}^{n_\ell}\dots\sum_{i_k=1}^{n_{\ell+k}}\prod_{p=1}^k \mathbb{E}[(\mathcal{V}^{\ell+p}_{i_{p-1}})^2\phi'(\mathcal{U}_{i_p}^{\ell+p})].
    \end{split}
    \end{equation*}
    As $\mathcal{U}_{i_p}^{\ell+p}$ that $\mathcal{V}_{i_{p-1}}^{\ell+p}$ depend only through $\W_{i_p i_{p-1}}^{\ell+p}$, we can condition the expectation of their product as follows:
    \begin{equation*}
        \mathbb{E}[(\mathcal{V}^{\ell+p}_{i_{p-1}})^2\phi'(\mathcal{U}_{i_p}^{\ell+p})] = \mathbb{E}\bigl[\mathbb{E}[(\mathcal{V}^{\ell+p}_{i_{p-1}})^2 \mid \W_{i_p i_{p-1}}^{\ell+p}]\cdot \mathbb{E}[\phi'(\mathcal{U}_{i_p}^{\ell+p}) \mid \W_{i_p i_{p-1}}^{\ell+p}]\bigr].
    \end{equation*}
    To simplify the notation, let us denote $w_{i_p i_{p-1}} := \sqrt{\frac{n_{\ell+p-1}}{\sigma_w^2}}\W_{i_p i_{p-1}}^{\ell+p} \sim \mathcal{N}(0,1)$,  $a_j := \x^{\ell+p}_j/\|\x^{\ell+p}\|$ and $b_k ~:=~\g^{\ell+p}_k/\|\g^{\ell+p}\|$. Then we have $\mathcal{V}^{\ell+p}_{i_{p-1}} = \sum_{k=1}^{n_{\ell+p}} w_{k i_{p-1}} b_k = w_{i_p i_{p-1}} b_{i_p} + \sum_{k\neq i_p}w_{k i_{p-1}} b_k$ and \newline $\mathcal{U}^{\ell+p}_{i_{p}} = \sum_{j=1}^{n_{\ell+p-1}} w_{j i_{p-1}} a_j = w_{i_p i_{p-1}} a_{i_{p-1}} + \sum_{j\neq i_{p-1}}w_{i_{p}j} a_j$. We can then open the conditional expectations:
    \begin{equation*}
        \begin{split}
            \mathbb{E}[(\mathcal{V}^{\ell+p}_{i_{p-1}})^2 \mid \W_{i_p i_{p-1}}^{\ell+p}] &= w^2_{i_p i_{p-1}}b_{i_p}^2 + \mathbb{E}[(\sum_{k\neq i_p}w_{k i_{p-1}} b_k)^2] = 1 - b_{i_p}^2(1 - w^2_{i_p i_{p-1}}),\\ 
            \mathbb{E}[\phi'(\mathcal{U}_{i_p}^{\ell+p}) \mid \W_{i_p i_{p-1}}^{\ell+p}] &= \mathbb{P}[ \sum_{j\neq i_{p-1}}w_{i_{p}j} b_k > - w_{i_p i_{p-1}} a_{i_{p-1}}] = \Phi\Bigl(\dfrac{w_{i_p i_{p-1}} a_{i_{p-1}}}{\sqrt{1 - a^2_{i_{p-1}}}}\Bigr),
        \end{split}
    \end{equation*}
    where $\Phi(\cdot)$ is the CDF of the standard normal distribution. Here we used that $\sum_{k\neq i_p}w_{k i_{p-1}} b_k \sim \mathcal{N}(0, 1-b_{i_p}^2)$ and $\sum_{j\neq i_{p-1}}w_{i_{p}j} a_j \sim \mathcal{N}(0,1-a_{i_{p-1}}^2)$. Then we have:
    
\begin{equation*}
  \begin{multlined}
      \mathbb{E}[(\mathcal{V}^{\ell+p}_{i_{p-1}})^2\phi'(\mathcal{U}_{i_p}^{\ell+p})] =  (1 - b_{i_p}^2)\mathbb{E}[\Phi(A \cdot w_{i_p i_{p-1}})] + b_{i_p}^2\mathbb{E}[w^2_{i_p i_{p-1}}\Phi(A \cdot w_{i_p i_{p-1}})]
       = \dfrac{1}{2},
  \end{multlined}
  \end{equation*}
  where we used the following integrals:
 \begin{equation*} 
  \begin{split}
      \mathbb{E}[\Phi(A \cdot w_{i_p i_{p-1}})] &= \dfrac{1}{2} + \dfrac{1}{2}\int_{-\infty}^{\infty} \dfrac{1}{\sqrt{2\pi}}\erf\Bigl(\frac{A}{\sqrt 2} w_{i_p i_{p-1}}\Bigr) \exp \Bigl(-\frac{w_{i_p i_{p-1}}^2}{2}\Bigr) dw_{i_p i_{p-1}} = \dfrac{1}{2},\\
   &\begin{multlined}
   \mathllap{\mathbb{E}[w_{i_p i_{p-1}}^2\Phi(A \cdot w_{i_p i_{p-1}})]} = \dfrac{1}{2}\mathbb{E}[w_{i_p i_{p-1}}^2]\\ + \dfrac{1}{2}\int_{-\infty}^{\infty} \dfrac{1}{\sqrt{2\pi}}w_{i_p i_{p-1}}^2\erf\Bigl(\frac{A}{\sqrt 2} w_{i_p i_{p-1}}\Bigr) \exp \Bigl(-\frac{w_{i_p i_{p-1}}^2}{2}\Bigr) dw_{i_p i_{p-1}} = \dfrac{1}{2}.
  \end{multlined}
    \end{split}
\end{equation*}
Thus, the expectation of the product of the ratios of backpropagated errors is exactly equal to the product of their expectations:
\begin{equation*}
\begin{multlined}
    \mathbb{E}\Bigl[\prod_{p=0}^k \mathcal{N}^{\ell+p}_\delta\Bigr] =\dfrac{1}{2}\prod_{p=0}^k \dfrac{\sigma_w^2}{n_{\ell+p}}\sum_{i_0=1}^{n_\ell}\dots\sum_{i_k=1}^{n_{\ell+k}} \dfrac{1}{2^k} = \Bigl( \dfrac{\sigma_w^2}{2}\Bigr)^{k+1} = \prod_{p=0}^k \mathbb{E}[\mathcal{N}_\delta^{\ell+p}],
\end{multlined}
\end{equation*}
which completes the proof of the first statement.

\textbf{Part 2.} We now consider the expectation of products of the activations' ratios and the backpropagated errors' ratios for the same layers. The product in a single layer is given by:

\begin{equation*}
\begin{split}
     \mathcal{N}^\ell_\delta\mathcal{N}_x^\ell &= \dfrac{\sigma_w^2}{n_{\ell-1}}\dfrac{\sigma_w^2}{n_{\ell}} \sum_{i=1}^{n_{\ell}}\sum_{j=1}^{n_{\ell}} \phi^2(\mathcal{U}^\ell_j)(\phi'(\mathcal{U}_i^\ell))^2 (\mathcal{V}^{\ell+1}_i)^2 = \dfrac{\sigma_w^2}{n_{\ell-1}}\dfrac{\sigma_w^2}{n_{\ell}} \sum_{i=1}^{n_{\ell}}\sum_{j=1}^{n_{\ell}} \phi'(\mathcal{U}^\ell_i)\phi'(\mathcal{U}_j^\ell) (\mathcal{U}^\ell_j)^2(\mathcal{V}^{\ell+1}_i)^2,
\end{split}
\end{equation*}
where we noticed that $\phi(\mathcal{U}_i^\ell)\phi'(\mathcal{U}_i^\ell) = \phi(\mathcal{U}_i^\ell) = \mathcal{U}_i^\ell\phi'(\mathcal{U}_i^\ell)$. Then for the product involving multiple layers we have:
\begin{equation*}
\begin{split}
    \qquad\qquad\qquad\qquad&\begin{multlined}
    \mathllap{\prod_{p=0}^k \mathcal{N}^{\ell+p}_\delta\mathcal{N}^{\ell+p}_x} = \dfrac{\sigma_w^2}{n_{\ell-1}}\dfrac{\sigma_w^2}{n_{\ell+k}}\prod_{j=0}^{k-1}\Bigl(\dfrac{\sigma_w^2}{n_{\ell+j}}\Bigr)^2 \sum_{i_0=1}^{n_{\ell}}\sum_{j_0=1}^{n_{\ell}} \dots \\
    \dots \sum_{i_k=1}^{n_{\ell+k}}\sum_{j_k=1}^{n_{\ell+k}} \phi'(\mathcal{U}^\ell_{i_0})\phi'(\mathcal{U}_{j_0}^\ell) (\mathcal{U}^\ell_{j_0})^2(\mathcal{V}^{\ell+1}_{i_0})^2\dots \phi'(\mathcal{U}^{\ell+k}_{i_k})\phi'(\mathcal{U}_{j_k}^{\ell+k}) (\mathcal{U}^{\ell+k}_{j_k})^2(\mathcal{V}^{\ell+k+1}_{i_k})^2,
    \end{multlined}
\end{split}
\end{equation*}
And the expectation can be decomposed into products as follows:
\begin{equation*}
\begin{split}
    \qquad\qquad\qquad\qquad&\begin{multlined}
      \mathllap{\mathbb{E}\Bigl[\prod_{p=0}^k \mathcal{N}^{\ell+p}_\delta\mathcal{N}^{\ell+p}_x\Bigr]} =
     \dfrac{\sigma_w^2}{n_{\ell-1}}\dfrac{\sigma_w^2}{n_{\ell+k}}\prod_{j=0}^{k-1}\Bigl(\dfrac{\sigma_w^2}{n_{\ell+j}}\Bigr)^2 \sum_{i_0=1}^{n_{\ell}}\sum_{j_0=1}^{n_{\ell}}\dots \\
     \dots \sum_{i_k=1}^{n_{\ell+k}}\sum_{j_k=1}^{n_{\ell+k}}\mathbb{E}[\phi'(\mathcal{U}^\ell_{i_0})\phi'(\mathcal{U}_{j_0}^\ell) (\mathcal{U}^\ell_{j_0})^2] \mathbb{E}[(\mathcal{V}^{\ell+k+1}_{i_k})^2]
     \cdot \prod_{p=1}^k \mathbb{E}[(\mathcal{V}^{\ell+p}_{i_{p-1}})^2\phi'(\mathcal{U}^{\ell+p}_{i_p})\phi'(\mathcal{U}_{j_p}^{\ell+p}) (\mathcal{U}^{\ell+p}_{j_p})^2]
     \end{multlined}\\
     \qquad\qquad\qquad\qquad&\begin{multlined} =  \dfrac{n_\ell}{4}\Bigl(1+\frac{1}{n_\ell}\Bigr) \dfrac{\sigma_w^2}{n_{\ell-1}}\dfrac{\sigma_w^2}{n_{\ell+k}}\prod_{j=0}^{k-1}\Bigl(\dfrac{\sigma_w^2}{n_{\ell+j}}\Bigr)^2\sum_{i_0=1}^{n_{\ell}}\dots\sum_{i_k=1}^{n_{\ell+k}}\sum_{j_k=1}^{n_{\ell+k}}
    \prod_{p=1}^k \mathbb{E}[(\mathcal{V}^{\ell+p}_{i_{p-1}})^2\phi'(\mathcal{U}^{\ell+p}_{i_p})\phi'(\mathcal{U}_{j_p}^{\ell+p}) (\mathcal{U}^{\ell+p}_{j_p})^2],
    \end{multlined}
\end{split}
\end{equation*}
where we used that $\mathbb{E}[(\mathcal{V}_{i_k}^{l+k+1})^2]=1$ and  $\sum_{j_0=1}^{n_\ell}\mathbb{E}[\phi'(\mathcal{U}^\ell_{i_0})\phi'(\mathcal{U}_{j_0}^\ell) (\mathcal{U}^\ell_{j_0})^2] = \frac{1}{4}({n_\ell}-1) + \frac{1}{2} = \frac{n_\ell}{4}\Bigl(1+\frac{1}{n_\ell}\Bigr)$. If $j_p\neq i_p$, we also know that the terms in $\mathbb{E}[(\mathcal{V}^{\ell+p}_{i_{p-1}})^2\phi'(\mathcal{U}^{\ell+p}_{i_p})\phi'(\mathcal{U}_{j_p}^{\ell+p}) (\mathcal{U}^{\ell+p}_{j_p})^2]$ are only dependent through $\{\W^{\ell+p}_{i_pi_{p-1}}, \W^{\ell+p}_{j_{p},i_{p-1}}\}$. Same as in Part 1, we can condition the product on these weights: 
\begin{equation*}
\begin{multlined}
    \mathbb{E}[(\mathcal{V}^{\ell+p}_{i_{p-1}})^2\phi'(\mathcal{U}^{\ell+p}_{i_p})\phi'(\mathcal{U}_{j_p}^{\ell+p}) (\mathcal{U}^{\ell+p}_{j_p})^2] = \mathbb{E}\bigl[ \mathbb{E}[(\mathcal{V}^{\ell+p}_{i_{p-1}})^2 \mid w_{i_pi_{p-1}}, w_{j_{p}i_{p-1}}]\cdot \mathbb{E}[\phi'(\mathcal{U}^{\ell+p}_{i_p})\mid w_{i_pi_{p-1}}] \cdot\\
    \cdot \mathbb{E}[\phi'(\mathcal{U}_{j_p}^{\ell+p}) \mid w_{j_{p}i_{p-1}}]
    \cdot \mathbb{E}[(\mathcal{U}^{\ell+p}_{j_p})^2 \mid w_{j_{p}i_{p-1}}]
    \bigr].
\end{multlined}
\end{equation*}
And we can again write the conditional expectations in case $j_p\neq i_p$ as follows:
\begin{equation*}
    \begin{split}
        \mathbb{E}[(\mathcal{V}^{\ell+p}_{i_{p-1}})^2 \mid w_{i_pi_{p-1}}, w_{j_{p}i_{p-1}}] &= 1 - b_{i_p}^2(1-w_{i_pi_{p-1}}^2) - b_{j_p}^2(1-w_{j_pi_{p-1}}^2) + 2b_{i_p}b_{j_p}w_{i_pi_{p-1}}w_{j_pi_{p-1}},\\
        \mathbb{E}[\phi'(\mathcal{U}^{\ell+p}_{i_p})\mid w_{i_pi_{p-1}}] &= \Phi\Bigl(\dfrac{w_{i_p i_{p-1}} a_{i_{p-1}}}{\sqrt{1 - a^2_{i_{p-1}}}}\Bigr),\\
        \mathbb{E}[\phi'(\mathcal{U}^{\ell+p}_{j_p})\mid w_{j_pi_{p-1}}] &= \Phi\Bigl(\dfrac{w_{j_p i_{p-1}} a_{i_{p-1}}}{\sqrt{1 - a^2_{i_{p-1}}}}\Bigr),\\
        \mathbb{E}[(\mathcal{U}^{\ell+p}_{j_p})^2 \mid w_{j_{p}i_{p-1}}] &= 1 - a_{i_{p-1}}^2(1 - w_{j_{p}i_{p-1}}^2).
    \end{split}
\end{equation*}
To calculate the expectation of the product here we will need to use that $\mathbb{E}[\Phi(A \cdot w)] = \mathbb{E}[w^2\Phi(A \cdot w)] =\frac{1}{2}$, which we already computed in Part 1. One can also easily see that $\mathbb{E}[w^4\Phi(A \cdot w)] = \frac{3}{2}$. The other expectations involved in the product can be calculated as follows:

\begin{equation*}
    \begin{split}
        \mathbb{E}[w\Phi(A \cdot w)] &= \dfrac{1}{2\sqrt{2\pi}}\int_{-\infty}^{\infty} w \erf\Bigl(\frac{A}{\sqrt 2} w\Bigr) \exp \Bigl(-\frac{w^2}{2}\Bigr) dw =\dfrac{1}{\sqrt{2\pi}}\dfrac{A}{\sqrt{A^2+1}},\\
        \mathbb{E}[w^3\Phi(A \cdot w)] &= \dfrac{1}{2\sqrt{2\pi}}\int_{-\infty}^{\infty} w^3 \erf\Bigl(\frac{A}{\sqrt 2} w\Bigr) \exp \Bigl(-\frac{w^2}{2}\Bigr) dw = \dfrac{1}{\sqrt{2\pi}}\dfrac{A}{\sqrt{A^2+1}}\Bigl(2 + \dfrac{1}{A^2+1} \Bigr).
    \end{split}
\end{equation*}
Expressions for the integrals above can be found e.g. in \cite{korotkov2020integrals}. Using all the above expressions, we can obtain the following expression for the considered expectation in case $j_p\neq i_p$:
\begin{equation*}
\begin{split}
     \mathbb{E}[(\mathcal{V}^{\ell+p}_{i_{p-1}})^2\phi'(\mathcal{U}^{\ell+p}_{i_p})\phi'(\mathcal{U}_{j_p}^{\ell+p}) (\mathcal{U}^{\ell+p}_{j_p})^2] &= \dfrac{1}{4}(1-b_{i_p}^2-b_{j_p}^2)(1-a_{i_{p-1}}^2) + \dfrac{1}{4}(b_{i_p}^2+b_{j_p}^2)(1-a_{i_{p-1}}^2)\\ 
     &+ \dfrac{1}{4}(1-b_{i_p}^2-b_{j_p}^2)a_{i_{p-1}}^2
     + \dfrac{1}{4}b^2_{i_p}a_{i_{p-1}}^2 + \dfrac{3}{4}b_{j_p}^2a_{i_{p-1}}^2 
     \\&+ 2b_{i_p}b_{j_p}(1-a_{i_{p-1}}^2) \dfrac{4A^2}{A^2+1} + 2b_{i_p}b_{j_p}a_{i_{p-1}}^2 \dfrac{4A^2}{A^2+1}\Bigl(2 + \dfrac{1}{A^2+1} \Bigr) \\
     &= \dfrac{1}{4} +\dfrac{1}{2}b_{j_p}^2a_{i_{p-1}}^2 + 8b_{i_p}b_{j_p}a_{i_{p-1}}^2(1+2a_{i_{p-1}}^2-a_{i_{p-1}}^4)
\end{split}
\end{equation*}
On the other hand, if $j_p = i_p$ we have $\phi'(\mathcal{U}_{i_p}^{l+p})\phi'(\mathcal{U}_{j_p}^{l+p})(\mathcal{U}_{j_p}^{l+p})^2 = \phi'(\mathcal{U}_{i_p}^{l+p}))(\mathcal{U}_{i_p}^{l+p})^2$ and therefore the expectation is given by:
\begin{equation*}
    \mathbb{E}[(\mathcal{V}^{\ell+p}_{i_{p-1}})^2\phi'(\mathcal{U}^{\ell+p}_{i_p}) (\mathcal{U}^{\ell+p}_{i_p})^2] = \dfrac{1}{2} + a_{i_{p-1}}^2b^2_{i_p}.
\end{equation*}
We now notice that index $j_p$ appears only in one expectation term in the product for each $p$. Therefore, we can sum over $j_p$ independently for all $p$:
\begin{equation*}
\begin{multlined}
    \sum_{j_p=1}^{n_{\ell+p}}\mathbb{E}[(\mathcal{V}^{\ell+p}_{i_{p-1}})^2\phi'(\mathcal{U}^{\ell+p}_{i_p})\phi'(\mathcal{U}_{j_p}^{\ell+p}) (\mathcal{U}^{\ell+p}_{j_p})^2] = \dfrac{n_{\ell+p}}{4}(1 + \dfrac{1}{n_{\ell+p}}) + \dfrac{1}{2}a_{i_{p-1}}^2 + \dfrac{1}{2}a_{i_{p-1}}^2 b_{i_p}^2\\ + \Bigl(\sum_{j_p=1}^{n_{\ell+p}}b_{j_p}\Bigr) 8b_{i_p}a_{i_{p-1}}^2(1+2a_{i_{p-1}}^2-a_{i_{p-1}}^4).
\end{multlined}
\end{equation*}
On the other hand, we need to sum over $i_{p-1}$ values sequentially over different values of $p$. First, we can calculate the following sum:
\begin{equation*}
\begin{multlined}
    \sum_{i_{p-1}=1}^{n_{\ell+p-1}}\sum_{j_p=1}^{n_{\ell+p}}\mathbb{E}[(\mathcal{V}^{\ell+p}_{i_{p-1}})^2\phi'(\mathcal{U}^{\ell+p}_{i_p})\phi'(\mathcal{U}_{j_p}^{\ell+p}) (\mathcal{U}^{\ell+p}_{j_p})^2] = \dfrac{n_{\ell+p-1}n_{\ell+p}}{4}(1 + \dfrac{1}{n_{\ell+p}}) + \dfrac{1}{2} + \dfrac{1}{2} b_{i_p}^2\\ + \Bigl(\sum_{j_p=1}^{n_{\ell+p}}b_{j_p}\Bigr) 8b_{i_p}(1+2\sum_{i_{p-1}=1}^{n_{\ell+p-1}}a_{i_{p-1}}^4-\sum_{i_{p-1}=1}^{n_{\ell+p-1}}a_{i_{p-1}}^6).
\end{multlined}
\end{equation*}
Then we can obtain the following bounds for the sum of $b_{j_p}$ and $a_{i_{p-1}}$ given by Hölder's inequality:
\begin{equation*}
\begin{split}
    \Bigl| \sum_{j_p=1}^{n_{\ell+p}}b_{j_p} \Bigr| &\leq \sum_{j_p=1}^{n_{\ell+p}} |b_{j_p}| = \|\mathbf{b}\|_1 \leq \sqrt{n_{\ell+p}}\|\mathbf{b}^2\| = \sqrt{n_{\ell+p}},\\
     0 &\leq 1+2\sum_{i_{p-1}=1}^{n_{\ell+p-1}}a_{i_{p-1}}^4  -\sum_{i_{p-1}=1}^{n_{\ell+p-1}}a_{i_{p-1}}^6 \leq 3,\\
     | b_{i_p} | &\leq 1, \quad b_{i_p}^2 \leq 1.
\end{split}
\end{equation*}
Therefore, we can rewrite the previous sum as
\begin{equation*}
     \sum_{i_{p-1}=1}^{n_{\ell+p-1}}\sum_{j_p=1}^{n_{\ell+p}}\mathbb{E}[(\mathcal{V}^{\ell+p}_{i_{p-1}})^2\phi'(\mathcal{U}^{\ell+p}_{i_p})\phi'(\mathcal{U}_{j_p}^{\ell+p}) (\mathcal{U}^{\ell+p}_{j_p})^2] = \dfrac{n_{\ell+p-1}n_{\ell+p}}{4}\Bigl(1 + \dfrac{1}{n_{\ell+p}} + O\bigl(\dfrac{1}{M^{3/2}}\bigr)\Bigr)
\end{equation*}
Finally, for the expectation of the whole product we have:
\begin{equation*}
\begin{split}
    \mathbb{E}\Bigl[\prod_{p=0}^k \mathcal{N}^{\ell+p}_\delta\mathcal{N}^{\ell+p}_x\Bigr] &=
     \dfrac{\sigma_w^2}{n_{\ell-1}}\dfrac{\sigma_w^2}{n_{\ell+k}}\prod_{j=0}^{k-1}\Bigl(\dfrac{\sigma_w^2}{n_{\ell+j}}\Bigr)^2 \dfrac{n_\ell n_{\ell+k}}{4}\Bigl(1+\frac{1}{n_\ell}\Bigr) \prod_{p=1}^{k}\dfrac{n_{\ell+p-1}n_{\ell+p}}{4}\Bigl(1 + \dfrac{1}{n_{\ell+p}} + O\bigl(\dfrac{1}{M^{3/2}}\bigr)\Bigr)\\
     &= \Bigl(\dfrac{\sigma_w^2}{2}\Bigr)^{2(k+1)} \dfrac{n_{\ell+k}}{n_{\ell-1}}\prod_{p=0}^k \Bigl(1 + \dfrac{1}{n_{\ell+p}} + O\bigl(\dfrac{1}{M^{3/2}}\bigr) \Bigr)\\
     &= \prod_{p=0}^k \mathbb{E}[\mathcal{N}^{\ell+p}_\delta]\mathbb{E}[\mathcal{N}^{\ell+p}_x] \Bigl(1 + \dfrac{1}{n_{\ell+p}} + O\bigl(\dfrac{1}{M^{3/2}}\bigr) \Bigr)
\end{split}
\end{equation*}

\textbf{Part 3.} Finally, we consider the expectation of a product of squared ratios of the backpropagated errors. In a single layer we have:
\begin{equation*}
    (\mathcal{N}^\ell_\delta)^2 = \Bigl(\dfrac{\sigma_w^2}{n_\ell}\Bigr)^2 \sum_{i=1}^{n_\ell}\sum_{j=1}^{n_\ell} \phi'(\mathcal{U}_i^\ell)\phi'(\mathcal{U}_j^\ell) (\mathcal{V}^{\ell+1}_i)^2(\mathcal{V}^{\ell+1}_j)^2.
\end{equation*}
And for the product in multiple layers we have:
\begin{equation*}
\begin{split}
    \qquad\qquad\qquad\qquad&\begin{multlined}
    \mathllap{\prod_{p=0}^k (\mathcal{N}^{\ell+p}_\delta)^2} = \prod_{p=0}^{k}\Bigl(\dfrac{\sigma_w^2}{n_{\ell+j}}\Bigr)^2 \sum_{i_0=1}^{n_{\ell}}\sum_{j_0=1}^{n_{\ell}} \dots \\
    \dots \sum_{i_k=1}^{n_{\ell+k}}\sum_{j_k=1}^{n_{\ell+k}} \phi'(\mathcal{U}^\ell_{i_0})\phi'(\mathcal{U}_{j_0}^{\ell}) (\mathcal{V}^{\ell+1}_{i_0})^2(\mathcal{V}^{\ell+1}_{j_0})^2\dots \phi'(\mathcal{U}^{\ell+k}_{i_k})\phi'(\mathcal{U}_{j_k}^{\ell+k}) (\mathcal{V}^{\ell+k+1}_{j_k})^2(\mathcal{V}^{\ell+k+1}_{i_k})^2,
    \end{multlined}\\
    \qquad\qquad\qquad\qquad&\begin{multlined}
      \mathllap{\mathbb{E}\Bigl[\prod_{p=0}^k (\mathcal{N}^{\ell+p}_\delta)^2]} =
     \prod_{p=0}^{k}\Bigl(\dfrac{\sigma_w^2}{n_{\ell+p}}\Bigr)^2 \sum_{i_0=1}^{n_{\ell}}\sum_{j_0=1}^{n_{\ell}}\dots \\
     \dots \sum_{i_k=1}^{n_{\ell+k}}\sum_{j_k=1}^{n_{\ell+k}}\mathbb{E}[\phi'(\mathcal{U}^\ell_{i_0})\phi'(\mathcal{U}_{j_0}^\ell)] \mathbb{E}[(\mathcal{V}^{\ell+k+1}_{i_k})^2(\mathcal{V}^{\ell+k+1}_{j_k})^2]
     \cdot \prod_{p=1}^k \mathbb{E}[(\mathcal{V}^{\ell+p}_{i_{p-1}})^2(\mathcal{V}^{\ell+p}_{j_{p-1}})^2\phi'(\mathcal{U}^{\ell+p}_{i_p})\phi'(\mathcal{U}_{j_p}^{\ell+p})]
     \end{multlined}
\end{split}
\end{equation*}
Here the expectations under product are more complicated since the variables in $\mathbb{E}[(\mathcal{V}^{\ell+p}_{i_{p-1}})^2(\mathcal{V}^{\ell+p}_{j_{p-1}})^2\phi'(\mathcal{U}^{\ell+p}_{i_p})\phi'(\mathcal{U}_{j_p}^{\ell+p})]$ are dependent through $\{\W^{\ell+p}_{i_pi_{p-1}}, \W^{\ell+p}_{j_{p},i_{p-1}}, \W^{\ell+p}_{i_pj_{p-1}},\W^{\ell+p}_{j_pj_{p-1}}\}$. Nevertheless, we can still decompose the expectation as before into the following terms:
\begin{equation*}
\begin{multlined}
    \mathbb{E}[(\mathcal{V}^{\ell+p}_{i_{p-1}})^2(\mathcal{V}^{\ell+p}_{j_{p-1}})^2\phi'(\mathcal{U}^{\ell+p}_{i_p})\phi'(\mathcal{U}_{j_p}^{\ell+p})] = \mathbb{E}\bigl[ \mathbb{E}[(\mathcal{V}^{\ell+p}_{i_{p-1}})^2 \mid w_{i_pi_{p-1}}, w_{j_{p}i_{p-1}}]\cdot  
    \mathbb{E}\bigl[ \mathbb{E}[(\mathcal{V}^{\ell+p}_{j_{p-1}})^2 \mid w_{i_pj_{p-1}}, w_{j_{p}j_{p-1}}]\cdot\\ \mathbb{E}[\phi'(\mathcal{U}^{\ell+p}_{i_p})\mid w_{i_pi_{p-1}},w_{i_pj_{p-1}}]
    \cdot \mathbb{E}[\phi'(\mathcal{U}^{\ell+p}_{j_p})\mid w_{j_pi_{p-1}},w_{j_pj_{p-1}}]
    \bigr].
\end{multlined}
\end{equation*}
And each conditional expectations can again be calculated explicitly:
\begin{equation*}
    \begin{split}
        \mathbb{E}[(\mathcal{V}^{\ell+p}_{i_{p-1}})^2 \mid w_{i_pi_{p-1}}, w_{j_{p}i_{p-1}}] &= 1 - b_{i_p}^2(1-w_{i_pi_{p-1}}^2) - b_{j_p}^2(1-w_{j_pi_{p-1}}^2) + 2b_{i_p}b_{j_p}w_{i_pi_{p-1}}w_{j_pi_{p-1}},\\
        \mathbb{E}[(\mathcal{V}^{\ell+p}_{j_{p-1}})^2 \mid w_{i_pj_{p-1}}, w_{j_{p}j_{p-1}}] &= 1 - b_{i_p}^2(1-w_{i_pj_{p-1}}^2) - b_{j_p}^2(1-w_{j_pj_{p-1}}^2) + 2b_{i_p}b_{j_p}w_{i_pj_{p-1}}w_{j_pj_{p-1}},\\
        \mathbb{E}[\phi'(\mathcal{U}^{\ell+p}_{i_p})\mid w_{i_pi_{p-1}},w_{i_pj_{p-1}}] &= \Phi\Bigl(\dfrac{w_{i_p i_{p-1}} a_{i_{p-1}} + w_{i_p j_{p-1}} a_{j_{p-1}}}{\sqrt{1 - a^2_{i_{p-1}}-a^2_{j_{p-1}}}}\Bigr),\\
       \mathbb{E}[\phi'(\mathcal{U}^{\ell+p}_{j_p})\mid w_{j_pi_{p-1}},w_{j_pj_{p-1}}] &= \Phi\Bigl(\dfrac{w_{j_p i_{p-1}} a_{i_{p-1}} + w_{j_p j_{p-1}} a_{j_{p-1}}}{\sqrt{1 - a^2_{i_{p-1}}-a^2_{j_{p-1}}}}\Bigr).
    \end{split}
\end{equation*}
We open the expectation using the following expressions, which, as before, are integrals involving the error function computed e.g. in \cite{korotkov2020integrals}:
\begin{equation*}
    \begin{split}
        \mathbb{E}[\Phi(A_iw_i + A_jw_j)] = \mathbb{E}[w_i^2\Phi(A_iw_i + A_jw_j)] &= \dfrac{1}{2},\\
        \mathbb{E}\Biggl[w_{i}\Phi\Bigl(\dfrac{w_{i} a_{i_{p-1}} + w_{j} a_{j_{p-1}}}{\sqrt{1 - a^2_{i_{p-1}}-a^2_{j_{p-1}}}}\Bigr)\Biggr] &= \sqrt{\dfrac{1}{2\pi}}a_{i_{p-1}},\\ 
         \mathbb{E}\Biggl[w_{i}w_{j}\Phi\Bigl(\dfrac{w_{i} a_{i_{p-1}} + w_{j} a_{j_{p-1}}}{\sqrt{1 - a^2_{i_{p-1}}-a^2_{j_{p-1}}}}\Bigr)\Biggr] &= 0,\\
         \mathbb{E}\Biggl[w_{i}^2w_{j}^2\Phi\Bigl(\dfrac{w_{i} a_{i_{p-1}} + w_{j} a_{j_{p-1}}}{\sqrt{1 - a^2_{i_{p-1}}-a^2_{j_{p-1}}}}\Bigr)\Biggr] &= \dfrac{1}{2},\\ 
        \mathbb{E}\Biggl[w_{i}w_{j}^2\Phi\Bigl(\dfrac{w_{i} a_{i_{p-1}} + w_{j} a_{j_{p-1}}}{\sqrt{1 - a^2_{i_{p-1}}-a^2_{j_{p-1}}}}\Bigr)\Biggr] &= \dfrac{1}{\sqrt{2\pi}}a_{i_{p-1}}(1-a_{j_{p-1}}^2).
    \end{split}
\end{equation*}
Using all of the above, we get the following expression for the expectation in case $i_{p-1}\neq j_{p-1}$:
\begin{equation*}
    \mathbb{E}[(\mathcal{V}^{\ell+p}_{i_{p-1}})^2(\mathcal{V}^{\ell+p}_{j_{p-1}})^2\phi'(\mathcal{U}^{\ell+p}_{i_p})\phi'(\mathcal{U}_{j_p}^{\ell+p})] = 
    \begin{cases}
    \dfrac{1}{2} & i_p=j_p,\\[15pt]
    \dfrac{1}{4} +\dfrac{1}{\pi}b_{i_p}b_{j_p}(a_{i_{p-1}}^2 + a_{j_{p-1}}^2 - 2a_{i_{p-1}}^2a_{j_{p-1}}^2(b_{i_p}^2 + b_{j_p}^2)) & i_p\neq j_p,
    \end{cases}
\end{equation*}

In case $i_{p-1}= j_{p-1}$, we have 

\begin{equation*}
    \begin{split}
        \mathbb{E}[(\mathcal{V}^{\ell+p}_{i_{p-1}})^4 \mid w_{i_pi_{p-1}}, w_{j_{p}i_{p-1}}] &= 1 - b_{i_p}^2(1-w_{i_pi_{p-1}}^2) - b_{j_p}^2(1-w_{j_pi_{p-1}}^2) + 2b_{i_p}b_{j_p}w_{i_pi_{p-1}}w_{j_pi_{p-1}},\\
        \mathbb{E}[\phi'(\mathcal{U}^{\ell+p}_{i_p})\mid w_{i_pi_{p-1}}] &= \Phi\Bigl(\dfrac{w_{i_p i_{p-1}} a_{i_{p-1}}}{\sqrt{1 - a^2_{i_{p-1}}}}\Bigr),\\
       \mathbb{E}[\phi'(\mathcal{U}^{\ell+p}_{j_p})\mid w_{j_pi_{p-1}}] &= \Phi\Bigl(\dfrac{w_{j_p i_{p-1}} a_{i_{p-1}}}{\sqrt{1 - a^2_{i_{p-1}}}}\Bigr).
    \end{split}
\end{equation*}
Therefore, the expectation in this case is given by:
\begin{equation*}
    \mathbb{E}[(\mathcal{V}^{\ell+p}_{i_{p-1}})^4\phi'(\mathcal{U}^{\ell+p}_{i_p})\phi'(\mathcal{U}^{\ell+p}_{j_p})] = 
    \begin{cases}
    \dfrac{3}{2} & i_p=j_p,\\[15pt]
    \dfrac{3}{4} + \dfrac{2}{\pi}b_{i_p}b_{j_p}a^2_{i_{p-1}}(3 - a_{i_{p-1}}^2(b_{i_p}^2+b_{j_p}^2)) & i_p\neq j_p.
    \end{cases}
\end{equation*}
To compute the sum, we now notice that equality of indices in one layer ($i_p=j_p$) amounts to multiplying the product $\prod_{p=1}^k\mathbb{E}[(\mathcal{V}^{\ell+p}_{i_{p-1}})^2(\mathcal{V}^{\ell+p}_{j_{p-1}})^2\phi'(\mathcal{U}^{\ell+p}_{i_p})\phi'(\mathcal{U}_{j_p}^{\ell+p})]$ by $6+O(1/M^{3/2})$ and for every pair of indices $(i_p, j_p)$ there are only $n_{\ell+p}$ summands with this multiplier and $n_{\ell+p}(n_{\ell+p}-1)$ summands without it. We can also see that if we computed the sum with all the pairs of indices not equal, we would get $\prod_{p=0}^{k}\Bigl({n_{\ell+p}}/{2}\Bigr)^2\Bigl(1 -1/n_{\ell+p} +O(1/M^{3/2})\Bigr)$. Therefore, we get the desired expression for the expectation of the product:

\begin{equation*}
\begin{split}
    {\mathbb{E}\Bigl[\prod_{p=0}^k (\mathcal{N}^{\ell+p}_\delta)^2]} &=
     \prod_{p=0}^{k}\Bigl(\dfrac{\sigma_w^2}{n_{\ell+p}}\Bigr)^2\prod_{p=0}^{k}\Bigl(\dfrac{n_{\ell+p}}{2}\Bigr)^2 \Bigl(1 -\dfrac{1}{n_{\ell+p}} + \dfrac{6}{n_{\ell+p}} + O\Bigl(\dfrac{1}{M^{3/2}} \Bigr)\Bigr)\\
     &=\prod_{p=0}^{k}\Bigl(\dfrac{\sigma_w^2}{2}\Bigr)^2 
     \Bigl(1 + \dfrac{5}{n_{\ell+p}} + O\Bigl(\dfrac{1}{M^{3/2}} \Bigr)\Bigr)\\
     &= \prod_{p=0}^k \mathbb{E}[(\mathcal{N}^{\ell+p}_\delta)^2]\Bigl( 1 + O\Bigl(\dfrac{1}{M^{3/2}}\Bigr)\Bigr).
\end{split}
\end{equation*}

\end{proof}

\begin{lemma}[Dispersion of $\Theta_W(x,x)$]\label{lemma:theta_w}{Consider a fully-connected DNN of depth $L$ defined in \eqref{eq:nn} initialized as in \eqref{eq:init}. The input dimension is given by $n_0=\alpha_0 M$, the output dimension is $1$, and the hidden layers have constant width $M$.  The activation function in the hidden layers is ReLU, i.e. $\phi(x)= x \mathbbm{1}\{x>0\}$, and the output layer is linear. Assume also that the biases are initialized to zero, i.e. $\sigma_b=0$, and the input data is normalized. Then the component of the NTK corresponding to the weights $\Theta_W(x,x):= \sum_{\ell=1}^L \sum_{ij} \Bigl(\dfrac{\partial f(x)}{\partial \W_{ij}^\ell} \Bigr)^2$ has the following properties at initialization: 
\begin{equation}
    \mathbb{E}[\Theta_W(x,x)] = \Bigl(\dfrac{\sigma_w^2}{2}\Bigr)^{L-1} \Bigl(1 + \dfrac{M}{n_0}(L-1)\Bigr),
\end{equation}

\begin{equation}
    \dfrac{\mathbb{E}[\Theta_W^2(x,x)]}{\mathbb{E}^2[\Theta_W(x,x)]} \xrightarrow[\substack{M\to\infty, L\to\infty, \\L/M \to \lambda\in\mathbb{R}}]{} \dfrac{1}{2\lambda} e^{5\lambda} \Biggl( 1  - \dfrac{1}{4\lambda}(1 - e^{-4\lambda}) \Biggr).
\end{equation}
}
\end{lemma}
\begin{proof} 
Using backpropagation formulas for the gradients, we can rewrite the NTK as follows:
\begin{equation*}
    \begin{split}
        \Theta_W(x,x) &= \sum_{\ell=1}^L\sum_{i=1}^{n_\ell}\sum_{j=1}^{n_{\ell-1}} (\g_i^\ell)^2(\x^{\ell-1}_j)^2\\
                      &= \sum_{\ell=1}^L \| \g^\ell \times \x^{\ell-1}\|^2 = \sum_{\ell=1}^L \|\g^\ell\|^2
                      \|\x^{\ell-1}\|^2\\
                      &= \sum_{\ell=1}^L \|\g^L\|^2\|\x^0\|^2 \prod_{j=\ell}^{L-1}\dfrac{\|\g^j\|^2}{\|\g^{j+1}\|^2} \prod_{k=1}^{\ell-1}\dfrac{\|\x^k\|^2}{\|\x^{k-1}\|^2}\\
                      &= \sum_{\ell=1}^L \|\g^L\|^2\|\x^0\|^2 \prod_{j=\ell}^{L-1}\mathcal{N}_{\delta}^{j} \prod_{k=1}^{\ell-1}\mathcal{N}_{x}^{k}.
    \end{split}
\end{equation*}
Here for the simplicity of notation we omit the dependence on $\g^\ell$ and $\x^{\ell-1}$ on the input $x$.

If the last layer has a linear activation and the input data is normalized, we also have that $\|\g^L\|^2\|\x^0\|^2=1$. Then, using the results about expectations of $\mathcal{N}_\delta^{\ell}$ and $\mathcal{N}_x^{\ell}$ from Lemma \ref{lemma:x_ratio} and Lemma \ref{lemma:d_ratio}, as well as the results about correlations from Lemma \ref{lemma:gia}, we can write the following for the expectation of $\Theta_W(x,x)$:

\begin{equation*}
    \begin{split}
        \mathbb{E}[\Theta_W(x,x)]  =  \sum_{\ell=1}^L \prod_{j=\ell}^{L-1} \dfrac{\sigma_w^2}{2} \prod_{k=1}^{\ell-1} \dfrac{\sigma_w^2}{2} \dfrac{n_k}{n_{k-1}} = \Bigl(\dfrac{\sigma_w^2}{2}\Bigr)^{L-1} \sum_{\ell=1}^L \dfrac{n_{\ell-1}}{n_{0}}
    \end{split}
\end{equation*}
And for constant width of hidden layers, i.e. $n_\ell = M, \ell=1,\dots,L-1$, this simplifies to 
\begin{equation*}
    \begin{split}
        \mathbb{E}[\Theta_W(x,x)] = \Bigl(\dfrac{\sigma_w^2}{2}\Bigr)^{L-1} \Bigl(1 + \dfrac{M}{n_0}(L-1) \Bigr) 
        \propto \Bigl(\dfrac{\sigma_w^2}{2}\Bigr)^{L} \dfrac{ML}{n_0}
    \end{split}
\end{equation*}
Now we consider the second moment of the NTK, which is given by:
\begin{equation*}
    \begin{split}
        \mathbb{E}[\Theta_W^2(x,x)] &= \sum_{\ell=1}^L \mathbb{E}[\theta_\ell^2] + 2\sum_{1\leq l_1<l_2\leq L} \mathbb{E}[\theta_{\ell_1}\theta_{\ell_2}],\\
        \theta_\ell &= \prod_{j=\ell}^{L-1} \mathcal{N}_\delta^{j}\prod_{k=1}^{\ell-1}\mathcal{N}_x^{k}, \quad \ell=1,\dots,L.
    \end{split}
\end{equation*}
We can open the expectation of the squared terms defined above as follows:
\begin{equation*}
    \begin{split}
        \mathbb{E}[\theta_\ell^2] &= \mathbb{E}\Bigl[\prod_{j=\ell}^{L-1} (\mathcal{N}_\delta^{j})^2 \Bigr] \mathbb{E}\Bigl[\prod_{k=1}^{\ell-1} (\mathcal{N}_x^{k})^2\Bigr]\\
        &=  \Bigl(\dfrac{\sigma_w^2}{2}\Bigr)^{2(L-1)} \Bigl(\dfrac{n_{\ell-1}}{n_{0}}\Bigr)^2 \prod_{j=\ell}^{L-1} \Bigl(1+\dfrac{5}{n_j}+O\Bigl(\dfrac{1}{M^{3/2}} \Bigr)\Bigr) \prod_{k=1}^{\ell-1} \Bigl(1+\dfrac{5}{n_k}\Bigr),
    \end{split}
\end{equation*}  
which simplifies to the following expressions in case of constant width $M$:
\begin{equation*}
    \begin{split}
        \mathbb{E}[\theta_\ell^2] &= \Bigl(\dfrac{\sigma_w^2}{2}\Bigr)^{2(L-1)}\Bigl(\dfrac{M}{n_{0}}\Bigr)^2 \Bigl(1+\dfrac{5}{M}+O\Bigl(\dfrac{1}{M^{3/2}} \Bigr)\Bigr)^{L-1}, \quad \ell>1, \\
        \mathbb{E}[\theta_1^2] &= \Bigl(\dfrac{\sigma_w^2}{2}\Bigr)^{2(L-1)} \Bigl(1+\dfrac{5}{M}+O\Bigl(\dfrac{1}{M^{3/2}} \Bigr)\Bigr)^{L-1}.
    \end{split}
\end{equation*} 
And the mixed terms with $1\leq\ell_1<\ell_2\leq L$ can be calculated as follows:
 \begin{equation*}
    \begin{split}       
        \mathbb{E}[\theta_{\ell_1}\theta_{\ell_2}] &= \mathbb{E}\Bigl[ \prod_{j=\ell_2}^{L-1} (\mathcal{N}_\delta^{j})^2 \Bigr]   \mathbb{E}\Bigl[\prod_{p=\ell_1}^{\ell_2-1} \mathcal{N}_\delta^{p}\mathcal{N}_x^{p}\Bigr] \prod_{k=1}^{\ell_1-1} \mathbb{E}[(\mathcal{N}_x^{k})^2]\\
        &= \Bigl(\dfrac{\sigma_w^2}{2}\Bigr)^{2(L-1)}\Bigl(\dfrac{n_{\ell_2-1}n_{\ell_1-1}}{n_{0}^2}\Bigr) \prod_{j=\ell_2}^{L-1}\Bigl(1+\dfrac{5}{n_j}+ O\Bigl(\dfrac{1}{M^{3/2}}\Bigr)\Bigr) \prod_{p=\ell_1}^{\ell_2-1} \Bigl(1 + \dfrac{1}{n_{p}} + O\Bigl(\dfrac{1}{M^{3/2}}\Bigr) \Bigr) \prod_{k=1}^{\ell_1-1}  \Bigl(1+ \dfrac{5}{n_k}\Bigr),  \\
    \end{split}
\end{equation*}
which for constant width $M$ simplifies to 
 \begin{equation*}
    \begin{split}       
        \mathbb{E}[\theta_{\ell_1}\theta_{\ell_2}] &=\Bigl(\dfrac{\sigma_w^2}{2}\Bigr)^{2(L-1)} \Bigl(\dfrac{M}{n_{0}}\Bigr)^2\Bigl(1+\dfrac{5}{M}+O\Bigl(\dfrac{1}{M^{3/2}} \Bigr)\Bigr)^{L-1-\Delta_\ell }\Bigl(1+\dfrac{1}{M}+O\Bigl(\dfrac{1}{M^{3/2}} \Bigr)\Bigr)^{\Delta_\ell },  \quad \ell_1>1,\\
        \mathbb{E}[\theta_{1}\theta_{\ell_2}] &= \Bigl(\dfrac{\sigma_w^2}{2}\Bigr)^{2(L-1)} \Bigl(\dfrac{M}{n_{0}}\Bigr)\Bigl(1+\dfrac{5}{M}+O\Bigl(\dfrac{1}{M^{3/2}} \Bigr)\Bigr)^{L- \ell_2}\Bigl(1+\dfrac{1}{M}+O\Bigl(\dfrac{1}{M^{3/2}} \Bigr)\Bigr)^{\ell_2-1}.
    \end{split}
\end{equation*}
To make the notation lighter, we will denote $x:=1+{5}/{M}+O\Bigl({M^{-3/2}} \Bigr)$, $y:=1+{1}/{M}+O\Bigl({M^{-3/2}} \Bigr)$, $a:={\sigma_w^2}/{2}$ and $\lambda := L/M$ here and in the following proofs.
Then we can rewrite the two sums that comprise the second moment of the NTK as follows:
\begin{equation*}
\begin{split}  
    \sum_{\ell=1}^L\mathbb{E}[\theta_\ell^2] &= a^{2(L-1)}x^{L-1}\Bigl(\dfrac{M^2}{n_0^2}(L-1) + 1\Bigr) \\
    &= a^{2(L-1)}\dfrac{M^2L^2}{n_0^2} \Biggl[x^{L-1}\dfrac{1}{\lambda M} +O\Bigl(\dfrac{1}{M^{3/2}}\Bigr)  \Biggr] = a^{2(L-1)}\dfrac{M^2L^2}{n_0^2} \Biggl[x^{L-1}\dfrac{1}{\lambda M} +O\Bigl(\dfrac{1}{M^{3/2}}\Bigr)  \Biggr],
\end{split}
\end{equation*}

 \begin{equation*}
    \begin{split}       
        \sum_{1\leq \ell_1<\ell_2\leq L} \mathbb{E}[\theta_{\ell_1}\theta_{\ell_2}] = & \ a^{2(L-1)} \dfrac{M^2}{n_{0}^2} \sum_{\Delta_\ell = 1}^{L-2} (L-1-\Delta_\ell) \ x^{L-1-\Delta_\ell}y^{\Delta_\ell} + a^{2(L-1)} \dfrac{M}{n_{0}} \sum_{\ell_2=2}^Lx^{L-\ell_2}y^{\ell_2-1}\\
        = & \ a^{2(L-1)} \dfrac{M^4}{16n_{0}^2} \bigl( (L-2)yx^L - (L-1)y^2x^{L-1} + xy^L \bigr)+ a^{2(L-1)} \dfrac{M^2}{4n_{0}} (yx^{L-1}-y^L)\\
        = & \ a^{2(L-1)}\dfrac{M^2L^2}{n_0^2}\Biggl[ x^{L}\Biggl(\dfrac{1}{4\lambda}\Bigl(1-\dfrac{5}{M}\Bigr) - \dfrac{1}{16\lambda^2}\Bigl(1+\dfrac{5-4\alpha_0}{M}\Bigr) +  O\Bigl(\dfrac{1}{M^{3/2}}\Bigr)\Biggr)  +\\
        &\qquad\qquad\qquad + y^L\Biggl(\dfrac{1}{16\lambda^2}\Bigl(1+\dfrac{5-4\alpha_0}{M}\Bigr) + O\Bigl(\dfrac{1}{M^{3/2}}\Bigr)\Biggr) \Biggr]
        \end{split}
\end{equation*}
Therefore, the complete expression for the second moment of $\Theta_W(x,x)$ is given by:
\begin{equation*}
    \begin{split}  
    \sum_{\ell=1}^L \mathbb{E}[\theta_\ell^2] + 2\sum_{1\leq \ell_1<\ell_2\leq L} \mathbb{E}[\theta_{\ell_1}\theta_{\ell_2}] = & \ a^{2(L-1)}\dfrac{M^2L^2}{n_0^2}\Biggl[ x^{L}\Biggl(\dfrac{1}{2\lambda}\Bigl(1-\dfrac{3}{M}\Bigr) - \dfrac{1}{8\lambda^2}\Bigl(1+\dfrac{5-4\alpha_0}{M}\Bigr) + O\Bigl(\dfrac{1}{M^{3/2}}\Bigr)\Biggr)  +\\
    &\qquad\qquad\qquad + y^L\Biggl(\dfrac{1}{8\lambda^2}\Bigl(1+\dfrac{5-4\alpha_0}{M}\Bigr) + O\Bigl(\dfrac{1}{M^{3/2}}\Bigr)\Biggr) \Biggr]\\
    & \ a^{2(L-1)}\dfrac{M^2L^2}{n_0^2}\Biggl[ x^{L}\Biggl(\dfrac{1}{2\lambda} - \dfrac{1}{8\lambda^2}+ O\Bigl(\dfrac{1}{M}\Bigr)\Biggr) + y^L\dfrac{1}{8\lambda^2}+ O\Bigl(\dfrac{1}{M}\Bigr)\Biggr) \Biggr]\\
    \end{split}
\end{equation*}
One can see that in the limit $L\to\infty, \ M\to\infty, \ L/M\to\lambda\in \mathbb{R}$, we have $x^{L}\to e^{5\lambda}$ and $y^L\to e^{\lambda}$. Therefore, we can find the limit of the desired ratio:
\vspace{-0.1ex}
\begin{equation*}
    \begin{split}  
    \dfrac{\mathbb{E}[\Theta_W^2(x,x)]}{\Bigl(\dfrac{\sigma_w^2}{2}\Bigr)^{2(L-1)}\dfrac{L^2M^2}{n_0^2}} \xrightarrow[\substack{M\to\infty, L\to\infty, \\L/M \to \lambda\in\mathbb{R}}]{} \dfrac{1}{2\lambda} e^{5\lambda} \Biggl( 1 - \dfrac{1}{4\lambda}(1 - e^{-4\lambda}) \Biggr).
    \end{split}
\end{equation*}
\end{proof}

\begin{lemma}[Dispersion of $\Theta_b(x,x)$]\label{lemma:theta_b}{Consider a fully-connected DNN  of depth $L$ defined in \eqref{eq:nn} initialized as in \eqref{eq:init}. The input dimension is given by $n_0=\alpha_0 M$, the output dimension is $1$, and the hidden layers have constant width $M$.  The activation function in the hidden layers is ReLU, i.e. $\phi(x)= x \mathbbm{1}\{x>0\}$, and the output layer is linear. Assume also that the biases are initialized to zero, i.e. $\sigma_b=0$, and the input data is normalized. Then the component of the NTK corresponding to the biases $\Theta_b(x,x):= \sum_{\ell=1}^L \sum_{i} \Bigl(\dfrac{\partial f(x)}{\partial \bias_{i}^\ell} \Bigr)^2$ has the following properties at initialization: 

\begin{equation}
    \qquad\mathbb{E}[\Theta_b(x,x)] = \begin{cases}
        \dfrac{\Bigl(\dfrac{\sigma_w^2}{2}\Bigr)^L - 1}{\dfrac{\sigma_w^2}{2}-1}  &\text{if } \dfrac{\sigma_w^2}{2}\neq 1\\
        L &\text{if } \dfrac{\sigma_w^2}{2} = 1\\
\end{cases}
\end{equation}

\begin{equation}
    \dfrac{\mathbb{E}[\Theta_b^2(x,x)]}{\mathbb{E}^2[\Theta_b(x,x)]} \xrightarrow[\substack{M\to\infty, L\to\infty, \\L/M \to \lambda\in\mathbb{R}}]{} \begin{cases}
            1  &\text{if } \dfrac{\sigma_w^2}{2} < 1\\
            \\[1pt]
            \dfrac{2}{25\lambda^2}(e^{5\lambda}-1)-\dfrac{2}{5\lambda} &\text{if } \dfrac{\sigma_w^2}{2} = 1\\
            \\[1pt]
            e^{5\lambda} &\text{if } \dfrac{\sigma_w^2}{2} > 1\\
\end{cases}
\end{equation}
}
\end{lemma}
\begin{proof} Using backpropagation equations \eqref{eq:backprop}, we can obtain the following expression for $\Theta_b(x,x)$:

\begin{equation*}
\begin{split}
    \Theta_b(x,x) = \sum_{\ell=1}^L \|\g^\ell\|^2 &= \|\g^L\|^2 \sum_{\ell=1}^L \prod_{j=\ell}^{L-1} \dfrac{\|\g^j\|^2}{\|\g^{j+1}\|^2} =  \sum_{\ell=1}^L \prod_{j=\ell}^{L-1} \mathcal{N}_\delta^j.
\end{split}
\end{equation*}
In this lemma, we will again denote $a:=\sigma_w^2/2$ and $x:=1+5/M+O(1/M^{3/2})$. And in the following computations, we will need to consider cases with $a\neq 1$ and $a = 1$ separately.

\textbf{Case 1: $a\neq 1$.}
In this case the expectation is given by a sum of a geometric progression:
\begin{equation*}
\begin{split}
    \mathbb{E}[\Theta_b(x,x)] = \sum_{\ell=1}^L \prod_{j=\ell}^{L-1} \mathbb{E}[\mathcal{N}_{\delta}^{j}] = \sum_{\ell=1}^L a^{L-\ell} = \dfrac{a^{L}-1}{a - 1}
\end{split}
\end{equation*}
And for the second moment we can write:
\begin{equation*}
\begin{split}
    \mathbb{E}[\Theta_b(x,x)^2] &= \mathbb{E}\Biggl[\Bigl(\sum_{\ell=1}^L \prod_{j=\ell}^{L-1} \mathcal{N}_{\delta}^{j}\Bigr)^2\Biggr] = \sum_{\ell=1}^L \mathbb{E}\Bigl[\prod_{j=\ell}^{L-1} (\mathcal{N}_{\delta}^{j})^2\Bigr]+ 2 \sum_{1\leq \ell_1<\ell_2\leq L}\mathbb{E}\Bigl[\prod_{j=\ell_1}^{\ell_2-1}\mathcal{N}_{\delta}^{j}\prod_{k=\ell_2}^{L-1}\Bigl(\mathcal{N}_{\delta}^{k}\Bigr)^2\Bigr]\\
    & = \sum_{\ell=1}^L a^{2(L-\ell)}\prod_{j=\ell}^{L-1} \Bigl(1+\dfrac{5}{n_j} + O\Bigl(\dfrac{1}{M^{3/2}} \Bigr)\Bigr) + 2 \sum_{1\leq \ell_1<\ell_2\leq L}a^{2L-\ell_1-\ell_2}\prod_{k=\ell_2}^{L-1}\Bigl(1+\dfrac{5}{n_k} + O\Bigl(\dfrac{1}{M^{3/2}} \Bigr)\Bigr) 
\end{split}
\end{equation*}
For constant width $M$ the above expression simplifies to the following sum:
\begin{equation*}
\begin{split}
    \mathbb{E}[\Theta_b(x,x)^2] & = \sum_{\ell=1}^L a^{2(L-\ell)}x^{L-\ell} + 2 \sum_{1\leq \ell_1<\ell_2\leq L}a^{2L-\ell_1-\ell_2}x^{L-\ell_2}\\
    & = \sum_{\ell=1}^L a^{2(L-\ell)}x^{L-\ell} + 2 \sum_{\ell_1=1}^{L-1}a^{2(L-\ell)}x^{L-\ell} \sum_{\Delta_\ell=1}^{L-\ell}a^{-\Delta_\ell}x^{-\Delta_\ell}
\end{split}
\end{equation*}
And the involved terms can be further calculated explicitly as follows:
\begin{equation*}
\begin{split}
     \mathbb{E}[\Theta_b(x,x)^2] & = \dfrac{a^{2L}x^L-1}{a^2x-1} + \dfrac{2}{ax-1}\sum_{\ell=1}^{L-1}a^{2(L-l)}x^{L-l} (1 - a^{l-L}x^{l-L})\\
    & = \dfrac{a^{2L}x^L-1}{a^2x-1} + \dfrac{2}{ax-1} \Bigl( \dfrac{a^{2L}x^L-1}{a^2x-1} - 1  - \dfrac{a^{L}-a}{a-1} \Bigr) \\
    & = \dfrac{a^{2L}x^L-1}{a^2x-1} \dfrac{ax+1}{ax-1} - \dfrac{2}{ax-1}\dfrac{a^L-1}{a-1}\\
    & = \dfrac{1}{(a-1)^2}\Biggl[a^{2L}x^L \Bigl( 1 +O\Bigl(\dfrac{1}{M}\Bigr)\Bigr)  - 2a^L\Bigl(1 +O\Bigl(\dfrac{1}{M}\Bigr)\Bigr) + 1 +O\Bigl(\dfrac{1}{M}\Bigr)\Bigr)\Biggr]
\end{split}
\end{equation*}

If $a<1$, the expectation and the second moment have finite limits:

\begin{equation*}
\begin{split}
    \mathbb{E}[\Theta_b(x,x)] &\xrightarrow[\substack{M\to\infty, L\to\infty, \\L/M \to \lambda\in\mathbb{R}}]{} \dfrac{1}{1-a}\\
    \mathbb{E}[\Theta_b(x,x)^2] &\xrightarrow[\substack{M\to\infty, L\to\infty, \\L/M \to \lambda\in\mathbb{R}}]{} \dfrac{-1}{a^2-1}\dfrac{a+1}{a-1} - \dfrac{2}{a-1}\dfrac{-1}{a-1} = \dfrac{1}{(a-1)^2}
\end{split}
\end{equation*}
Therefore for $a<1$ we have
\begin{equation*}
    \dfrac{\mathbb{E}[\Theta_b^2(x,x)]}{\mathbb{E}^2[\Theta_b(x,x)]}\xrightarrow[\substack{M\to\infty, L\to\infty, \\L/M \to \lambda\in\mathbb{R}}]{} 1
\end{equation*}
On the other hand, if $a>1$ then the limits are infinite but there is a finite limit of the ratio:
\begin{equation*}
     \dfrac{\mathbb{E}[\Theta_b(x,x)^2]}{a^{2L}/(a-1)^2} \xrightarrow[\substack{M\to\infty, L\to\infty, \\L/M \to \lambda\in\mathbb{R}}]{} e^{5\lambda}
\end{equation*}

\item \textbf{Case 2: $a=1$.}
In this case, the expectation is just a sum of ones, so we have 
\begin{equation*}
\begin{split}
    \mathbb{E}[\Theta_b(x,x)] = \sum_{\ell=1}^L 1 = L
\end{split}
\end{equation*}
And the second moment can be calculated as follows:
\begin{equation*}
\begin{split}
    \mathbb{E}[\Theta_b(x,x)^2] & = \sum_{\ell=1}^L x^{L-\ell}  + 2 \sum_{\ell=1}^{L-1}x^{L-\ell} \sum_{\Delta_\ell=1}^{L-l}x^{-\Delta_\ell}
    = \dfrac{x^L-1}{x-1} + \dfrac{2}{x-1}\Bigl( \sum_{\ell=1}^{L-1} x^{L-\ell} - \sum_{\ell=1}^{L-1} 1\Bigr) \\
    &= \dfrac{x^L-1}{x-1} + \dfrac{2}{x-1}\Bigl( \dfrac{x^L-x}{x-1} -L +1 \Bigr) = M^2 \Biggl[ x^L \Bigl( \dfrac{1}{5M} + \dfrac{2}{25}\Bigr) - \dfrac{1}{5M} - \dfrac{10\lambda +2}{25}\Biggr]\\
    & = \dfrac{x^L-1}{x-1} + \dfrac{2}{x-1}\Bigl( \dfrac{x^L-x}{x-1} -L +1 \Bigr) = L^2 \Biggl[ x^L \Bigl(\dfrac{2}{25\lambda^2} +O\Bigl(\dfrac{1}{M}\Bigr) \Bigr) - \dfrac{2 }{5\lambda} - \dfrac{2}{25\lambda^2} +O\Bigl(\dfrac{1}{M}\Bigr) \Biggr]\\
\end{split}
\end{equation*}
Then for the desired ratio we have the following result in the limit:
\begin{equation*}
\begin{split}
    \dfrac{\mathbb{E}[\Theta_b(x,x)^2]}{L^2} \xrightarrow[\substack{M\to\infty, L\to\infty, \\L/M \to \lambda\in\mathbb{R}}]{} \dfrac{2}{25\lambda^2}(e^{5\lambda}-1) - \dfrac{2}{5\lambda},
\end{split}
\end{equation*}
which completes the proof for all the cases.
\end{proof}

\begin{lemma}[Dispersion of $\Theta_W(x,x)\Theta_b(x,x)$]\label{lemma:theta_wb}{Consider a fully-connected DNN  of depth $L$ defined in \eqref{eq:nn} initialized as in \eqref{eq:init}. The input dimension is given by $n_0=\alpha_0 M$, the output dimension is $1$, and the hidden layers have constant width $M$.  The activation function in the hidden layers is ReLU, i.e. $\phi(x)= x \mathbbm{1}\{x>0\}$, and the output layer is linear. Assume also that the biases are initialized to zero, i.e. $\sigma_b=0$, and the input data is normalized. Then the following statements hold: 
\begin{enumerate}
    \item In the chaotic phase, i.e. if $\sigma_w^2/2 > 1$:
    \begin{equation}
        \dfrac{\mathbb{E}[\Theta_W(x,x)\Theta_b(x,x)]}{\mathbb{E}^2[\Theta_W(x,x)]} \xrightarrow[\substack{M\to\infty, L\to\infty, \\L/M \to \lambda\in\mathbb{R}}]{} 0
    \end{equation}

\item In the ordered phase, i.e. if $\sigma_w^2/2 < 1$:
\begin{equation}
    \dfrac{\mathbb{E}[\Theta_W(x,x)\Theta_b(x,x)]}{\mathbb{E}^2[\Theta_b(x,x)]} \xrightarrow[\substack{M\to\infty, L\to\infty, \\L/M \to \lambda\in\mathbb{R}}]{} 0
\end{equation}

\item At the EOC, i.e. if $\sigma_w^2/2 = 1$:
\begin{equation}
    \dfrac{\mathbb{E}[\Theta_W(x,x)\Theta_b(x,x)]}{L^2} \xrightarrow[\substack{M\to\infty, L\to\infty, \\L/M \to \lambda\in\mathbb{R}}]{} \dfrac{1}{4\alpha_0}\Bigl( e^{5\lambda}\dfrac{9}{25\lambda^2} - e^\lambda \dfrac{1}{\lambda^2} - \dfrac{4}{5\lambda} + \dfrac{16}{25\lambda^2}\Bigr)
\end{equation}

\end{enumerate}
}
\end{lemma}
\begin{proof} We can decompose $\Theta_W(x,x)\Theta_b(x,x)$ into telescopic products as follows:
\begin{equation*}
    \begin{split}
        \Theta_W(x,x)\Theta_b(x,x) &= \sum_{\ell=1}^L \|\g^\ell\|^2\|\x^{\ell-1}\|^2  \sum_{\ell'=1}^L \|\g^{\ell'}\|^2\\
        &=  \|\g^L\|^4\|\x^0\|^2\sum_{\ell=1}^L \prod_{j=\ell}^{L-1}\dfrac{\|\g^j\|^4}{\|\g^{j+1}\|^4} \prod_{k=1}^{\ell-1}\dfrac{\|\x^k\|^2}{\|\x^{k-1}\|^2}\\
        &+  \|\g^L\|^4\|\x^0\|^2 \sum_{1\leq \ell_1 < \ell_2 \leq L}\prod_{p=\ell_1}^{\ell_2-1}\dfrac{\|\g^p\|^2}{\|\g^{p+1}\|^2}\prod_{j=\ell_2}^{L-1}\dfrac{\|\g^j\|^4}{\|\g^{j+1}\|^4} \prod_{k=1}^{\ell_2-1}\dfrac{\|\x^k\|^2}{\|\x^{k-1}\|^2}\\
        &+  \|\g^L\|^4\|\x^0\|^2 \sum_{1\leq \ell_1 < \ell_2 \leq L}\prod_{p=\ell_1}^{\ell_2-1}\dfrac{\|\g^p\|^2}{\|\g^{p+1}\|^2}\prod_{j=\ell_2}^{L-1}\dfrac{\|\g^j\|^4}{\|\g^{j+1}\|^4} \prod_{k=1}^{\ell_1-1}\dfrac{\|\x^k\|^2}{\|\x^{k-1}\|^2}
    \end{split}
\end{equation*}
Then, as in the previous lemmas, we can calculate the expectation using the results of Lemmas \ref{lemma:x_ratio}, \ref{lemma:d_ratio} and \ref{lemma:gia}:
\begin{equation*}
\begin{split}
        \mathbb{E}[\Theta_W(x,x)\Theta_b(x,x)]
        &= \sum_{\ell=1}^L \mathbb{E}\Bigl[\prod_{j=\ell}^{L-1}(\mathcal{N}_\delta^j)^2\Bigr] \prod_{k=1}^{\ell-1}\mathbb{E}[\mathcal{N}_x^k] \ +   \sum_{1\leq \ell_1 < \ell_2 \leq L}\prod_{p=\ell_1}^{\ell_2-1}\mathbb{E}[\mathcal{N}_\delta^p\mathcal{N}_x^p]\mathbb{E}\Bigl[\prod_{j=\ell_2}^{L-1}(\mathcal{N}_\delta^j)^2\Bigr] \prod_{k=1}^{\ell_1-1}\mathbb{E}[\mathcal{N}_x^k]\\
        &+ \sum_{1\leq \ell_1 < \ell_2 \leq L}\prod_{p=\ell_1}^{\ell_2-1}\mathbb{E}[\mathcal{N}_\delta^p]\mathbb{E}\Bigl[\prod_{j=\ell_2}^{L-1}(\mathcal{N}_\delta^j)^2\Bigr] \prod_{k=1}^{\ell_1-1}\mathbb{E}[\mathcal{N}_x^k]\\
        & = \sum_{\ell=1}^L a^{2L-\ell-1}\dfrac{n_{\ell-1}}{n_0}\prod_{j=\ell}^{L-1}\Bigl(1+\dfrac{5}{n_j} + \Bigl(\dfrac{1}{M^{3/2}} \Bigr)\Bigr)\\ 
        & +   \sum_{1\leq \ell_1 < \ell_2 \leq L}a^{2L-\ell_1-1}\dfrac{n_{\ell_2-1}}{n_0}\prod_{p=\ell_1}^{\ell_2-1}\Bigl(1+\dfrac{1}{n_p} + \Bigl(\dfrac{1}{M^{3/2}}\Bigr)\Bigr)\prod_{j=\ell_2}^{L-1}\Bigl(1+\dfrac{5}{n_j} + \Bigl(\dfrac{1}{M^{3/2}}\Bigr)\Bigr)\\
        &+ \sum_{1\leq \ell_1 < \ell_2 \leq L}a^{2L-\ell_2-1}\dfrac{n_{\ell_1-1}}{n_0}\prod_{j=\ell_2}^{L-1}\Bigl(1+\dfrac{5}{n_j} + \Bigl(\dfrac{1}{M^{3/2}}\Bigr)\Bigr),
    \end{split}
\end{equation*}
where we denoted $a:={\sigma_w^2}/{2}$. As in Lemma \ref{lemma:theta_b}, we will need to consider the cases with $a\neq 1$ and $a=1$ separately here. 

\textbf{Case 1: $a\neq 1$.} For constant width $M$ the above expression for the expectation simplifies to:
\begin{equation*}
    \begin{split}
        \mathbb{E}[\Theta_W(x,x)\Theta_b(x,x)] & = \sum_{\ell=1}^L \dfrac{n_{\ell-1}}{n_0}a^{2L-\ell-1}x^{L-\ell} +   \sum_{\ell_1=1}^{L-1}a^{2L-\ell_1-1}\sum_{\ell_2=\ell_1+1}^{L}\dfrac{n_{\ell_2-1}}{n_0}y^{\ell_2-\ell_1}x^{L-\ell_2}\\
        &+ \sum_{\ell_1=1}^{L-1}\dfrac{n_{\ell_1-1}}{n_0}\sum_{\ell_2=\ell_1+1}^{L}a^{2L-\ell_2-1}x^{L-\ell_2} \\
        & = a^{2(L-1)}x^{L-1}\Bigl(1+\dfrac{M}{n_0}\dfrac{1}{ax-1} \Bigr) - \dfrac{M}{n_0}\dfrac{a^{L-1}}{ax-1}\\
        & + a^{2(L-1)}x^{L-1}\dfrac{M}{n_0}\dfrac{1}{x-y}\dfrac{ay}{ax-1}
        - a^{2(L-1)}y^{L} \dfrac{M}{n_0}\dfrac{1}{x-y}\dfrac{ay}{ay-1}\\
        & + a^L\dfrac{M}{n_0}\dfrac{1}{x-y}\Bigl(\dfrac{-xy}{ax-1} 
         + \dfrac{y^2}{ay-1}\Bigr) \\
        & + a^{2(L-1)}x^{L-1}\dfrac{1}{ax-1}\Bigl(1 + \dfrac{M}{n_0}\dfrac{1}{ax-1} \Bigr) - \dfrac{a^{L-1}}{ax-1} - \dfrac{a^{L-1}}{ax-1}\dfrac{M}{n_0}(L-2) - \dfrac{M}{n_0}\dfrac{a^L x}{(ax-1)^2}\\
        & = a^{2(L-1)}x^{L-1} \Biggl[ \Bigl(1+\dfrac{M}{n_0}\dfrac{1}{ax-1} \Bigr)\Bigl( 1 + \dfrac{1}{ax-1}\Bigr) + \dfrac{M}{n_0}\dfrac{1}{x-y}\dfrac{ay}{ax-1}\Biggr]\\
        & - a^{2(L-1)}y^{L} \dfrac{M}{n_0}\dfrac{1}{x-y}\dfrac{ay}{ay-1}\\
        & + a^{L-1}\Biggl[ \dfrac{M}{n_0}\dfrac{a}{x-y}\Bigl(\dfrac{-xy}{ax-1} 
         + \dfrac{y^2}{ay-1}\Bigr) - \dfrac{1}{ax-1}\Bigl(1 + \dfrac{M(L-1)}{n_0}\Bigr) - \dfrac{M}{n_0}\dfrac{a x}{(ax-1)^2}\Biggr] \\
         & = a^{2(L-1)}x^{L-1}\dfrac{M}{4\alpha_0}\dfrac{a}{ax-1} \Biggl[ y + \dfrac{4\alpha_0 x}{M} + \dfrac{4x}{(ax-1)M} + O\Bigl(\dfrac{1}{M}\Bigr)\Biggr] \\
         & -a^{2(L-1)}y^L \dfrac{M}{4\alpha_0}\dfrac{a}{ay-1}\Biggl[ y + O\Bigl(\dfrac{1}{M}\Bigr) \Biggr] \\
         & + a^{L-1}\dfrac{M}{4\alpha_0}\dfrac{a}{ax-1} \Biggl[\dfrac{16}{M^2(ay-1)(ax-1)} - \dfrac{4\alpha_0}{Ma}\Bigl(1+\dfrac{L-1}{\alpha_0} \Bigr) \Biggr]\\
         & = \dfrac{M}{4\alpha_0}\dfrac{a}{a-1} \Biggl[ a^{2(L-1)}x^{L-1} - a^{2(L-1)}y^{L} - 4a^{L-1}\lambda + O\Bigl(\dfrac{1}{M}\Bigr)\Biggr],
    \end{split}
\end{equation*}
where we also denoted $x:=1+5/M+O(1/M^{3/2})$ and $y:=1+1/M+O(1/M^{3/2})$.
From the last expression, we see that $\mathbb{E}[\Theta_W(x,x)\Theta_b(x,x)]$ tends to zero if $a<1$, therefore in this case we get
\begin{equation}
    \dfrac{\mathbb{E}[\Theta_W(x,x)\Theta_b(x,x)]}{\mathbb{E}^2[\Theta_b(x,x)]} \xrightarrow[\substack{M\to\infty, L\to\infty, \\L/M \to \lambda\in\mathbb{R}}]{} 0,
\end{equation}
using the result of Lemma \ref{lemma:theta_b} that $\mathbb{E}[\Theta^2_b(x,x)]$ has a finite limit when $a<1$.

On the other hand, if $a > 1$, we can see that $\mathbb{E}[\Theta_W(x,x)\Theta_b(x,x)]$ contains polynomials of $M$ and $L$ of degree not larger than $1$. Therefore, we have 

\begin{equation}
    \dfrac{\mathbb{E}[\Theta_W(x,x)\Theta_b(x,x)]}{a^{2L}M^2L^2/n_0^2} \xrightarrow[\substack{M\to\infty, L\to\infty, \\L/M \to \lambda\in\mathbb{R}}]{} 0,
\end{equation}
which completes the proof for the case when $a\neq 1$.

\textbf{Case 2: $a = 1$.} 
In this case, the expression for the expectation with constant width $M$ is given by:
\begin{equation*}
    \begin{split}
        \mathbb{E}[\Theta_W(x,x)\Theta_b(x,x)] & = \sum_{\ell=1}^L \dfrac{n_{\ell-1}}{n_0}x^{L-\ell} +   \sum_{\ell_1=1}^{L-1}\sum_{\ell_2=\ell_1+1}^{L}\dfrac{n_{\ell_2-1}}{n_0}y^{\ell_2-\ell_1}x^{L-\ell_2}\\
        &+ \sum_{\ell_1=1}^{L-1}\dfrac{n_{\ell_1-1}}{n_0}\sum_{\ell_2=\ell_1+1}^{L}x^{L-\ell_2} \\
        & = \dfrac{M^2}{4\alpha_0}\Biggl[ x^{L-1}\Bigl(\dfrac{9}{25} +O\Bigl(\dfrac{1}{M}\Bigr)\Bigr) - y^L\Bigl(1 +O\Bigl(\dfrac{1}{M}\Bigr)\Bigr) - \dfrac{4\lambda}{5} + \dfrac{16}{25}\Biggr]
    \end{split}
\end{equation*}

\end{proof}

\begin{theorem}[Dispersion of the NTK at initialization]{Consider a fully-connected DNN  of depth $L$ defined in \eqref{eq:nn} initialized as in \eqref{eq:init}. The input dimension is given by $n_0=\alpha_0 M$, the output dimension is $1$, and the hidden layers have constant width $M$.  The activation function in the hidden layers is ReLU, i.e. $\phi(x)= x \mathbbm{1}\{x>0\}$, and the output layer is linear. Assume also that the biases are initialized to zero, i.e. $\sigma_b=0$, and the input data is normalized. Then the dispersion of the NTK at initialization is given by the following expressions: 
\begin{enumerate}
    \item In the \textbf{chaotic phase} ($a:={\sigma_w^2}/{2} > 1$), the NTK dispersion grows exponentially with the depth-to-width ratio $\lambda:={L}/{M}$ as 
    \begin{equation}
        \dfrac{\mathbb{E}[\Theta^2(x,x)]}{\mathbb{E}^2[\Theta(x,x)]}\xrightarrow[\substack{M\to\infty, L\to\infty, \\L/M \to \lambda\in\mathbb{R}}]{} \dfrac{1}{2\lambda} e^{5\lambda} \Biggl( 1  - \dfrac{1}{4\lambda}(1 - e^{-4\lambda}) \Biggr)
    \end{equation}
    \item At the \textbf{EOC} ($a=1$), the NTK dispersion grows exponentially with the depth-to-width ratio $\lambda$ as well, but with a slower rate given by
    \begin{equation}
    \begin{split}
        \dfrac{\mathbb{E}[\Theta^2(x,x)]}{\mathbb{E}^2[\Theta(x,x)]}\xrightarrow[\substack{M\to\infty, L\to\infty, \\L/M \to \lambda\in\mathbb{R}}]{} \dfrac{1}{(1+\alpha_0)^2\lambda} &\Biggl[ e^{5\lambda}\Bigl(\dfrac{1}{2} + \dfrac{16\alpha_0^2+36\alpha_0-25}{200\lambda} \Bigr)\\
        &+ e^\lambda \dfrac{1-4\alpha_0}{8\lambda} + \dfrac{2\alpha_0(4-\alpha_0)}{25\lambda} - \dfrac{2\alpha_0(1+\alpha_0)}{5}\Biggr]
    \end{split}
    \end{equation}
    \item In the \textbf{ordered phase} ($a < 1$), the NTK variance does not grow with $\lambda$ and we have
    \begin{equation}
        \dfrac{\mathbb{E}[\Theta^2(x,x)]}{\mathbb{E}^2[\Theta(x,x)]}\xrightarrow[\substack{M\to\infty, L\to\infty, \\L/M \to \lambda\in\mathbb{R}}]{}  1
    \end{equation}
\end{enumerate}}
\end{theorem}
\begin{proof}
We will consider the cases of the ordered phase ($a:={\sigma_w^2}/{2} < 1$), the chaotic phase ($a > 1$) and the EOC ($a=1$) separately. 

\item  \textbf{Case 1: Chaotic phase.} 
Using the results of Lemmas \ref{lemma:theta_w}, \ref{lemma:theta_b}, and \ref{lemma:theta_wb} and taking into account that $a>1$, we obtain the following limit: 
\begin{equation*}
    \begin{split}
     \dfrac{\mathbb{E}[\Theta_b(x,x)]}{\mathbb{E}[\Theta_W(x,x)]} = \dfrac{\dfrac{a^L-1}{a-1}}{a^{L-1} (1 + \dfrac{M}{n_0}(L-1))} = \dfrac{\dfrac{a-a^{-L+1}}{a-1}}{1 + \dfrac{M}{n_0}(L-1)} \xrightarrow[\substack{\\M\to\infty, L\to\infty, \\L/M \to \lambda\in\mathbb{R}}]{} 0
    \end{split}
\end{equation*}
Therefore, recalling that $\Theta(x,x) = \Theta_W(x,x) + \Theta_b(x,x)$, we get the ratio between the complete NTK and its component corresponding to weights:
\begin{equation*}
    \begin{split}
     \dfrac{\mathbb{E}^2[\Theta(x,x)]}{\mathbb{E}^2[\Theta_W(x,x)]}&= \Bigl( 1 + \dfrac{\mathbb{E}[\Theta_b(x,x)]}{\mathbb{E}[\Theta_W(x,x)]}\Bigr)^2 \xrightarrow[\substack{\\M\to\infty, L\to\infty, \\L/M \to \lambda\in\mathbb{R}}]{} 1
    \end{split}
\end{equation*}
Similarly, from Lemmas \ref{lemma:theta_w}, \ref{lemma:theta_b}, \ref{lemma:theta_wb}, we can also obtain the following limit:
\begin{equation*}
    \begin{split}
     \dfrac{\mathbb{E}[\Theta^2(x,x)]}{\mathbb{E}^2[\Theta_W(x,x)]}&= \dfrac{\mathbb{E}[\Theta_W^2(x,x)]}{\mathbb{E}^2[\Theta_W(x,x)]} + \dfrac{\mathbb{E}[\Theta^2_b(x,x)]}{\mathbb{E}^2[\Theta_W(x,x)]} +  \dfrac{\mathbb{E}[\Theta_W(x,x)\Theta_b(x,x)]}{\mathbb{E}^2[\Theta_W(x,x)]} \xrightarrow[\substack{\\M\to\infty, L\to\infty, \\L/M \to  \lambda\in\mathbb{R}}]{}\dfrac{\mathbb{E}[\Theta_W^2(x,x)]}{\mathbb{E}^2[\Theta_W(x,x)]}
    \end{split}
\end{equation*}
Therefore, the dispersion of the NTK is determined by $\Theta_W(x,x)$ in the infinite-depth-and-width limit in case of the initialization in the chaotic phase. Then we have the following expression for the dispersion in the limit:
\begin{equation*}
    \begin{split}
     \dfrac{\mathbb{E}[\Theta^2(x,x)]}{\mathbb{E}^2[\Theta(x,x)]}\xrightarrow[\substack{M\to\infty, L\to\infty, \\L/M \to \lambda\in\mathbb{R}}]{} \dfrac{1}{2\lambda} e^{5\lambda} \Biggl( 1  - \dfrac{1}{4\lambda}(1 - e^{-4\lambda}) \Biggr),
    \end{split}
\end{equation*}
which completes the first part of the proof.
\item  \textbf{Case 2: Ordered phase.} 
In the ordered phase, i.e. if $a<1$, we have that $a^L\to 0$ as $L\to\infty$, so Lemmas \ref{lemma:theta_w}, \ref{lemma:theta_b} and \ref{lemma:theta_wb} suggest different relations between the terms of the NTK:
\begin{equation*}
    \begin{split}
     \dfrac{\mathbb{E}[\Theta_W(x,x)]}{\mathbb{E}[\Theta_b(x,x)]}= \dfrac{a^{L-1} (1 + \dfrac{M}{n_0}(L-1))}{\dfrac{a^L-1}{a-1}} \xrightarrow[\substack{\\M\to\infty, L\to\infty, \\L/M \to \lambda\in\mathbb{R}}]{} 0\\
     \mathbb{E}[\Theta^2_W(x,x)] \xrightarrow[\substack{\\M\to\infty, L\to\infty, \\L/M \to \lambda\in\mathbb{R}}]{} 0,\\ \mathbb{E}[\Theta_W(x,x)\Theta_b(x,x)] \xrightarrow[\substack{\\M\to\infty, L\to\infty, \\L/M \to \lambda\in\mathbb{R}}]{} 0\\
    \mathbb{E}[\Theta^2_b(x,x)] \xrightarrow[\substack{\\M\to\infty, L\to\infty, \\L/M \to \lambda\in\mathbb{R}}]{} \dfrac{1}{(a-1)^2}\\
    \end{split}
\end{equation*}
Therefore, the dispersion of the NTK is determined by the component corresponding to biases $\Theta_b$ in the limit in case of initialization in the ordered phase:
\begin{equation*}
    \begin{split}
     \dfrac{\mathbb{E}[\Theta^2(x,x)]}{\mathbb{E}^2[\Theta(x,x)]} \xrightarrow[\substack{M\to\infty, L\to\infty, \\L/M \to \lambda\in\mathbb{R}}]{} 1,
    \end{split}
\end{equation*}
which completes this part of the proof.

\item  \textbf{Case 3: EOC.} Here we have $a^\ell=1$ for any $\ell\in\mathbb{N}$. Therefore, we can simplify the expressions for expectations from Lemmas \ref{lemma:theta_w}, \ref{lemma:theta_b} and  \ref{lemma:theta_wb} as follows:

\begin{equation*}
\begin{split}
     \mathbb{E}[\Theta_W(x,x)] &= 1 + \dfrac{1}{\alpha_0}(L-1),\\
     \mathbb{E}[\Theta_b(x,x)] &= L.
\end{split}
\end{equation*}
Then the expectation of the complete NTK is given by:
\begin{equation*}
     \mathbb{E}[\Theta(x,x)] = \mathbb{E}[\Theta_W(x,x)] + \mathbb{E}[\Theta_b(x,x)] = \dfrac{L}{\alpha_0}\Bigl( 1 + \alpha_0 + \dfrac{\alpha_0}{L} - \dfrac{1}{L} \Bigr) \propto \dfrac{L}{\alpha_0}(1 + \alpha_0).
\end{equation*}
The squared NTK is given by $\Theta^2(x,x) = \Theta_W^2(x,x) + 2\Theta_W(x,x)\Theta_b(x,x) + \Theta_b^2(x,x)$. Then we need to consider the expectations of all the components of this sum:
\begin{equation*}
    \begin{split}
        \dfrac{\mathbb{E}[\Theta_W^2(x,x)]}{L^2/\alpha_0^2} & \xrightarrow[\substack{M\to\infty, L\to\infty, \\L/M \to \lambda\in\mathbb{R}}]{} e^{5\lambda}\Bigl(\dfrac{1}{2\lambda} - \dfrac{1}{8\lambda^2}\Bigr) + e^\lambda\dfrac{1}{8\lambda^2},\\
        \dfrac{\mathbb{E}[\Theta_b^2(x,x)]}{L^2/\alpha_0^2} & \xrightarrow[\substack{M\to\infty, L\to\infty, \\L/M \to \lambda\in\mathbb{R}}]{} \alpha_0^2 \Bigl( \dfrac{2}{25\lambda^2}e^{5\lambda} - \dfrac{2}{25\lambda^2} - \dfrac{2}{5\lambda}\Bigr),\\
        \dfrac{\mathbb{E}[\Theta_W(x,x)\Theta_b(x,x)]}{L^2/\alpha_0^2} & \xrightarrow[\substack{M\to\infty, L\to\infty, \\L/M \to \lambda\in\mathbb{R}}]{} \dfrac{\alpha_0}{4} \Bigl( \dfrac{9}{25\lambda^2}e^{5\lambda} - \dfrac{1}{\lambda^2}e^\lambda - \dfrac{4}{5\lambda} + \dfrac{16}{25\lambda^2} \Bigr).\\
    \end{split}
\end{equation*}
Putting the above expressions together, we get the following limit for the desired ratio:
\begin{equation*}
\begin{split}
    \dfrac{\mathbb{E}[\Theta^2(x,x)]}{\mathbb{E}^2[\Theta(x,x)]} \xrightarrow[\substack{M\to\infty, L\to\infty, \\L/M \to \lambda\in\mathbb{R}}]{} \dfrac{1}{(1+\alpha_0)^2 \lambda} &\Biggl[ e^{5\lambda}\Bigl(\dfrac{1}{2} + \dfrac{16\alpha_0^2+36\alpha_0-25}{200\lambda} \Bigr) \\
    &+ e^\lambda \dfrac{1-4\alpha_0}{8\lambda} + \dfrac{2\alpha_0(4-\alpha_0)}{25\lambda} - \dfrac{2\alpha_0(1+\alpha_0)}{5}\Biggr],
\end{split}
\end{equation*}
which completes the proof.
\end{proof}

\subsection{Non-diagonal elements of the NTK}\label{appendix:non-diag}
\begin{theorem}[Non-diagonal elements of the NTK]
Consider a fully-connected DNN  of depth $L$ defined in \eqref{eq:nn} initialized as in \eqref{eq:init}. The input dimension is given by $n_0=\alpha_0 M$, the output dimension is $1$, and the hidden layers have constant width $M$.  The activation function in the hidden layers is ReLU, i.e. $\phi(x)= x \mathbbm{1}\{x>0\}$, and the output layer is linear. Assume also that the biases are initialized to zero, i.e. $\sigma_b=0$, and the input data is normalized. Then for the ratio of non-diagonal and diagonal elements of the NTK we have:
\begin{equation*}
    1\geq \lim_{\substack{L\to\infty,M\to\infty\\L/M\to\lambda\in\mathbb{R}}} \dfrac{\mathbb{E}[\Theta(x,\Tilde{x})]}{\mathbb{E}[\Theta(x,x)]} \geq \dfrac{1}{4}
\end{equation*}
Moreover, for the dispersion of the non-diagonal elements we have:
\begin{equation*}
     \lim_{\substack{L\to\infty,M\to\infty\\L/M\to\lambda\in\mathbb{R}}}\dfrac{\mathbb{E}[\Theta^2(x,\Tilde{x})]}{\mathbb{E}^2[\Theta(x,\Tilde{x})]} \leq 16\lim_{\substack{L\to\infty,M\to\infty\\L/M\to\lambda\in\mathbb{R}}} \dfrac{\mathbb{E}[\Theta^2(x,x)]}{\mathbb{E}^2[\Theta(x,x)]}
\end{equation*}
\end{theorem}
\begin{proof}
The non-diagonal element of the NTK on point $x$ and $\Tilde{x}$ is given by
\begin{equation*}
    \Theta(x,\Tilde{x}) = \sum_{\ell=1}^L \langle \g^\ell,\Tilde{\g}^\ell \rangle \langle \x^{\ell-1},\Tilde{\x}^{\ell-1}\rangle + \sum_{\ell=1}^L \langle \g^\ell,\Tilde{\g}^\ell \rangle,
\end{equation*}
where the activations and the backpropagated errors with tilde correspond to $\Tilde{x}$. Same as in Lemma \ref{lemma:x_ratio}, we can write the following for the involved dot products:
\begin{equation*}
    \langle \x^{\ell},\Tilde{\x}^{\ell} \rangle = \dfrac{\sigma_w^2}{n_{\ell-1}} \|\x^{\ell-1}\| \|\Tilde{\x}^{\ell-1}\| \sum_{i=1}^{n_\ell} \phi(\mathcal{U}_i^\ell)\phi(\Tilde{\mathcal{U}}_i^\ell)
\end{equation*}
We notice that in this case $\mathcal{U}_i^\ell \sim \mathcal{N}(0,1)$ and $\Tilde{\mathcal{U}}_i^\ell\sim \mathcal{N}(0,1)$ are correlated variables and the covariance is given by $\rho_x^{\ell-1} := \frac{\langle \x^{\ell-1},\Tilde{\x}^{\ell-1} \rangle}{\|\x^{\ell-1}\| \|\Tilde{\x}^{\ell-1}\|}$. The distribution of $\mathcal{U}_i^\ell$ and $\Tilde{\mathcal{U}}_i^\ell$ depends only on the angle between the activations and not on the norms.

Assuming $\rho_x^{\ell-1}$ is given, we can calculate the expectation of $\phi(\mathcal{U}_i^\ell)\phi(\Tilde{\mathcal{U}}_i^\ell)$:
\begin{equation*}
    \mathbb{E}[\phi(\mathcal{U}_i^\ell)\phi(\Tilde{\mathcal{U}}_i^\ell) \mid \rho_x^{\ell-1} ] = \dfrac{1}{2\pi}\bigl(\sqrt{1-(\rho_x^{\ell-1})^2} + \rho_x^{\ell-1}\pi/2 + \rho_x^{\ell-1} \arcsin \rho_x^{\ell-1} \bigr)
\end{equation*}
Then, denoting $g(x) := \frac{1}{\pi}\bigl(\sqrt{1-x^2} + x\pi/2 + x \arcsin x\bigr)$, we have 
\begin{equation*}
    \mathbb{E}\Bigl[\dfrac{\langle \x^{\ell},\Tilde{\x}^{\ell} \rangle}{\langle \x^{\ell-1},\Tilde{\x}^{\ell-1}\rangle } \Bigr] = \dfrac{\sigma_w^2}{2}\dfrac{n_\ell}{n_{\ell-1}} \mathbb{E}\Bigl[ \dfrac{g(\rho_x^{\ell-1})}{\rho_x^{\ell-1}}\Bigr]
\end{equation*}
We can reason in the same way to find expected dot products of the backpropagated errors:
\begin{equation*}
    \langle \g^\ell,\Tilde{\g}^\ell\rangle = \dfrac{\sigma_w^2}{n_\ell} \|\g^{\ell+1}\| \|\Tilde{\g}^{\ell+1}\|\sum_{i=1}^{n_\ell}\phi'(\mathcal{U}_i^\ell)\phi'(\Tilde{\mathcal{U}}_i^\ell) \mathcal{V}_i^{\ell+1}\Tilde{\mathcal{V}}_i^{\ell+1}
\end{equation*} 
We can also calculate the involved expectations:
\begin{equation*}
    \begin{split}
        \mathbb{E}[\phi'(\mathcal{U}_i^\ell)\phi'(\Tilde{\mathcal{U}}_i^\ell)\mid \rho_x^{\ell-1}] &= \dfrac{1}{2\pi} \Bigl( \dfrac{\pi}{2} + \arcsin \rho_x^{\ell-1} \Bigr),\\
        \mathbb{E}[\mathcal{V}_i^{\ell}\Tilde{\mathcal{V}}_i^{\ell} \mid \rho_\delta^{\ell} ] &= \rho_\delta^{\ell} := \dfrac{\langle \g^\ell,\Tilde{\g}^\ell\rangle}{\|\g^{\ell}\| \|\Tilde{\g}^{\ell}\|}
    \end{split}
\end{equation*}
And, using the above expressions, we get 
\begin{equation*}
    \mathbb{E}\Bigl[ \dfrac{\langle \g^\ell,\Tilde{\g}^\ell\rangle}{\langle \g^{\ell+1},\Tilde{\g}^{\ell+1}\rangle}\Bigr] = \dfrac{\sigma_w^2}{2}\mathbb{E}\Bigl[\dfrac{1}{\pi}\Bigl( \dfrac{\pi}{2} + \arcsin \rho_x^{\ell-1} \Bigr)\Bigr]
\end{equation*}

We also need to consider the expectation of $\rho_x^{\ell}$:
\begin{equation*}
    \mathbb{E}[\rho_x^{\ell} \mid \rho_x^{\ell-1}] = \mathbb{E}\Biggl[ \dfrac{\sum_i \phi(\mathcal{U}_i^\ell)\phi(\Tilde{\mathcal{U}}_i^\ell)}{\sqrt{\sum_i \phi^2(\mathcal{U}_i^\ell)}\sqrt{\sum_i \phi^2(\Tilde{\mathcal{U}}_i^\ell)}} \mid \rho_x^{\ell-1} \Biggr] \xrightarrow[n_{\ell-1}\to\infty]{} g(\rho_x^{\ell-1}),
\end{equation*}
where the correction to the above expectation for finite width is of order $O(1/M)$ since the components approach normality with this rate. Moreover, the estimator of correlation coefficient has a negative bias, therefore $\mathbb{E}[\rho_x^{\ell} \mid \rho_x^{\ell-1}]$ approaches $g(\rho_x^{\ell-1})$ from below with $n_{\ell-1}\to\infty$. Then we have
\begin{equation*}
\begin{split}
    \mathbb{E}\Bigl[\langle \x^{\ell},\Tilde{\x}^{\ell} \rangle \Bigr] = \langle \x^{0},\Tilde{\x}^{0} \rangle  \mathbb{E}\Bigl[ \prod_{k=1}^\ell \dfrac{\langle \x^{k},\Tilde{\x}^{k} \rangle}{\langle \x^{k-1},\Tilde{\x}^{k-1}\rangle } \Bigr] &= \langle \x^{0},\Tilde{\x}^{0} \rangle a^{\ell} \dfrac{n_\ell}{n_{0}} \mathbb{E}\Bigl[ \dfrac{g(\rho_x^{\ell-1})}{\rho_x^0}\prod_{k=0}^{\ell-2} \dfrac{g(\rho^{k})}{\rho^{k+1}}  \Bigr] \geq \|\x^{0}\| \|\Tilde{\x}^{0}\|a^{\ell} \dfrac{n_\ell}{n_{0}} \mathbb{E}\bigl[\rho_x^{\ell}\bigr] 
\end{split}
\end{equation*}
Similarly, denoting $f(x) := \frac{1}{\pi}(\pi/2+\arcsin x)$, we get the following for the products of backpropagated errors:
\begin{equation*}
    \mathbb{E}[\langle \g^\ell,\Tilde{\g}^\ell\rangle ] \geq a^{L-\ell} \prod_{k=\ell}^{L-1}f\bigl(\mathbb{E}[\rho_x^{k-1}]\bigr)
\end{equation*}
Now we notice that $\mathbb{E}[\rho_x^\ell]\to g^{\circ \ell}(\rho^0_x)$ not only if $M\to\infty$ but also if $\ell\to\infty$ with finite $M$, where $g^{\circ k}$ denotes composition of the function $k$ times. Indeed $g(x)$ is a monotonically increasing function with $g(x)\geq x$ and a single fixed point at $x=1$, so we have $\mathbb{E}[\rho_x^{\ell}]\to1$ and $g^{\circ\ell}(\rho^0_x)\to 1$ if $\ell\to\infty$. In other words, if $\mathbb{E}[\rho_x^\ell]/g^{\circ\ell}(\rho^0_x) = 1+c_\ell/M$ for some coefficients $c_\ell$, then $c_\ell\to 0$ as $\ell\to\infty$. Therefore, we can replace $\mathbb{E}[\rho_x^\ell]$ with $g^{\circ \ell}(\rho^0_x)$ in the above bounds to obtain the infinite-depth-and-width limit.

Putting everything together, we can write the following bound for the expectation of a non-diagonal element of the NTK
\begin{equation*}
\begin{split}
    \lim_{\substack{L\to\infty,M\to\infty\\L/M\to\lambda\in\mathbb{R}}}[\Theta(x,\Tilde{x})] \geq \lim_{\substack{L\to\infty,M\to\infty\\L/M\to\lambda\in\mathbb{R}}} \Biggr[\|\x^0\| \|\Tilde{\x}^0\|a^{L-1}\sum_{\ell=1}^L\dfrac{n_{\ell-1}}{n_0}g^{\circ{\ell-1}}(\rho^0_x) \prod_{k=\ell}^{L-1}f\bigl(g^{\circ(k-1)}(\rho_x^0)\bigr)+\\  \sum_{\ell=1}^La^{L-\ell}\prod_{k=\ell}^{L-1}f\bigl(g^{\circ(k-1)}(\rho_x^0)\bigr)\Biggr],
\end{split}
\end{equation*}
Now studying the expressions above we can find the following bounds: 
\begin{equation*}
    1\geq \lim_{\substack{L\to\infty,M\to\infty\\L/M\to\lambda\in\mathbb{R}}}\dfrac{\mathbb{E}[\Theta(x,\Tilde{x})]}{\mathbb{E}[\Theta(x,x)]} \geq \dfrac{1}{4},
\end{equation*}
The upper bound is trivial. We obtain the lower bound in case of initialization in the chaotic phase by noticing that $\sum_{\ell=1}^L g^{\circ{\ell-1}}(\rho^0_x)\prod_{k=\ell}^{L-1}f\bigl(g^{\circ(k-1)}(\rho_x^0)\bigr) \geq L/4$ for $L\geq 2$, which, by Chebyshev's sum inequality, gives the maximal ratio between diagonal and non-diagonal elements of the NTK, since $\mathbb{E}[\Theta_W(x,x)] = a^{L-1}\sum_{\ell=1}^L \dfrac{n_{\ell-1}}{n_0}$. In the ordered phase, we have $\sum_{\ell=1}^L \prod_{k=\ell}^{L-1}f\bigl(g^{\circ(k-1)}(\rho_x^0)\bigr) \geq L/4$ and $\mathbb{E}[\Theta_b(x,x)] = \sum_{\ell=1}^L a^{L-\ell}$, which gives the same bound.

Moreover, it is easy to see that $\mathbb{E}[\Theta^2(x,\Tilde{x})]\leq \mathbb{E}[\Theta^2(x,x)]$, therefore we can write
\begin{equation*}
   \lim_{\substack{L\to\infty,M\to\infty\\L/M\to\lambda\in\mathbb{R}}}\dfrac{\mathbb{E}[\Theta^2(x,\Tilde{x})]}{\mathbb{E}^2[\Theta(x,\Tilde{x})]} \leq 16 \lim_{\substack{L\to\infty,M\to\infty\\L/M\to\lambda\in\mathbb{R}}}\dfrac{\mathbb{E}[\Theta^2(x,x)]}{\mathbb{E}^2[\Theta(x,x)]}
\end{equation*}
\end{proof}

\subsection{Training dynamics of the NTK}\label{appendix:train}
\begin{theorem}[GD step of the NTK]
Consider a fully-connected DNN  of depth $L$ defined in \eqref{eq:nn} initialized as in \eqref{eq:init}. The input dimension is given by $n_0=\alpha_0 M$, the output dimension is $1$, and the hidden layers have constant width $M$.  The activation function in the hidden layers is ReLU, i.e. $\phi(x)= x \mathbbm{1}\{x>0\}$, and the output layer is linear. Assume also that the biases are initialized to zero, i.e. $\sigma_b=0$, and the input data is normalized. Then, if we perform a GD step on a point $(x,y)\in\mathcal{D}$ with learning rate $\eta$, the following holds for the changes of the corresponding element of the NTK:
\begin{enumerate}
    \item In the \textbf{chaotic phase} ($a:={\sigma_w^2}/{2} > 1$), the changes to the NTK value are infinite in the limit for a constant learning rate:
     \begin{equation}
            \dfrac{\mathbb{E}[\Delta{\Theta}(x,x)] }{\mathbb{E}[{\Theta}(x,x)] }  \xrightarrow[\substack{M\to\infty, L\to\infty, \\L/M \to \lambda\in\mathbb{R}}]{} \infty
        \end{equation}
    The scaling of the learning rate needed to avoid the infinite limit is given by $\eta=O(a^{-L})$, which tends to zero with depth. 
    \item In the \textbf{ordered phase} ($a < 1$), the NTK stays constant in the limit:
       \begin{equation}
            \dfrac{\mathbb{E}[\Delta{\Theta}(x,x)] }{\mathbb{E}[{\Theta}(x,x)] }  \xrightarrow[\substack{M\to\infty, L\to\infty, \\L/M \to \lambda\in\mathbb{R}}]{} 0
        \end{equation}
\end{enumerate}
\end{theorem}

\begin{proof}
A derivative of the NTK in gradient flow can be expanded as follows:
\begin{equation*}
    \dot{\Theta}(x,x) = \sum_{\ell=1}^L \Bigl(\sum_{i,j}\dfrac{\partial \Theta(x,x)}{\partial \W_{ij}^\ell} \dot{\W}_{ij}^\ell  + \sum_i \dfrac{\partial \Theta(x,x)}{\partial \bias_i^\ell} \dot{\bias}_i^\ell\Bigr),
\end{equation*}
where the parameters change in the direction of the negative gradient:
\begin{equation*}
\begin{split}
    \dot{\W}_{ij}^\ell = - \dfrac{\partial \mathcal{L}(\mathcal{D})}{\partial \W_{ij}^\ell}
    , \quad \dot{\bias}_{i}^\ell = - \dfrac{\partial \mathcal{L}(\mathcal{D})}{\partial \bias_{i}^\ell}, \quad i=1,\dots n_\ell,\  j=1,\dots,n_{\ell-1}, \ \ell=1,\dots L
\end{split}
\end{equation*}
If we now assume that the gradient descent step is performed on a single point of the dataset $x$, which is the same point for which the NTK is calculated, we have:
\begin{equation*}
\begin{split}
    \dot{\W}_{ij}^\ell &= - \dfrac{\partial \mathcal{L}(x)}{\partial \W_{ij}^\ell} = - \dfrac{\partial \mathcal{L}(x)}{\partial f(x)}\dfrac{\partial f(x)}{\partial \W_{ij}^\ell} =  -\dfrac{\partial \mathcal{L}(x)}{\partial f(x)}\g_i^\ell \x_j^{\ell-1}, \\
    \dot{\bias}_{i}^\ell &= -\dfrac{\partial \mathcal{L}(x)}{\partial \bias_{i}^\ell} = - \dfrac{\partial \mathcal{L}(x}{\partial f(x)}\dfrac{\partial f(x)}{\partial \bias_{i}^\ell}  = - \dfrac{\partial \mathcal{L}(x)}{\partial f(x)}\g_i^\ell
\end{split}
\end{equation*}
It remains to calculate the derivatives of the NTK with resepect to the parameters. The involved terms are:
\begin{equation*}
    \dfrac{\partial \Theta_W(x,x)}{\partial \W_{ij}^\ell} =  \sum_{\ell'}\sum_{i',j'}\dfrac{\partial}{\partial \W_{ij}^\ell}\Bigl(\dfrac{\partial f(x)}{\partial \W_{i'j'}^{\ell'}} \Bigr)^2 = 2 \sum_{\ell'}\sum_{i',j'}\dfrac{\partial f(x)}{\partial \W_{i'j'}^{\ell'}} \dfrac{\partial^2 f(x)}{\partial \W_{ij}^\ell\W_{i'j'}^{\ell'}},
\end{equation*}

\begin{equation*}
    \dfrac{\partial \Theta_W(x,x)}{\partial \bias_{k}^\ell} =  \sum_{\ell'}\sum_{i',j'}\dfrac{\partial}{\partial \bias_{k}^\ell}\Bigl(\dfrac{\partial f(x)}{\partial \W_{i'j'}^{\ell'}} \Bigr)^2 = 2 \sum_{\ell'}\sum_{i',j'}\dfrac{\partial f(x)}{\partial \W_{i'j'}^{\ell'}} \dfrac{\partial^2 f(x)}{\partial \bias_k^\ell\W_{i'j'}^{\ell'}},
\end{equation*}

\begin{equation*}
    \dfrac{\partial \Theta_b(x,x)}{\partial \bias_{k}^\ell} =  \sum_{\ell'}\sum_{i'}\dfrac{\partial}{\partial \bias_{k}^\ell}\Bigl(\dfrac{\partial f(x)}{\partial \bias_{i'}^{\ell'}} \Bigr)^2 = 2 \sum_{\ell'}\sum_{i'}\dfrac{\partial f(x)}{\partial \bias_{i'}^{\ell'}} \dfrac{\partial^2 f(x)}{\partial \bias_k^\ell\bias_{i'}^{\ell'}},
\end{equation*}

\begin{equation*}
    \dfrac{\partial \Theta_b(x,x)}{\partial \W_{ij}^\ell} =  \sum_{\ell'}\sum_{i',j'}\dfrac{\partial}{\partial \W_{ij}^\ell}\Bigl(\dfrac{\partial f(x)}{\partial \bias_{i'}^{\ell'}} \Bigr)^2 = 2 \sum_{\ell'}\sum_{i'}\dfrac{\partial f(x)}{\partial \bias_{i'}^{\ell'}} \dfrac{\partial^2 f(x)}{\partial \W_{ij}^\ell\bias_{i'}^{\ell'}}
\end{equation*}
To calculate these terms, we need to find the second derivatives of the DNN's output function. 

\begin{equation*}
    \dfrac{\partial^2 f(x)}{\partial \W_{ij}^\ell\W_{i'j'}^{\ell'}} = \g_{i'}^{\ell'}\dfrac{\partial \x_{j'}^{\ell'-1}}{\partial \W_{ij}^\ell} + \x_{j'}^{\ell'-1}\dfrac{\partial \g_{i'}^{\ell'}}{\partial \W_{ij}^\ell} = \mathbbm{1}_{\ell<\ell'}\g_{i'}^{\ell'}\x_j^{\ell-1}\dfrac{\partial \x_{j'}^{\ell'-1}}{\partial \h_i^\ell} + \mathbbm{1}_{\ell>\ell'}\g_i^\ell\x_{j'}^{\ell'-1}\phi'(\h_j^{\ell-1})\dfrac{\partial \g_{i'}^{\ell'}}{\partial \g_j^{\ell-1}}, 
\end{equation*}
In the above equation the first term is non-zero only in case $\ell'>\ell$ and the second term is non-zero only if $\ell'<\ell$. Then we can write the following:
\begin{equation*}
\begin{split}
    \sum_\ell\sum_{i,j}\dfrac{\partial \Theta_W(x,x)}{\partial \W_{ij}^\ell}\dot{\W}_{ij}^\ell = &-\dfrac{\partial \mathcal{L}(x)}{\partial f(x)} \sum_{\ell'>\ell}\|\g^{\ell'}\|^2 \|\x^{\ell-1}\|^2 \sum_i \g_i^\ell \dfrac{\partial \|\x^{\ell'-1}\|^2}{\partial \h_i^\ell}\\
     &-\dfrac{\partial \mathcal{L}(x)}{\partial f(x)} \sum_{\ell>\ell'}\|\g^{\ell}\|^2 \|\x^{\ell'-1}\|^2 \sum_j \x_j^{\ell-1} \dfrac{\partial \|\g^{\ell'}\|^2}{\partial \g_j^{\ell-1}}
\end{split}
\end{equation*}
Opening the remaining parts of the derivative in the same way, we obtain the following expression:
\begin{equation*}
\begin{split}
    \dot{\Theta}(x,x) = -\dfrac{\partial \mathcal{L}(x)}{\partial f(x)} \Biggl( &\sum_{\ell'>\ell}\bigl(\|\g^{\ell'}\|^2 \|\x^{\ell-1}\|^2 + \|\g^{\ell'}\|^2\bigr)\sum_i \g_i^\ell \dfrac{\partial \|\x^{\ell'-1}\|^2}{\partial \h_i^\ell}\\
      + &\sum_{\ell>\ell'}\bigl(\|\g^{\ell}\|^2 \|\x^{\ell'-1}\|^2 + \|\g^{\ell}\|^2\bigr)\sum_j \x_j^{\ell-1} \dfrac{\partial \|\g^{\ell'}\|^2}{\partial \g_j^{\ell-1}}\Biggr)
\end{split}
\end{equation*}

\textbf{Case 1. Chaotic phase.} Let us bound the change of the NTK by computing only the terms with $\ell'=\ell+1$. In this case, $\frac{\partial \|\x^{\ell'-1}\|^2}{\partial \h_i^\ell}=2\x_i^\ell$. We then notice that $\sum_{i} \x_i^\ell\g_i^\ell = \sum_k \g^{\ell+1}_k \sum_i \W_{ki}^{\ell+1}\x_i^\ell = \sum_k  \g^{\ell+1}_k \h_k^{\ell+1} = \sum_k \g^{\ell+1}_k \x_k^{\ell+1}$, which by induction gives $\sum_{i} \x_i^\ell\g_i^\ell = f(x)$. Therefore, taking into account that for quadratic loss we have $\partial \mathcal{L}(x)/\partial f(x) = f(x) - y$, we can write the following bound:
\begin{equation*}
    \mathbb{E}[|\Delta{\Theta}(x,x)|] \geq 2\eta \mathbb{E}\Bigl[f(x)^2 \sum_{\ell=1}^L \bigl(\|\g^{\ell+1}\|^2 \|\x^{\ell-1}\|^2 + \|\g^{\ell+1}\|^2\bigr)\Bigr]
\end{equation*}
Then, using the results of Lemmas \ref{lemma:x_ratio}, \ref{lemma:d_ratio} and \ref{lemma:gia} again, we obtain the expectation of the first part:
\begin{equation*}
    \mathbb{E}[|\Delta{\Theta}(x,x)|] \geq 4\eta\dfrac{\|\x^0\|^4}{n_0} a^{2L-1}\sum_{\ell=1}^L\prod_{j=\ell+1}^{L-1}\Bigl(1+\dfrac{1}{n_j} + O\Bigl(\dfrac{1}{M^{3/2}}\Bigr)\Bigr)\prod_{i=1}^{\ell-1}\Bigl(1+\dfrac{5}{n_j}\Bigr) \propto a^{2L-1}
\end{equation*}
And since in case of the chaotic phase $\mathbb{E}[\Theta(x,x)]\propto\mathbb{E}[\Theta_W(x,x)]\propto a^L LM/n_0$, we have the desired limit:

\begin{equation*}
    \dfrac{\mathbb{E}[|\Delta{\Theta}(x,x)|] }{\mathbb{E}[{\Theta}(x,x)] } \xrightarrow[\substack{M\to\infty, L\to\infty, \\L/M \to \lambda\in\mathbb{R}}]{} \infty
\end{equation*}

\textbf{Case 2. Ordered phase.} The bound that we used for the chaotic phase above gives zero in the limit in case of the ordered phase. We will now show that the upper bound of the relative change of the NTK is also zero in the limit in this case. We notice that $\sum_i f(x)\g_i^\ell \frac{\partial\|\x^{\ell'-1}\|^2}{\partial \h_i^\ell} = \sum_i \g_i^\ell \sum_k \g_k^\ell \frac{\partial\|\x^{\ell'-1}\|^2}{\partial \h_i^\ell} \h_k^\ell = \sum_i (\g_i^\ell)^2 \frac{\partial\|\x^{\ell'-1}\|^2}{\partial \h_i^\ell} \h_i^\ell + \sum_{i\neq k} \g_i^\ell \g_k^\ell\frac{\partial\|\x^{\ell'-1}\|^2}{\partial \h_i^\ell} \h_k^\ell$. Then we have $\sum_i f(x)\g_i^\ell \frac{\partial\|\x^{\ell'-1}\|^2}{\partial \h_i^\ell} \leq \|\g^\ell\|^2\|\x^{\ell'-1}\|^2 + A$, where the expectation of $A$ is zero. Similarly, we have $\sum_j f(x) \x_j^{\ell-1}\frac{\partial \|\g^{\ell'}\|^2}{\partial \g_j^{\ell-1}} \leq \|\g^{\ell'}\|^2\|\x^{\ell-1}\|^2 + B$ with a term $B$ of zero expectation. The we have the following bound for the change of the NTK:
\begin{equation*}
    \mathbb{E}[|\Delta{\Theta}(x,x)|] \leq 2\eta \mathbb{E}\Bigl[\sum_{\ell_1} \|\g^{\ell_1}\|^2 \|\x^{\ell_1-1}\|^2 \sum_{\ell_2<\ell_1}\|\g^{\ell_2}\|^2\bigl(\|\x^{\ell_2}\|^2 +1 \bigr)\Bigr]
\end{equation*}
The expectation of $\sum_{\ell_2<\ell_1}\theta_W^{\ell_1}\theta_W^{\ell_2}$, where $\theta_W^\ell:=\|\g^{\ell}\|^2 \|\x^{\ell-1}\|^2$, was calculated in Lemma \ref{lemma:theta_w} and the expectation of $\sum_{\ell_2<\ell_2}\theta_W^{\ell_1}\theta_b^{\ell_2}$, where $\theta_b^\ell:=\|\g^{\ell}\|^2$, was calculated in Lemma \ref{lemma:theta_wb}. In particular, we have the following results for the two sums:
\begin{equation*}
\begin{split}
\mathbb{E}[\sum_{\ell_2<\ell_1}\theta_W^{\ell_1}\theta_W^{\ell_2}]&\propto a^{2L}\dfrac{L^2}{\alpha_0^2}\dfrac{1}{4\lambda}e^{5\lambda} \Biggl( 1 - \dfrac{1}{4\lambda}(1 - e^{-4\lambda}) \Biggr),\\
\mathbb{E}[\sum_{\ell_2<\ell_1}\theta_W^{\ell_1}\theta_b^{\ell_2}]&\propto \dfrac{a^{2L}}{a-1}\dfrac{L}{\alpha_0}\dfrac{1}{4\lambda}e^{5\lambda}(1-e^{-4\lambda}).
\end{split}
\end{equation*}
Then we see that the upper bound on the changes of the NTK is proportional to $a^{2L}L^2$, which tends to zero with depth in the ordered phase. Given that the expectation of the NTK in the ordered phase has a non-zero limit given by $1/(1-a)$, we can then conclude that 
\begin{equation*}
    \dfrac{\mathbb{E}[\Delta{\Theta}(x,x)] }{\mathbb{E}[{\Theta}(x,x)] }  \xrightarrow[\substack{M\to\infty, L\to\infty, \\L/M \to \lambda\in\mathbb{R}}]{} 0
\end{equation*}
in case of initialization in the ordered phase.
\end{proof}

\section{Additional observations}

\subsection{Effects of $\alpha_0:=n0/M$ at the EOC}\label{appendix:eoc_a0}
The theoretical expression for the NTK dispersion in the infinite-width limit, which we derived in Theorem \ref{th:NTK_dispersion_lim}, depends on the ratio $\alpha_0:=n_0/M$ at the EOC: 
 \begin{equation*}
    \begin{split}
        V_{\text{EOC}}:=\dfrac{\mathbb{E}[\Theta^2(x,x)]}{\mathbb{E}^2[\Theta(x,x)]} \to  \dfrac{1}{(1+\alpha_0)^2} \Biggl[ e^{5\lambda}\Bigl(\dfrac{1}{2\lambda} + \dfrac{2\alpha_0^2-8\alpha_0}{25\lambda^2}\Bigr)
        + (e^\lambda - e^{5\lambda}) \dfrac{1-4\alpha_0}{8\lambda^2}+
        \dfrac{2\alpha_0}{5\lambda}\Bigl( \dfrac{4-\alpha_0}{5\lambda} - 1-\alpha_0\Bigr)\Biggr].
    \end{split}
    \end{equation*}
Examining this expression, one can see that it tends to the limiting expression for the NTK dispersion in the chaotic phase as the ratio $\alpha_0$ decreases:
\begin{equation*}
    \begin{split}
        V_{\text{EOC}} \xrightarrow[\alpha_0 \to 0]{} \dfrac{1}{2\lambda} e^{5\lambda} \Biggl( 1 - \dfrac{1}{4\lambda}(1 - e^{-4\lambda}) \Biggr).
    \end{split}
\end{equation*}
We illustrate this effect in Figure \ref{fig:eoc_a0}. One can see that gradually decreasing the value of $\alpha_0$ moves the NTK dispersion at the EOC closer to the NTK dispersion in the chaotic phase.

\begin{figure}
    \centering
    \includegraphics[width=0.4\textwidth]{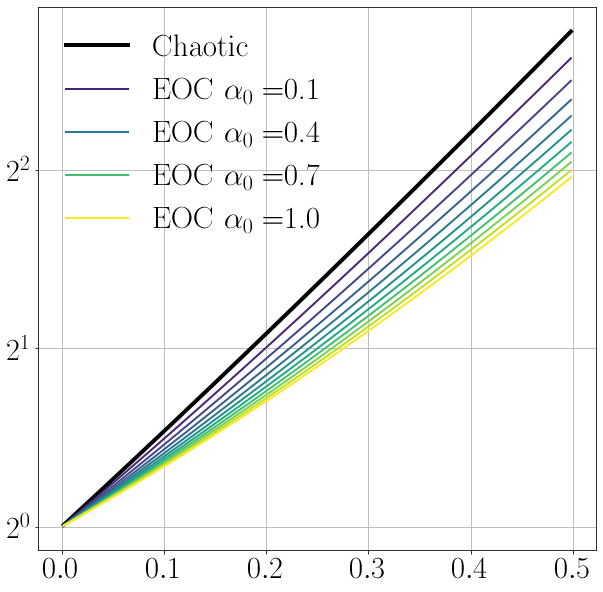}
    \caption{Effects of $\alpha_0:=n0/M$ on the NTK dispersion at the EOC in the infinite-depth-and-width limit. All the lines show the theoretical expressions from Theorem \ref{th:NTK_dispersion_lim}. The black line (uppermost) corresponds to the NTK dispersion in the chaotic phase, while all the other lines show the NTK dispersion at the EOC with varying $\alpha_0$ values. The colors spanning from yellow to violet (from lighter to darker tones) indicate the value of $\alpha_0$ spanning from $1$ (yellow) to $0.1$ (violet). }
    \label{fig:eoc_a0}
\end{figure}

\begin{figure}
    \centering
    \includegraphics[width=\textwidth]{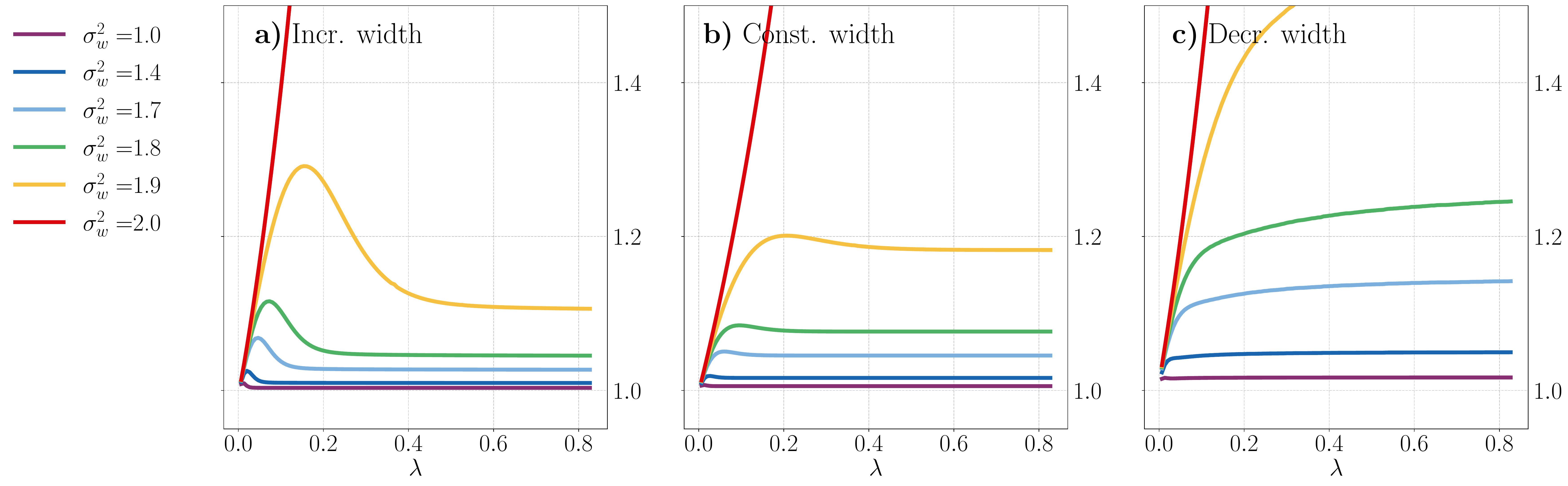}
   \includegraphics[width=\textwidth]{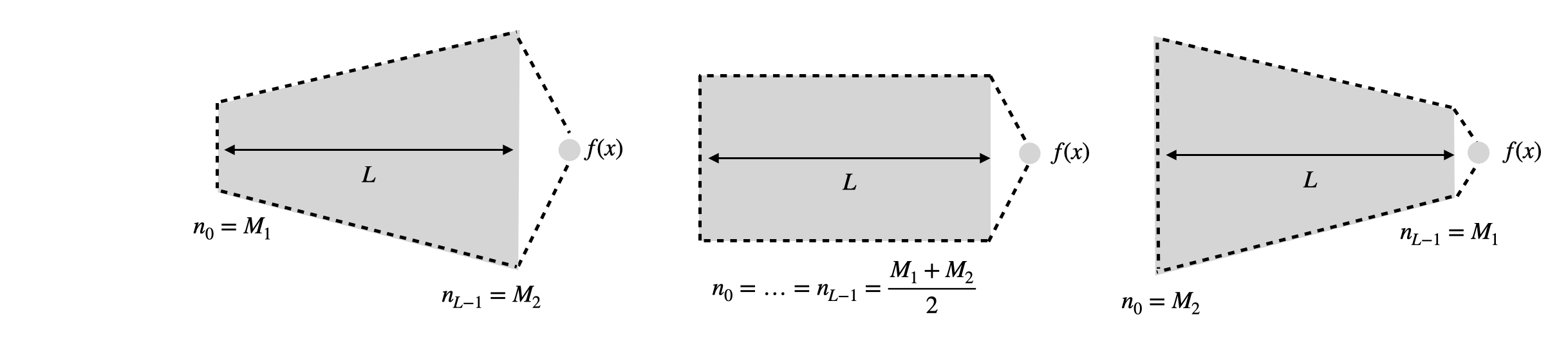}
    \caption{Effects of the architecture on the NTK dispersion ratio $\mathbb{E}[\Theta^2(x,x)]/\mathbb{E}^2[\Theta(x,x)]$ as predicted by Theorem \ref{th:NTK_moments}. The subplots show the dispersion for varying values of $\sigma_w^2$ for three different architectures. The lower row of the figure illustrates the considered architectures. Formally, the widths for each architecture are given by: \textbf{a)} $n_\ell=M_1+ \bigl\lceil {\ell(M_2-M_1)}/{L} \bigr\rceil$, \textbf{b)} $n_\ell=\bigl\lceil {(M_1+M_2)}/{2} \bigr\rceil$, \textbf{c)}~$n_\ell=M_2+ \bigl\lceil {\ell(M_1-M_2)}/{L} \bigr\rceil$ for $0\leq\ell\leq L$. The width parameters are given by $M_1=100$, $M_2=500$ and do not change as the depth grows. The depth-to-width ratio is computed for the average width, i.e. $\lambda=2L/(M_1+M_2)$.}
    \label{fig:architecture}
\end{figure}

\subsection{Effects of the architecture}\label{appendix:architecture}
In Section \ref{section:finite-width}, we showed that constant-width DNNs that increase the input dimensionality, i.e. $n_0<n_1=\dots=n_{L-1}$, get more robust with depth in a sense that the dispersion of their NTK decreases. Here we show how the theoretical expressions for the NTK moments in Theorem \ref{th:NTK_moments} reveal more effects of the DNN's architecture. In Figure \ref{fig:architecture}, we compute the theoretical prediction of the NTK dispersion for three architectures: the first one gradually increases width over $L$ layers from $n_0=M_1$ to $n_{L-1}=M_2$, the second one keeps constant width, i.e. $n_{0}=\dots=n_{L-1}=(M_1+M_2)/2$, and the third one gradually decreases width over $L$ layers from $n_0=M_2$ to $n_{L-1}=M_1$. We note that all the architectures have the same average width in this setting. We also note that we keep $M_1$ and $M_2$ fixed while varying the depth $L$. Therefore, we compare networks that increase or decrease the dimensionality equally but over a different number of layers. Figure \ref{fig:architecture} demonstrates that the NTK dispersion is lower for the DNNs that increase the dimensionality. Moreover, for such DNNs the peak of dispersion falls on the relatively shallow networks. Therefore, it may be beneficial to increase dimensionality over more layers if the goal is to decrease the variance of the DNN. On the contrary, the dispersion only increases with depth for DNNs that decrease the dimensionality. Thus, it may be beneficial to keep such networks shallow if one wants to keep the variance minimal.

\subsection{Lazy training}\label{appendix:lazy_training}
The NTK regime of neural networks is often discussed in connection with the so-called lazy training phenomenon \cite{chizat2019lazy}. In lazy training, a model behaves as its linearization around the initialial parameters due to rescaling given by:
\begin{equation*}
    \tilde{f}_\alpha(x) = \alpha f(x), \qquad \tilde{\mathcal{L}}_\alpha(\mathcal{D}) = \dfrac{1}{\alpha^2} \mathcal{L}(\mathcal{D}), \quad \alpha \in \mathbb{R},
\end{equation*}
where $f(\cdot)$ is the original model's output function and $\mathcal{L}(\mathcal{D})$ is the training loss. \citet{chizat2019lazy} showed that the dynamics of a rescaled model defined by $\tilde{f}_\alpha(\cdot)$ and $\tilde{\mathcal{L}}_\alpha(\mathcal{D})$ is close to its linearization if the scaling factor $\alpha$ is large. Thus, in this section we discuss the effects of the lazy training rescaling on the results presented in our paper. 

One can see that the NTK changes trivially if we rescale the output function: 

\begin{equation*}
    \nabla_{\w} \tilde{f}_\alpha (x) =  \alpha \nabla_{\w} f(x) \Rightarrow \tilde{\Theta}_\alpha(x_1,x_2) =  \alpha^2 \Theta(x_1,x_2),
\end{equation*}
where we denoted the NTK of the rescaled model as $\tilde{\Theta}_\alpha$. Therefore, all our results concerning the NTK dispersion (Theorem~\ref{th:NTK_dispersion_lim}) and the ratios of expectations at initialization (Theorem~\ref{th:non-diag}) do not change if we rescale the model, since the constants added to the nominator and the denominator cancel each other:
\begin{equation*}
    \dfrac{\mathbb{E}[\tilde{\Theta}^2_\alpha(x_1,x_2)]}{\mathbb{E}^2[\tilde{\Theta}_\alpha(x_1,x_2)]} =  \dfrac{\mathbb{E}[\Theta^2(x_1,x_2)]}{\mathbb{E}^2[\Theta(x_1,x_2)]}.
\end{equation*}
On the other hand, the relative change of the NTK in a gradient descent step (Theorem \ref{theorem:NTK_gd}) is affected by the rescaling. Recall that the NTK derivative is given by:
\begin{equation*}
    \dot{\Theta}(x,x) = \sum_{\ell=1}^L \Bigl(\sum_{i,j}\dfrac{\partial \Theta(x,x)}{\partial \W_{ij}^\ell} \dot{\W}_{ij}^\ell  + \sum_i \dfrac{\partial \Theta(x,x)}{\partial \bias_i^\ell} \dot{\bias}_i^\ell\Bigr).
\end{equation*}
Terms of the above expression change as follows due to the rescaling:
\begin{equation*}
\begin{split}
    \dfrac{\partial \tilde{\Theta}_\alpha(x,x)}{\partial \W_{ij}^\ell} = \alpha^2 \dfrac{\partial \Theta(x,x)}{\partial \W_{ij}^\ell}, \quad \dfrac{\partial \tilde{\Theta}_\alpha(x,x)}{\partial \bias_{i}^\ell} = \alpha^2 \dfrac{\partial \Theta(x,x)}{\partial \bias_{i}^\ell},\\
    \bigl(\tilde{\dot{\W}}_{ij}^\ell\bigr)_\alpha = - \dfrac{\partial \tilde{\mathcal{L}}_\alpha(\mathcal{D})}{\partial \W_{ij}^\ell} = -\dfrac{1}{\alpha^2}\dfrac{\partial \mathcal{L}(\mathcal{D})}{\partial \W_{ij}^\ell} =  -\dfrac{1}{\alpha^2}\tilde{\dot{\W}}_{ij}^\ell, \quad
    \bigl(\tilde{\dot{\bias}}_i^\ell\bigr)_\alpha = \dfrac{1}{\alpha^2}\dot{\bias}_i^\ell.
\end{split}
\end{equation*}
Therefore, the NTK derivative is not changed by the rescaling and for the relative change of the NTK we have:
\begin{equation*}
    \dfrac{\mathbb{E}[\Delta{\tilde{\Theta}}_\alpha(x,x)] }{\mathbb{E}[{\tilde{\Theta}_\alpha}(x,x)] }  = \dfrac{1}{\alpha^2}\dfrac{\mathbb{E}[\Delta{\Theta}(x,x)] }{\mathbb{E}[{\Theta}(x,x)] }.
\end{equation*}
For any constant choice of $\alpha\in\mathbb{R}$, this scaling does not change our limiting results in Theorem \ref{theorem:NTK_gd}. However, if the rescaling parameter is scaled exponentially with $L$ as $\alpha\propto (\sigma_w/\sqrt{2})^L$ then the relative change of the NTK tends to zero in the chaotic phase, as well as in the ordered phase and at the EOC. Thus, it is possible to enforce lazy training in deep networks outside of the ordered phase but the required rescaling parameter grows exponentially with depth.

\section{Additional experiments}

\subsection{Non-diagonal elements of the NTK}\label{appendix:non_diag_exp}
We provide empirical results on the dispersion of the non-diagonal elements of the NTK in this subsection. Figure \ref{fig:non_diag_var} is analogous to Figure \ref{fig:NTK_var_lim}: it compares the dispersion of the NTK in different phases of initialization. One can see that the non-diagonal elements of the NTK behave similarly to the diagonal ones. In particular, the dispersion of the non-diagonal elements grows exponentially with the depth-to-width ratio and reaches high values for deep networks in the chaotic phase and at the EOC, whereas in the ordered phase the dispersion is low and does not grow with depth. Figure \ref{fig:eoc_non_diag_var} is analogous to Figure \ref{fig:NTK_var_eoc}: it characterizes the behavior of the NTK dispersion close to the EOC, where the finite-width effects are significant. Here the picture is again similar to the one described in Section \ref{section:finite-width} for the diagonal elements of the NTK. In particular, the dispersion gradually increases as $\sigma_w^2$ grows and approaches the EOC. One can also see that the dispersion of the non-diagonal elements decreases with depth in the ordered phase for networks with $\alpha_0<1$, same as in the case of the diagonal elements. In all the figures, we provide experiments for varying initial angles between the two input samples of the NTK and conclude that the dispersion does not depend significantly on this angle.
\begin{figure*}
    \centering
    \includegraphics[width=\textwidth]{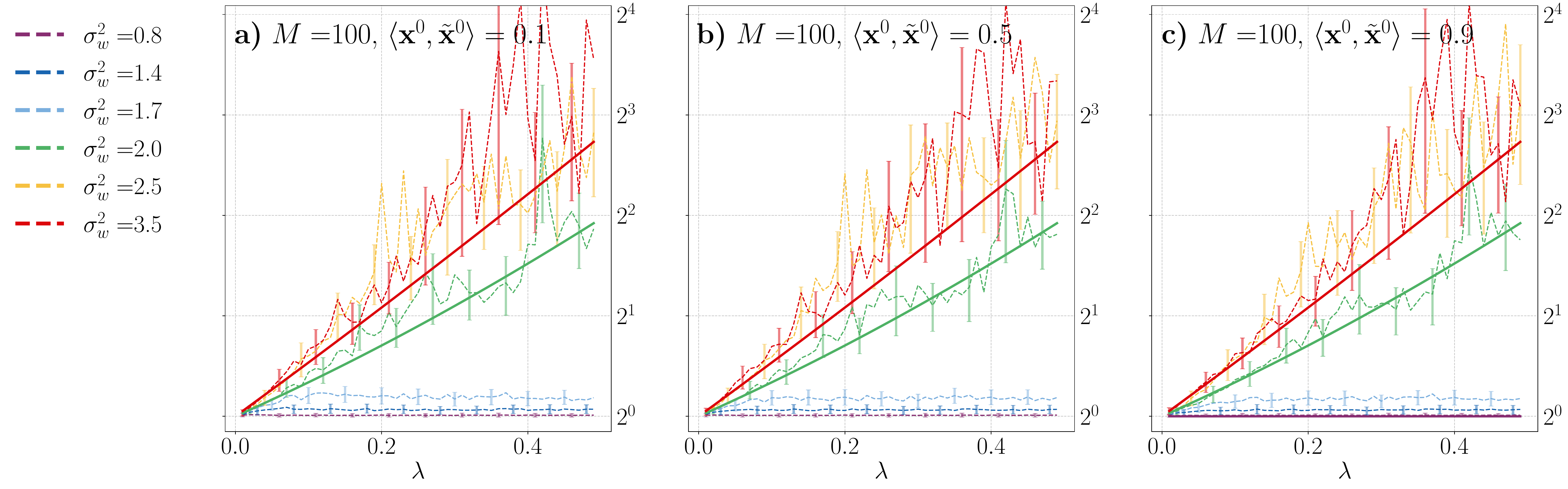}
    \includegraphics[width=\textwidth]{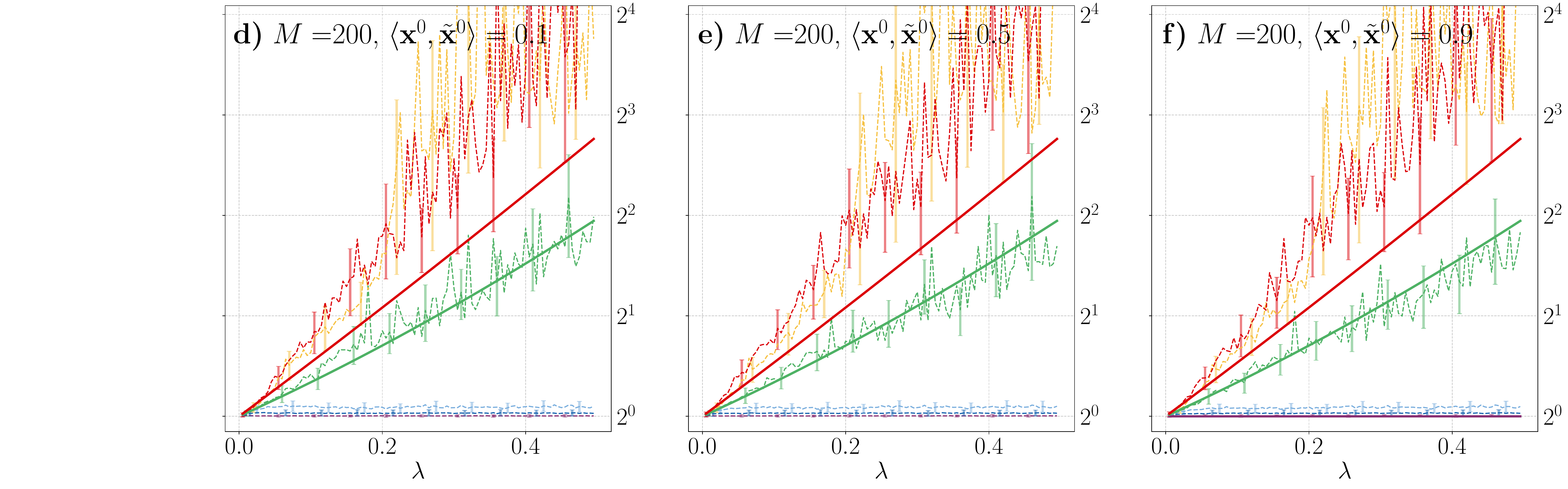}
    \includegraphics[width=\textwidth]{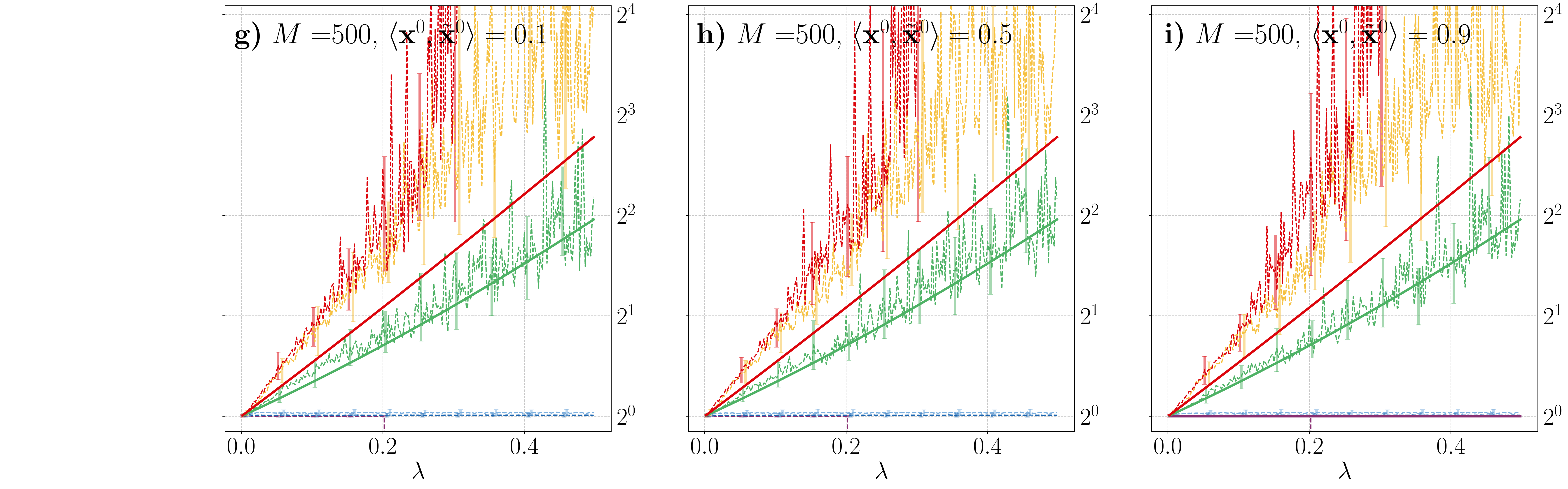}
    \caption{Ratio $\mathbb{E}[\Theta^2(x,\Tilde{x})]/\mathbb{E}^2[\Theta(x,\Tilde{x})]$ at initialization (on a pair of different input samples, i.e. $x\neq\Tilde{x}$) for fully-connected ReLU networks of constant width $M=200$ with $\alpha_0\in\{2.0,0.5,0.1\}$ and varying initial angle between the input samples $\langle \x^0,\Tilde{\x}^0\rangle~\in~\{0.1,0.5,0.9\}$. The dashed lines show the experimental results and the solid lines show the corresponding theoretical predictions for diagonal elements of the NTK from Theorem \ref{th:NTK_dispersion_lim}. For each network configuration, we sampled 500 random initializations and computed an unbiased estimator for the ratio (see details in Appendix \ref{appendix:estimator}). The error bars show the bootstrap estimation of the standard error.}
    \label{fig:non_diag_var}
\end{figure*}

\begin{figure*}
    \centering
    \includegraphics[width=\textwidth]{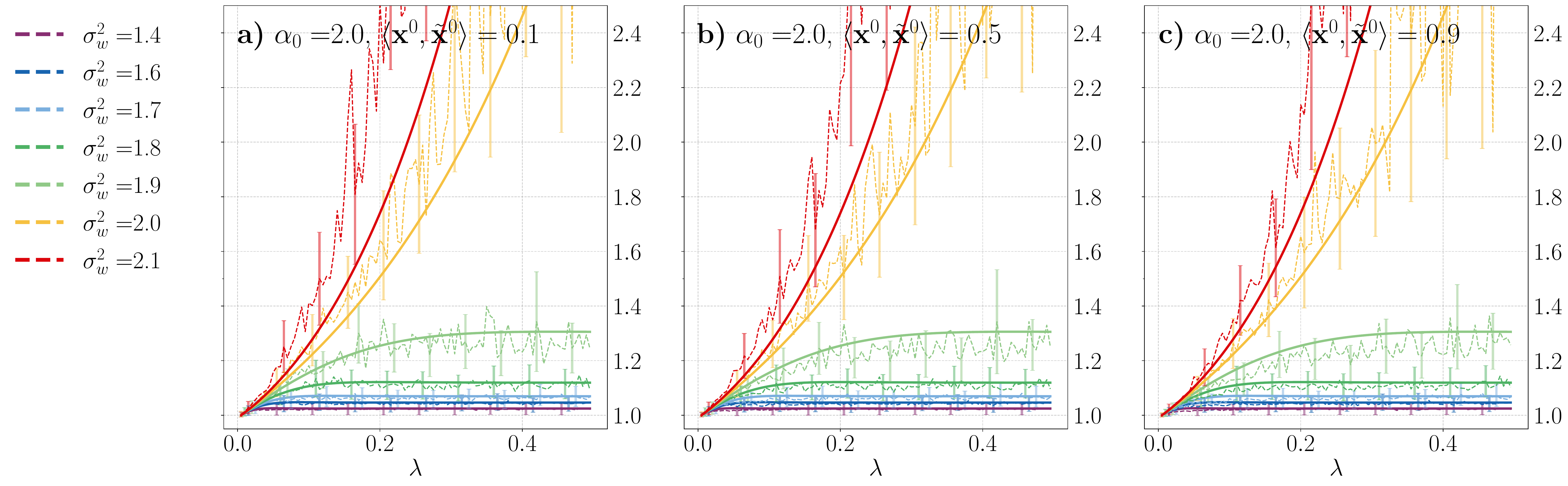}
    \includegraphics[width=\textwidth]{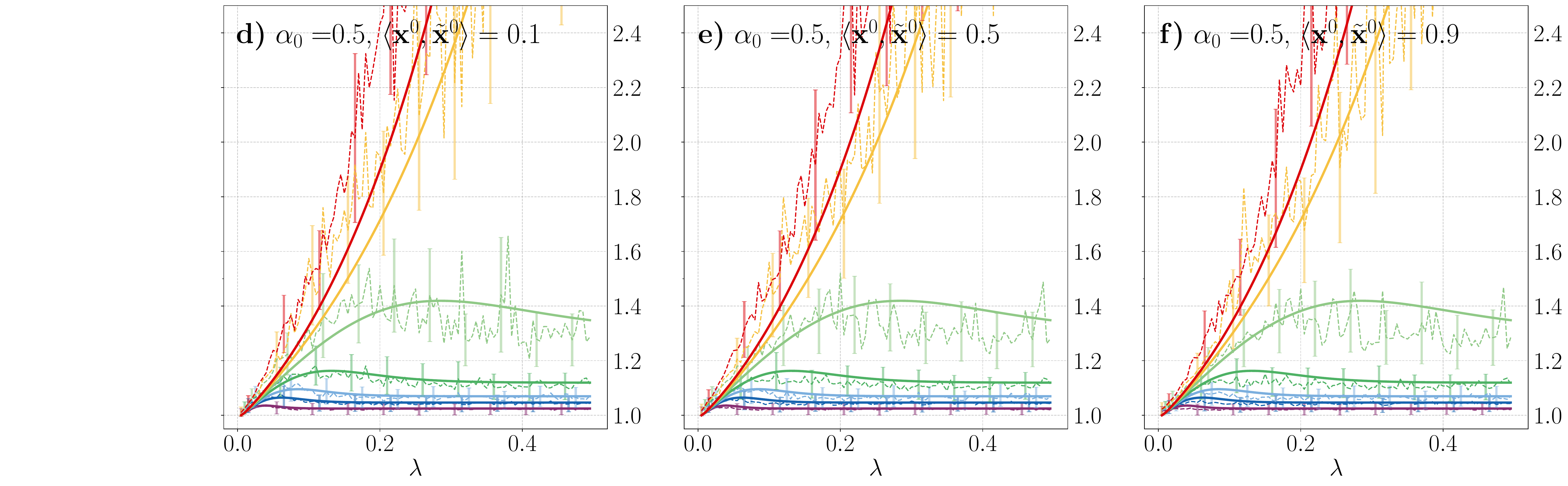}
    \includegraphics[width=\textwidth]{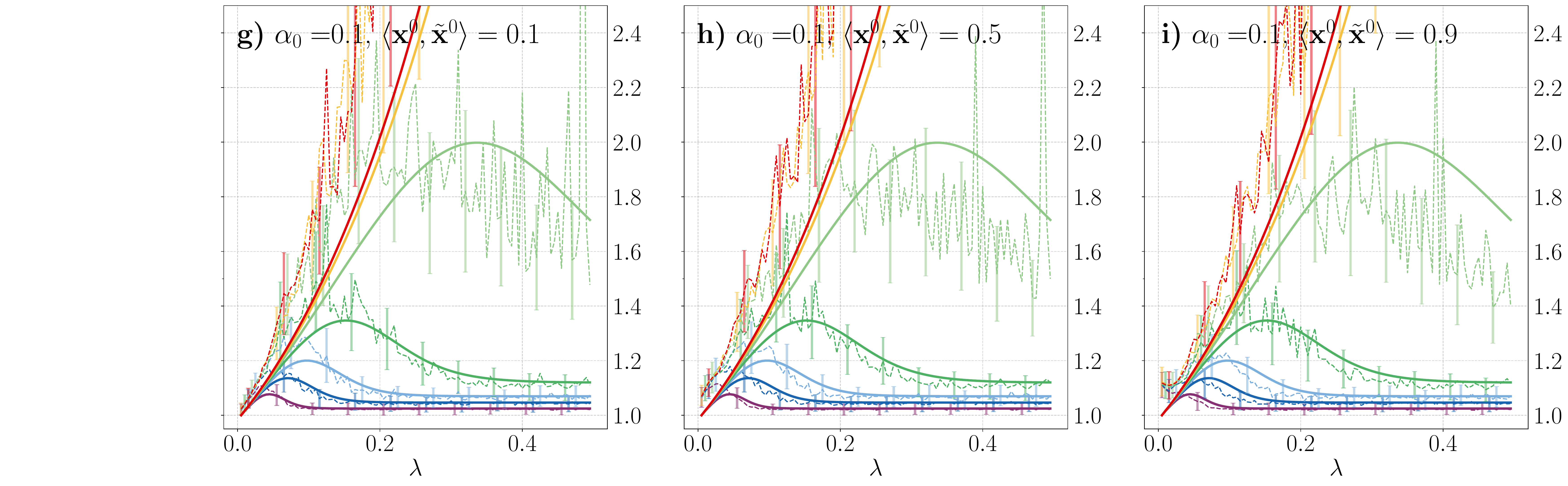}
    \caption{Ratio $\mathbb{E}[\Theta^2(x,\Tilde{x})]/\mathbb{E}^2[\Theta(x,\Tilde{x})]$ at initialization (on a pair of different input samples, i.e. $x\neq\Tilde{x}$) for fully-connected ReLU networks of constant width $M=200$ with $\alpha_0\in\{2.0,0.5,0.1\}$ and varying initial angle between the input samples $\langle \x^0,\Tilde{\x}^0\rangle~\in~\{0.1,0.5,0.9\}$. The dashed lines show the experimental results and the solid lines show the theoretical predictions given by Theorem \ref{th:NTK_moments} for the diagonal elements of the NTK. For each network configuration, we sampled 500 random initializations and computed an unbiased estimator for the ratio (see details in Appendix \ref{appendix:estimator}). The error bars show the bootstrap estimation of the standard error.}
    \label{fig:eoc_non_diag_var}
\end{figure*}

\subsection{Additional error bars for Figures \ref{fig:NTK_var_lim}, \ref{fig:NTK_var_eoc} and \ref{fig:non_diag}}\label{appendix:error_bars}

To keep all the figures readable, we include error bars only in a subset of points in Figures \ref{fig:NTK_var_lim} and \ref{fig:NTK_var_eoc} in the main text. We also omit error bars in Figure \ref{fig:non_diag}. To give the reader a better idea about the variance observed in our experiments, we include additional figures with continuous error bars in this section. Figure \ref{fig:fig1_errors} is analogous to Figure \ref{fig:NTK_var_lim}: it shows the estimated dispersion of the NTK along with the theoretical expressions in the infinite-depth-and-width limit from Theorem \ref{th:NTK_dispersion_lim}. We include fewer lines (values of $\sigma_w^2$) in this figure to keep the continuous error bars distinguishable. Similarly, Figure \ref{fig:fig2_errors} is analogous to Figure \ref{fig:NTK_var_eoc}: it shows the results concerning the NTK dispersion around the EOC. Finally, Figure \ref{fig:fig3_errors} shows a subset of lines from Figure \ref{fig:non_diag} with their continuous error bars and concerns the ratio between non-diagonal and diagonal elements of the NTK.

\begin{figure}
    \centering
    \includegraphics[width=\textwidth]{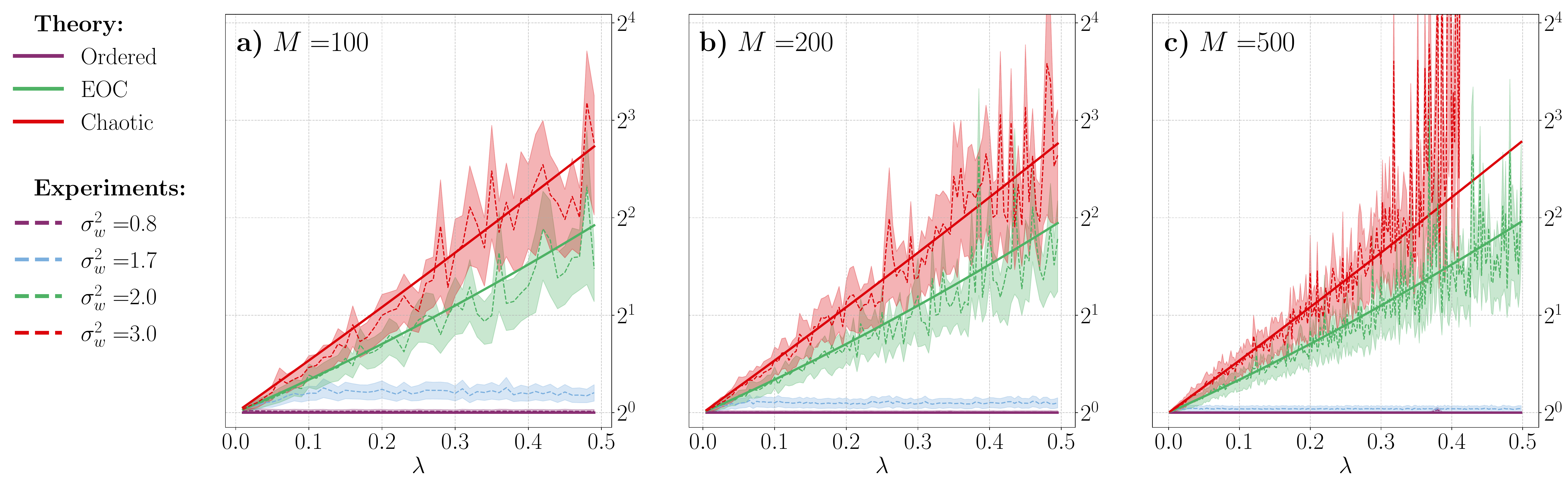}
    \caption{Ratio $\mathbb{E}[\Theta^2(x,x)]/\mathbb{E}^2[\Theta(x,x)]$ at initialization for fully-connected ReLU DNNs of constant width $M\in\{100,200,500\}$ with $\alpha_0=1$. The experiment setup is the same as in Figure \ref{fig:NTK_var_lim}. Continuous error bars show the bootstrap estimation of the standard error.}
    \label{fig:fig1_errors}
\end{figure}

\begin{figure}
    \centering
    \includegraphics[width=\textwidth]{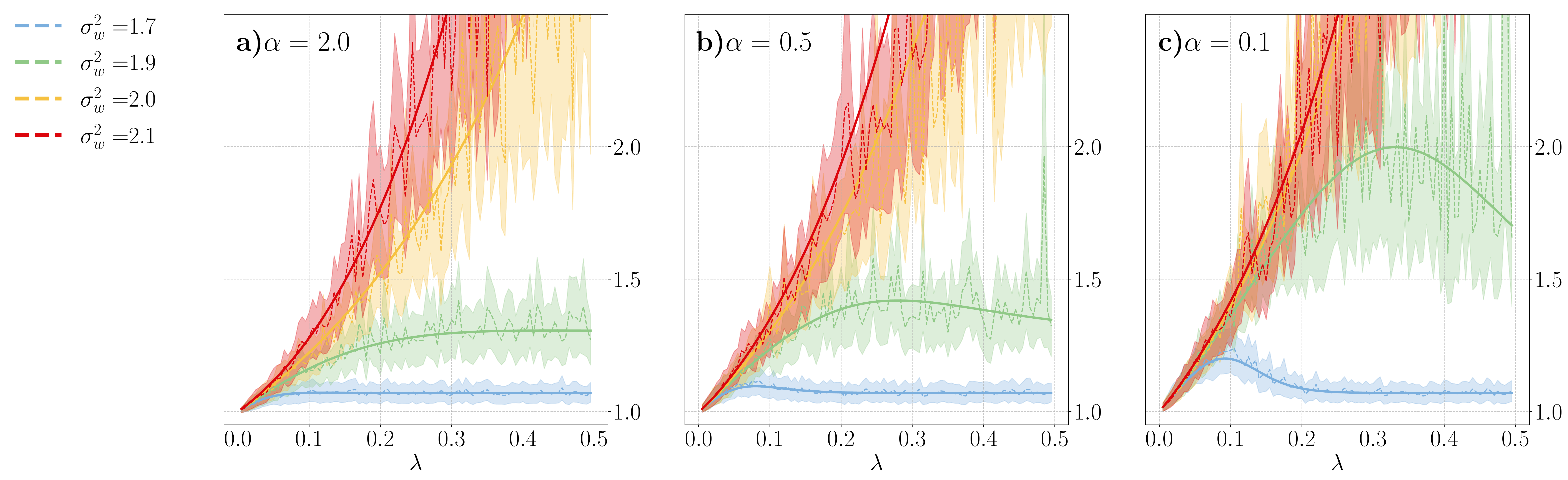}
    \caption{Ratio $\mathbb{E}[\Theta^2(x,x)]/\mathbb{E}^2[\Theta(x,x)]$ at initialization for fully-connected ReLU networks of constant width $M=200$ with the ratio $\alpha_0:=n_0/M \in\{2.0,0.5,0.1\}$. The initialization hyperparameter $\sigma_w^2$ is close to the EOC for all the lines. The experiment setup is the same as in Figure \ref{fig:NTK_var_eoc}. Continuous error bars show the bootstrap estimation of the standard error.}
    \label{fig:fig2_errors}
\end{figure}

\begin{figure}
    \centering
    \includegraphics[width=\textwidth]{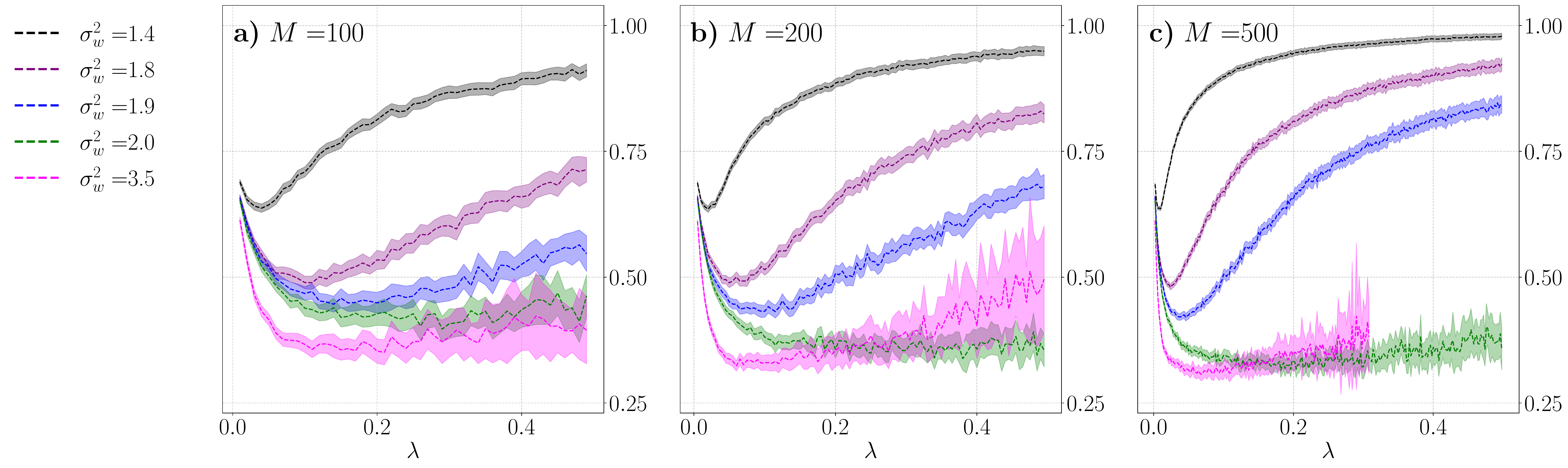}
    \caption{Ratio $\mathbb{E}[\Theta^2(x,x)]/\mathbb{E}^2[\Theta(x,x)]$ at initialization for fully-connected ReLU networks of constant width $M=200$ with the ratio $\alpha_0:=n_0/M \in\{2.0,0.5,0.1\}$. The initialization hyperparameter $\sigma_w^2$ is close to the EOC for all the lines. The experiment setup is the same as in Figure \ref{fig:NTK_var_eoc}. Continuous error bars show the bootstrap estimation of the standard error.}
    \label{fig:fig3_errors}
\end{figure}

\subsection{Estimating the NTK dispersion from a sample}\label{appendix:estimator}

In our experiments, we estimate the ratio $r:=\mathbb{E}[\Theta^2(x,x)]/\mathbb{E}^2[\Theta(x,x)]$ at initialization from a sample. To do so, we sample an element of the NTK $N$ times with independently chosen intialization parameters and get a sample $\{\theta_i\}_{i=1}^N$. 

In this setting, the standard estimators for the first and the second moments given by $\hat{\mu}_1:=\sum_{i=1}^N \theta_i/N$  and $\hat{\mu}_2~:=~\sum_{i=1}^N \theta_i^2/N $  are unbiased:
\begin{equation*}
    \begin{split}
        \mathbb{E}[\hat{\mu}_1]=\mathbb{E}\Bigl[\dfrac{1}{N}\sum_{i=1}^N \theta_i \Bigr] = \dfrac{1}{N}\sum_{i=1}^N  \mathbb{E}[\theta_i] = \mu_1,\\
        \mathbb{E}[\hat{\mu}_2]=\mathbb{E}\Bigl[\dfrac{1}{N}\sum_{i=1}^N \theta_i^2 \Bigr] = \dfrac{1}{N}\sum_{i=1}^N  \mathbb{E}[\theta_i^2] = \mu_2,\\
    \end{split}
\end{equation*}
where we denoted the actual moments as $\mu_1:=\mathbb{E}[\theta_i]$ and $\mu_2:=\mathbb{E}[\theta_i^2]$. 

However, $\hat{\mu}_1$ and $\hat{\mu}_2$ computed on the same sample are dependent, so the estimator for the desired ratio given by $\hat{\mu}_2/(\hat{\mu}_1)^2$ is biased and we need to correct it. First, we notice that the estimator for $\mu_1^2$ given by the square of $\hat{\mu}_1$ is biased as follows:
\begin{equation*}
     \mathbb{E}[\hat{\mu}_1^2]=\mathbb{E}\Bigl[\dfrac{1}{N^2}\Bigl(\sum_{i=1}^N \theta_i\Bigr)^2 \Bigr] = \dfrac{\mu_2}{N} + \dfrac{N-1}{N}\mu_1^2
\end{equation*}
Therefore, an unbiased estimator for $\mu_1^2$ can be computed as 
\begin{equation*}
    \widehat{(\mu_1^2)} = \dfrac{N}{N-1}(\hat{\mu}_1^2 - \dfrac{1}{N}\hat{\mu}_2)
\end{equation*}
Second, we want to remove the dependence between the numerator and the denominator, which can be done simply by using disjoint parts of the sample to compute the two:
\begin{equation*}
    \hat{r}:= \dfrac{1}{N-2}\sum_{i=1}^N \dfrac{\theta_i^2}{\frac{1}{(N-1)(N-2)}\Bigl[ \Bigl(\sum_{j\neq i}\theta_j\Bigr)^2 - \sum_{j\neq i}\theta_j^2 \Bigr]},
\end{equation*}
where we used the unbiased version of the estimator for $\mu_1^2$ computed on a sample without $\theta_i$ in the denominator. Then we can compute the expectation of our new estimator for the ratio as follows:
\begin{equation*}
    \mathbb{E}[\hat{r}] = (N-1)\sum_{i=1}^N \dfrac{\mathbb{E}[\theta_i^2]}{\mathbb{E}\Bigl[ \Bigl(\sum_{j\neq i}\theta_j\Bigr)^2 - \sum_{j\neq i}\theta_j^2 \Bigr]} = \dfrac{\mu_2}{\mu_1^2}.
\end{equation*}
Therefore, $\hat{r}$ is an unbiased estimator of the ratio.

\end{document}